%% file: Main.tex
\documentclass{article}

\usepackage{arxiv}

\usepackage[utf8]{inputenc} % allow utf-8 input
\usepackage[T1]{fontenc}    % use 8-bit T1 fonts
\usepackage{hyperref}       % hyperlinks
\usepackage{url}            % simple URL typesetting
\usepackage{booktabs}       % professional-quality tables
\usepackage{amsfonts}       % blackboard math symbols
\usepackage{nicefrac}       % compact symbols for 1/2, etc.
\usepackage{microtype}      % microtypography
\usepackage{xcolor}         % colors
\usepackage{amsmath}
\usepackage{changepage}
\usepackage{pifont}
\usepackage{dsfont}
\usepackage{amsthm}
\newtheorem{theorem}{Theorem}
\newtheorem{corollary}{Corollary}
\newtheorem{lemma}{Lemma}

\newtheorem{definition}{Definition}

\newtheorem{assumption}{Assumption}
\newtheorem{proposition}{Proposition}

\newtheorem{remark}{Remark}
\usepackage{bm}
\usepackage{graphicx, wrapfig}
\usepackage{caption}
\usepackage{subcaption}
\hypersetup{hidelinks}
\usepackage{color}
\usepackage{array}
\usepackage{multirow}
\input{math_commands}

\usepackage{amssymb}

\newcommand{\bx}{\boldsymbol{x}}

\newcommand{\bw}{\boldsymbol{w}}
\newcommand{\bW}{\boldsymbol{W}}
\newcommand{\by}{\boldsymbol{y}}
\newcommand{\bo}{\boldsymbol{o}}
\newcommand{\bom}{\boldsymbol{m}}
\newcommand{\bz}{\boldsymbol{z}}
\newcommand{\bS}{\boldsymbol{S}}
\newcommand{\bP}{\mathbf{P}}
\newcommand{\bv}{\boldsymbol{v}}
\newcommand{\bA}{\boldsymbol{A}}
\newcommand{\bB}{\boldsymbol{B}}
\newcommand{\bn}{\boldsymbol{n}}
\newcommand{\br}{\boldsymbol{r}}
\newcommand{\bu}{\boldsymbol{u}}
\newcommand{\be}{\boldsymbol{e}}
\newcommand{\bnu}{\boldsymbol{\nu}}
\newcommand{\hbw}{\hat{\bw}}
\newcommand{\hbW}{\hat{\bW}}
\newcommand{\sigmax}{\sigma_{max}}
\newcommand{\tell}{\tilde{\ell}}

\newcommand{\tbw}{\tilde{\bw}}
\newcommand{\tbS}{\tilde{\bS}}
\newcommand{\cbw}{\check{\bw}}

\newcommand{\hbom}{\hat{\bom}}

\newcommand{\hbnu}{\hat{\bnu}}
\newcommand{\vect}{\text{vec}}
\usepackage{float}
\usepackage[linesnumbered,ruled,vlined]{algorithm2e}

\title{Does Momentum Change the Implicit Regularization on Separable Data?}

\author{Bohan Wang \\
University of Science \& Technology of China
\\
Microsoft Research Asia
\And
Qi Meng
\\
Microsoft Research Asia
\And
 Huishuai Zhang
 \\
 Microsoft Research Asia
 \And
 Ruoyu Sun
 \\
University of Illinois at Urbana-Champaign
 \And
 Wei Chen
 \\
 Institute of Computing Technology,
 \\
 Chinese Academy of Sciences
 \And
  Zhi-Ming Ma
  \\
  Academy of Mathematics and Systems Science,
  \\
  Chinese Academy of Sciences
  \And
  Tie-Yan Liu
  \\
  Microsoft Research Asia
}

\renewcommand{\shorttitle}{}

\hypersetup{
pdftitle={Thoughts},
pdfsubject={CS,MATH},
pdfauthor={Bohan Wang},
pdfkeywords={Thoughts},
}

\begin{document}

\maketitle

\input{Tex Files/Abstract}

\input{Tex Files/introduction}

\input{Tex Files/Preliminaries}

\input{Tex Files/MainResults_v1}

\input{Tex Files/Conclusion}

\newpage
\bibliography{iclr2022_conference}
\bibliographystyle{abbrv}

\newpage
\input{Tex Files/Appendix}

%%%%%%%%%%%%%%%%%%%%%%%%%%%%%%%%%%%%%%%%%%%%%%%%%%%%%%%%%%%%

\end{document}

%% file: math_commands.tex
\usepackage{amsmath,amsfonts,bm}

\def\eqref#1{equation~\ref{#1}}
\def\1{\bm{1}}

\def\mE{{\mathbb{E}}}
\def\mF{{\mathcal{F}}}

\def\mI{{\bm{I}}}

\def\mL{{\mathcal{L}}}

\def\mO{{\bm{O}}}

\def\mR{{\bm{R}}}
\def\mS{{\bm{S}}}
\def\mT{{\bm{T}}}

\DeclareMathAlphabet{\mathsfit}{\encodingdefault}{\sfdefault}{m}{sl}
\SetMathAlphabet{\mathsfit}{bold}{\encodingdefault}{\sfdefault}{bx}{n}

%% file: Tex Files/Abstract.tex
\begin{abstract}

The momentum acceleration technique is widely adopted in many optimization algorithms. However, there is no theoretical answer on how the momentum affects the generalization performance of the optimization algorithms. This paper studies this problem by analyzing the implicit regularization of momentum-based optimization. We prove that on the linear classification problem with separable data and exponential-tailed loss, gradient descent with momentum (GDM) converges to the $L_2$ max-margin solution, which is the same as vanilla gradient descent. 
That means gradient descent with momentum acceleration still converges to a low-complexity model, which guarantees their generalization. We then analyze the stochastic and adaptive variants of GDM (i.e., SGDM and deterministic Adam) and show they also converge to the $L_2$ max-margin solution. Technically, to overcome the difficulty of the error accumulation in analyzing the momentum, we construct new potential functions to analyze the gap between the model parameter and the max-margin solution. Numerical experiments are conducted and support our theoretical results.

\end{abstract}

%% file: Tex Files/introduction.tex
\section{Introduction}
\label{sec: intro}

  It is widely believed that the optimizers have the implicit regularization in terms of selecting output parameters among all the minima on the landscape  \cite{neyshabur2015search, keskar2016large,wilson2017marginal}. Parallel to the analysis of \emph{coordinate descent} (\cite{schapire2013boosting,telgarsky2013margins}), \cite{soudry2018implicit} shows that \emph{gradient descent} would converge to the $L^2$ max-margin solution for the linear classification task with exponential-tailed loss, which mirrors its good generalization property in practice. Since then, many efforts have been taken on analyzing the implicit regularization of various local-search optimizers, including stochastic gradient descent \cite{nacson2019stochastic}, steepest descent \cite{gunasekar2018characterizing}, AdaGrad~\cite{qian2019implicit}  and optimizers for homogeneous neural networks  \cite{lyu2019gradient,ji2020directional,wang2021implicit}.

However, though the momentum acceleration technique is widely adopted in the optimization algorithms in both convex and non-convex learning tasks \cite{sutskever2013importance, vaswani2017attention,tan2019efficientnet}, the understanding of how the momentum would affect the generalization performance of the optimization algorithms is still unclear, as the historical gradients in the momentum may significantly change the searching direction of the optimization dynamics. A natural question is:

\begin{center}
\emph{Can we theoretically analyze the implicit regularization of momentum-based optimizers?}
\end{center}

In this paper, we take the first step to analyze the convergence of momentum-based optimizers and unveil their implicit regularization.  %Specifically, 
Our study starts from the classification problem with the linear model and exponential-tailed loss using Gradient Descent with Momentum (GDM) optimizer. Then the variants of GDM such as Stochastic Gradient Descent with Momentum (SGDM) and deterministic Adam are also analyzed. 
We consider the optimizers with constant learning rate and constant momentum hyper-parameters, which are widely adopted in practice, e.g., the default setting in popular machine learning frameworks \cite{paszke2019pytorch} and in experiments \cite{xie2017aggregated}. %We note that Gradient Descent with Momentum (GDM) can be viewed as a particular case of SGDM and naturally share the properties of SGDM. 
Our main results are summarized in Theorem \ref{thm: inform}.

\begin{theorem}[informal]
\label{thm: inform}
With linearly separable dataset $\bS$, linear model and exponential-tailed loss:
\begin{itemize}
    \item For GDM with a constant learning rate, the parameter norm diverges to infinity, with its direction converging to the $L^2$ max-margin solution. The same conclusion holds for SGDM with a constant learning rate.
    \item For deterministic Adam with a constant learning rate and stochastic RMSProp (i.e., Adam without momentum) with a decaying learning rate, the same conclusion holds.
\end{itemize}
 
\end{theorem}
\begin{table}[t!]
  \caption{The algorithms investigated in this paper (GDM and Adam) along with algorithms (GD) already investigated in the existing literature. We also compare the learning rates required to obtain the characterization of implicit regularization. As for stochastic Adam with $\beta_1\ne 0$, we leave its implicit regularization as future work.}
  \label{sample-table}
  \centering
  \begin{tabular}{cccc}
    \toprule
       Method & With Random Sampling & Learning Rate& Corresponding Literature               \\
    \midrule
    \multirow{2}{*}{GD} & $\times$& Constant & \cite{soudry2018implicit} \\ 
     \cmidrule(r){2-4}
      & \checkmark & Constant & \cite{nacson2019stochastic} \\
      \midrule
      \multirow{2}{*}{GDM} & $\times$& Constant & This Work \\ 
\cmidrule(r){2-4}
 & \checkmark & Constant & This Work \\
 \midrule
 \multirow{2}{*}{Adam} & $\times$& Constant & This Work \\ 
\cmidrule(r){2-4}
 & \checkmark & Decaying $\left(\frac{1}{\sqrt{t}}\right)$ & $\beta_1=0$ in this Work 
 \\
    \bottomrule
  \end{tabular}
\end{table}

% \begin{table*}
% \centering
%  \begin{tabular}{||c| c| c| c||} 
%  \hline
%  Method & With Random Sampling & Learning Rate& Corresponding Literature \\ [0.5ex] 
%  \hline\hline
% \multirow{2}{*}{GD} & $\times$& Constant & \cite{soudry2018implicit} \\ 
% \cline{2-4}
%  & \checkmark & Constant & \cite{nacson2019stochastic} \\
%  \hline
% \multirow{2}{*}{GDM} & $\times$& Constant & This Work \\ 
% \cline{2-4}
%  & \checkmark & Constant & This Work \\
%  \hline
%  \multirow{2}{*}{Adam} & $\times$& Constant & This Work \\ 
% \cline{2-4}
%  & \checkmark & Decaying $\left(\frac{1}{\sqrt{t}}\right)$ & $\beta_1=0$ in this Work 
%  \\
%  \hline
%  \end{tabular}
%  \caption{The algorithms investigated in this paper (GDM and Adam) along with algorithms (GD) already investigated in the existing literature. We also compare the learning rates required to obtain the characterization of implicit regularization. As for stochastic Adam with $\beta_1\ne 0$, we leave its implicit regularization as future work. }
% \end{table*}

Theorem \ref{thm: inform} states that GDM and its variants converge to the $L^2$ max-margin solution, which is the same as their without-momentum versions, indicating that momentum does not affect the convergent direction. Therefore, the good generalization behavior of the output parameters of these optimizers is well validated as the margin of a classifier is positively correlated with its generalization error \cite{jiang2019fantastic} and is supported by existing experimental observations (e.g., \cite{soudry2018implicit,nacson2019convergence,wang2021implicit}). 

Our contributions are significant in terms of the following aspects:
\begin{itemize}
    \item We establish the implicit regularization of the momentum-based optimizers,  an open problem since the initial work \cite{soudry2018implicit}. The momentum-based optimizers are widely used in practice, and our theoretical characterization deepens the understanding of their generalization property, which is important on its own.
    
    \item Technically, we design a two-stage framework to analyze the momentum-based optimizers, which generalizes the proof techniques in \cite{soudry2018implicit} and \cite{nacson2019stochastic}. The first stage shows the convergence of the loss. New potential functions for SGDM and Adam are proposed and can be of independent interest for convergence analysis of momentum-based optimizers. The second stage shows the convergence of the parameter. We propose an easy-to-check condition of whether the difference between learned parameters and the scaled max-margin solution is bounded. This condition can be generalized to implicit regularization analyses of other momentum-based optimizers.
    
    \item We further verify our theory through numerical experiments.

\end{itemize}
%and tackle the difficulty of direction analysis due to the complexity of momentum

\textbf{Organization of This Paper.} 
Section \ref{sec: related_work} collects further related works on the implicit regularization of the first order optimizers and the convergence of momentum-based optimizers. Section \ref{sec: preliminary} shows basic settings and assumptions which will be used throughout this paper. Section \ref{subsec: main_gdm} studies the implicit regularization of GDM, while Section \ref{subsec: main_sgdm} and Section \ref{subsec: main_adam} explore respectively the implicit regularization of SGDM and Adam. Discussions of these results are put in Section \ref{sec: discussion}. Detailed proofs and experiments can be found in the appendix.

\section{Further related works}
\label{sec: related_work}
\textbf{Implicit Regularization of First-order Optimization Methods.} \cite{soudry2018implicit} prove that gradient descent on linear classification problem with exponential-tailed loss converges to the direction of the max $L^2$ margin solution of the corresponding hard-margin Support Vector Machine.  \cite{nacson2019stochastic} extend the results in \cite{soudry2018implicit} to the stochastic case, proving that the convergent direction of SGD is the same as GD almost surely. \cite{qian2019implicit} go beyond the vanilla gradient descent methods and consider the AdaGrad optimizer instead. They prove that the convergent direction of AdaGrad has a dependency on the optimizing trajectory, which varies according to the initialization. \cite{ji2021characterizing} propose a primal-dual analysis framework for the linear classification models and prove a faster convergent rate of the margin by increasing the learning rate according to the loss. Based on \cite{ji2021characterizing}, \cite{ji2021fast} design another algorithm with an even faster convergent rate of margin by applying the Nesterov's Acceleration Method on the dual space. However, the corresponding form of the algorithm on the primal space is no longer a Nesterov's Acceleration Method nor GDM, which is significantly different from our settings.

On the other hand, another line of work is trying to extend the linear case result to deep neural networks. \cite{ji2018gradient,gunasekar2018implicit} study the deep linear network and \cite{soudry2018implicit} study the two-layer neural network with ReLU activation. \cite{lyu2019gradient} propose a framework to analyze the asymptotic direction of GD on homogeneous neural networks, proving that given there exists a time the network achieves $100\%$ training accuracy, GD will converge to some KKT point of the $L^2$ max-margin problem. \cite{wang2021implicit} extend the framework of \cite{lyu2019gradient} to adaptive optimizers and prove RMSProp and Adam without momentum have the same convergent direction as GD, while AdaGrad does not. The results  \cite{lyu2019gradient,wang2021implicit} indicate that results in the linear model can be extended to deep homogeneous neural networks and suggest that the linear model is an appropriate starting point to study the implicit bias.

Except for the exponential-tailed loss, there are also works on the implicit bias with squared loss. Interesting readers can refer to \cite{rosasco2015learning,lin2017optimal,ali2020implicit} etc. for details.

\textbf{Convergence of Momentum-Based Optimization Methods.} {  For convex optimization problems, the convergence rate of Nesterov's Acceleration Method \cite{nesterov1983method} has been proved in various approaches (e.g., \cite{nesterov1983method,su2014differential,wilson2021lyapunov}).}  In contrast, although GDM (Polyak's Heavy-Ball Method) was proposed in \cite{polyak1964some} before the Nesterov's Acceleration Method, the convergence of GDM on convex loss with Lipschitz gradient was not solved until \cite{ghadimi2015global} provides an ergodic convergent result for GDM, i.e., the convergent result for the running average of the iterates. \cite{sun2019non} provide a non-ergodic analysis when the training loss is coercive (the training loss goes to infinity whenever parameter norm goes to infinity), convex, and globally smooth. However, all existing results cannot be directly applied to exponential-tailed loss, which is non-coercive.

There are also works on the convergence of SGDM under various settings. \cite{yan2018unified} prove SGDM converges to a bounded region assuming both bounded gradient norm and bounded gradient variance. The bounded gradient norm assumption is further removed by  \cite{yu2019linear,liu2020improved}. Nevertheless, a converging-to-stationary-point analysis is required in the implicit regularization analysis. Thus their results can not be directly applied. \cite{tao2021role} analyze a particular case when the momentum parameter increases over iterations, which, however, does not agree with the practice where the momentum parameter is fixed.

As for (stochastic) Adam, its convergence analysis is still an open problem, and the current analyses are restricted to specific settings (e.g., bounded gradient, dynamical momentum hyperparameters). We recommend interested readers to refer to  \cite{defossez2020simple,de2018convergence,shi2021rmsprop,guo2021novel} for details.

%% file: Tex Files/Preliminaries.tex
\section{Preliminaries}
\label{sec: preliminary}
This paper focuses on the linear model with the exponential-tailed loss. We mainly investigate binary classification. However, the methodology can be easily extended to the multi-class classification problem { (please refer to Appendix \ref{appen: multi-class} for details)}.  

\textbf{Problem setting.} The dataset used for training is denoted as  $\bS=(\bx_i,\by_i)_{i=1}^N$, where $\bx_i\in \mathbb{R}^d$ is the $i$-th input feature, and $\by_i\in \mathbb{R}$ is the $i$-th label ($i=1,2,\cdots,N$). We will use the linear model to fit the label: for any feature $\bx\in \mathbb{R}^d$ and parameter $\bw\in \mathbb{R}^{d}$, the prediction is given by $\langle\bw, \bx\rangle$.

For binary classification, given any data $\bz_i=(\bx_i,\by_i)\in \bS$, the individual loss for parameter $\bw$ is given as $\ell(\by_i\langle \bw,\bx_i \rangle)$. As only $\by_i\bx_i$ is used in the loss, we then ensemble the feature and label together and assume $\by_i=1$ ($\forall i\in \{1,\cdots,N\}$) without the loss of generality. We then drop $\by_i$ for brevity and redefine $\bS=(\bx_i)_{i=1}^N$. The spectral norm of the data matrix $(\bx_1,\cdots,\bx_N)$ is defined as $\sigma_{max}$.
We use 
$    \tilde{\ell}(\bw,\bx)\triangleq\ell(\langle \bw,{\bx} \rangle)
$ for brevity.

The optimization target is defined as the averaged loss:
$
    \mL(\bw)=\frac{\sum_{i=1}^N\tilde{\ell}(\bw,\bx_i)}{N}.
$

\textbf{Optimizer.} Here we will introduce the update rules of GDM, SGDM and {  deterministic Adam}. GDM's update rule is 
\small
\begin{equation}
\label{eq: def_gdm}
    \bom(0)=\boldsymbol{0},\bom(t)=\beta\bom(t-1)+(1-\beta) \nabla \mL(\bw(t)), \bw(t+1)=\bw(t)-\eta \bom(t).
\end{equation}
\normalsize

SGDM can be viewed as a stochastic version of GDM by randomly choosing a subset of the dataset to  update. Specifically, SGDM changes the update of $\bom(t)$ into 
\small
\begin{equation}
    \label{eq: def_sgdm}
   \bom(t)=\beta\bom(t-1)+(1-\beta) \nabla \mL_{\bB(t)}(\bw(t)),
\end{equation}
\normalsize
    where $\bB(t)$ is a subset of $\bS$ with size $b$ which can be sampled either with replacement (\textbf{abbreviated as "w/. r"}) or without replacement (i.e., with random shuffling, \textbf{abbreviated as "w/o. r"}), and $\mL_{\bB(t)}$ is defined as 
$
    \mL_{\bB(t)}(\bw)=\frac{\sum_{\bx\in \bB(t)}\tilde{\ell}(\bw,\bx)}{b}
$. We also define $\mF_t$ as the sub-sigma algebra over the mini-batch sampling, such that $\forall t\in\mathbb{N}$, $\bw(t)$ is adapted with respect to the sigma algebra flow $\mF_t$.

The Adam optimizer can be viewed as a variant of SGDM in which the preconditioner is adopted, whose form is characterized as follows:
\small
\begin{gather}
\nonumber
    \bom(0)=\boldsymbol{0}, \bom(t)=\beta_1 \bom(t-1)+(1-\beta_1)\nabla \mL_{\bB(t)}(\bw(t)),
    \\
\nonumber
\bnu(0)=0,\bnu(t)=\beta_2 \bnu(t-1)+(1-\beta_2)\nabla \mL(\bw(t)) \odot\nabla \mL(\bw(t))
\\\nonumber
 \hbom(t)=\frac{1}{1-\beta^t_1}\bom(t),
    \hbnu(t)=\frac{1}{1-\beta^t_2}\bnu(t),
\\
\label{eq: def_adam}
   \textit{(Update Rule)}:~ \bw(t)=\bw(t-1)-\eta\frac{1}{\sqrt{\hbnu(t-1)+\varepsilon\mathds{1}_d}}\odot\hbom(t-1),
\end{gather}
\normalsize
where $\frac{1}{\sqrt{\hbnu(t-1)+\varepsilon\mathds{1}_d}}$ is called the preconditioner.

% {  \textbf{Lyapunov functions.} They are commonly adopted to analyze the convergence of optimizers \cite{wilson2021lyapunov}. In this paper, we slightly generalize the non-negative and non-increasing definition of Lyapunov function in the existing literature to non-negativeness and upper-boundedness.}

\textbf{Assumptions:} The analysis of this paper are based on { three common assumptions in existing literature (first proposed by \cite{soudry2018implicit})}. They are respectively on the separability of the dataset, the individual loss behavior at the tail, and the smoothness of the individual loss. We list them as follows:
\begin{assumption}[Linearly Separable Dataset]
\label{assum: separable}
There exists one parameter $\bw\in \mathbb{R}^d$, such that
\small
\begin{equation*}
    \langle \bw, \bx_i \rangle>0,~\forall i\in [N].
\end{equation*}
\normalsize
\end{assumption}

\begin{assumption}[Exponential-tailed Loss]
\label{assum: exponential-tailed}
The individual loss $\ell$ is exponential-tailed, i.e.,
\begin{itemize}
    \item Differentiable and  monotonically decreasing to zero, with its derivative  converging to zero at positive infinity and to non-zero at negative infinity, i.e., $\lim_{x\rightarrow\infty}\ell(x)$ $=\lim_{x\rightarrow \infty}\ell'(x)=0$, $\varlimsup_{x\rightarrow -\infty }\ell' (x)<0 $, and $\ell'(x){ <}0$, $\forall x\in \mathbb{R}$;
    \item Close to exponential loss when $x$ is large enough, i.e., there exist positive constants $c,a,\mu_+,\mu_-,x_+$, and  $x_-$, such that,
    \small
    \begin{align}
    \label{eq: tail_upper_bound}
        \forall x>x_+: -\ell'(x)\le c(1+e^{-\mu_+ x})e^{-ax},
        \\
    \label{eq: tail_lower_bound}
        \forall x>x_-: -\ell'(x)\ge c(1-e^{-\mu_- x})e^{-ax}.
    \end{align}
    \normalsize
\end{itemize}
\end{assumption}

\begin{assumption}[Smooth Loss]
 \label{assum: smooth}
Either of the following assumptions holds regarding the case: 
 
 \textbf{(D):}  (Without Stochasticity) The individual loss $\ell$ is locally smooth, i.e., for any $s_0 \in \mathbb{R}$, there exists a positive real $H_{s_0}$, such that $\forall x,y\ge s_0$, $\vert \ell'(x)-\ell'(y) \vert\le H_{s_0}\vert x-y \vert $.
 
 \textbf{(S):} (With Stochasticity) The individual loss $\ell$ is globally smooth, i.e., there exists a positive real $H$, such that $\forall x,y\in \mathbb{R}$, $\vert \ell'(x)-\ell'(y) \vert \le H\vert x-y \vert $.
 \end{assumption}
 
We provide explanations of these three assumptions, respectively. Based on Assumption \ref{assum: separable}, we can formally define the margin and the maximum margin solution of an optimization problem:
\begin{definition}
\label{def: margin}
Let the margin $\hat{\gamma}(\bw)$ of parameter $\bw$ defined as the lowest score of the prediction of $\bw$ over the dataset $\bS$, i.e.,  $\hat{\gamma}(\bw)=\min_{\bx\in\bS}\langle \bw, \bx\rangle$. We then define the maximum margin solution $\hbw$ and the  {$L^2$ max} margin $\gamma$ of the dataset S as follows: 
\small
\begin{equation*}
    \hbw\overset{\triangle}{=} \arg\min_{\hat{\gamma}(\bw)\ge 1} \Vert\bw\Vert^2,~ \gamma\overset{\triangle}{=}\frac{1}{\Vert \hbw\Vert}
\end{equation*}
\normalsize
\end{definition}

Since $\Vert \cdot\Vert^2$ is strongly convex and set $\{\bw: \hat{\gamma}(\bw)\ge 1\}$ is convex, $\hbw$ is uniquely defined.

Assumption \ref{assum: exponential-tailed} constraints the loss to be exponential-tailed, which is satisfied by many popular choices of $\ell$, including the exponential loss ($\ell_{exp}(x)=e^{-x}$) and the logistic loss ($\ell_{log}(x)=\log(1+ e^{-x})$).
Also, as $c$ and $a$ can be respectively absorbed by resetting the learning rate and data as $\eta=c\eta$ and $\bx_i=a \bx_i$, without loss of generality,  \textbf{in this paper we only analyze the case that $c=a=1$}.
 
The globally smooth assumption (Assumption \ref{assum: smooth}. (S)) is strictly stronger than the locally smooth assumption (Assumption \ref{assum: smooth}. (D)). One can easily verify that both the exponential loss  and the logistic loss  meet Assumption \ref{assum: smooth}. (\textbf{D}), and the logistic loss also meets Assumption \ref{assum: smooth}. (\textbf{S}).

%% file: Tex Files/MainResults_v1.tex
\section{The implicit regularization of GDM}
\label{subsec: main_gdm}
In this section, we analyze the implicit regularization of GDM with a two-stage framework \footnote{It should be noticed that the proof sketches in Sections \ref{subsec: main_gdm}, \ref{subsec: main_sgdm}, and \ref{subsec: main_adam} only hold for almost every dataset (means except a zero-measure set in $\mathbb{R}^{d\times N}$), as we want the presentation more simple and straightforward. However, the proof can be extended to every dataset with a more careful analysis (please refer to Appendix \ref{appen: every_data_set} for details). }. Later, we will use this framework to investigate SGDM further and {  deterministic Adam}. The formal theorem of the implicit regularization of GDM is as follows:

\begin{theorem}
\label{thm: gdm_main}
Let Assumptions \ref{assum: separable}, \ref{assum: exponential-tailed}, and \ref{assum: smooth}. (\textbf{D}) hold. Let $\beta \in [0,1)$ and $\eta<2\frac{N}{\sigma_{max}^2H_{\ell^{-1}}(N\mathcal{L}(\boldsymbol{w}_1))}$ ($\ell^{-1}$ is the inverse function of $\ell$). Then, for almost every data set $S$, with arbitrary initialization point $\bw(1)$, GDM (Eq.~\eqref{eq: def_gdm}) satisfies that
$\bw(t)-\ln(t) \hbw$ is bounded as $t\rightarrow \infty$, and $\lim_{t\rightarrow\infty}\frac{\bw(t)}{\Vert \bw(t) \Vert} =\frac{\hbw}{\Vert \hbw \Vert}$.
\end{theorem}
Theorem \ref{thm: gdm_main} shows that the implicit regularization of GDM agrees with GD in linear classification with exponential-tailed loss (c.f. \cite{soudry2018implicit} for results on GD). This consistency can be verified by existing and our experiments (c.f. Section \ref{sec: discussion} for detailed discussions).

\begin{remark}[On the hyperparameter setting]
Firstly, the learning rate upper bound $2\frac{N}{\sigma_{max}^2H_{\ell^{-1}}(N\mathcal{L}(\boldsymbol{w}_1))}$ agrees with that of GD exactly \cite{soudry2018implicit}, indicating our analysis is tight.
Secondly, Theorem \ref{thm: gdm_main} adopts a constant momentum hyper-parameter, which agrees with the practical use (e.g., $\beta$ is fixed to be $0.9$ \cite{xie2017aggregated}).  Also, Theorem \ref{thm: gdm_main} puts no restriction on the range of $\beta$, which allows wider choices of hyper-parameter tuning. 
\end{remark}

We then present a proof sketch of Theorem \ref{thm: gdm_main}, which is divided into two parts: we first prove that the sum of squared gradients is bounded, which indicates both the loss and the norm of gradient converge to $0$ and the parameter diverges to infinity; these properties will then be applied to show the difference between $\bw(t)$ and $\ln(t) \hbw$ is bounded, and therefore, the direction of $\hbw$ dominates as $t\rightarrow \infty$.

\textbf{Stage I: Loss Dynamics.} The goal of this stage is to characterize the dynamics of the loss and prove the convergence of GDM. The core of this stage is to select a proper potential function $\xi(t)$, which is required to correlate with the training loss $\mL$ and be non-increasing along the optimization trajectory. For GD, since $\mL$ is non-increasing with a properly chosen learning rate, we can pick $\xi(t)=\mL(t)$. However, as the update of GDM does not align with the direction of the negative gradient, training loss $\mL(t)$ in GDM is no longer monotonously decreasing, and the potential function requires special construction. Inspired by \cite{sun2019non}, we choose the following $\xi(t)$:

\begin{lemma}
\label{lem: update_rule_loss_gdm}
Let all conditions in Theorem \ref{thm: gdm_main} hold. Define $\xi(t)\overset{\triangle}{=}\mL(\bw(t))+\frac{1-\beta}{2\beta\eta} \Vert \bw(t)-\bw(t-1)\Vert^2.$ Define ${ C_1}$ as a positive real with ${ C_1}\triangleq\frac{\sigma_{max}^2H_{\ell^{-1}}(N\mathcal{L}(\boldsymbol{w}_1))}{2N} \eta$. We then have  
\small
\begin{align}
\label{eq: gdm_direct_taylor}
   \xi(t)
    \ge\xi(t+1)
    +\frac{1-{ C_1}}{\eta}\Vert \bw(t+1)-\bw(t)\Vert^2.
\end{align}    \normalsize
\end{lemma}
\begin{remark}
Although this potential function is obtained by \cite{sun2019non} by directly examining Taylor's expansion at $\bw(t)$, the proof here is non-trivial as we only require the loss to be locally smooth instead of globally smooth in \cite{sun2019non}. We need to prove that the smoothness parameter along the trajectory is upper bounded. We defer the detailed proof to Appendix \ref{appen: gdm_sum_finite}.
\end{remark}

% To prove Lemma \ref{lem: update_rule_loss_gdm}, we define $\bw(t+\alpha)=\alpha\bw(t)+(1-\alpha)\bw(t+1)$ for any $t\in\mathbb{Z}^{+}$ and $\alpha \in (0,1)$, and prove a generalized version of Eq. (\ref{eq: gdm_direct_taylor}):
% \fontsize{8.5pt}{24pt}
% \begin{align}
% \label{eq: gdm_direct_taylor_extended}
%     \xi(t)
%     \ge \xi(t+\alpha)
%     +\frac{(1-\beta)(1-c)\alpha^2}{\eta}\Vert \bw(t+1)-\bw(t)\Vert^2.
% \end{align}
% \normalsize
%  The proof idea is that as long as Eq. (\ref{eq: gdm_direct_taylor_extended}) holds for all time in $[1,t+\alpha)$, the training loss across $[1,t+\alpha]$ will be smaller than $\mL(\bw(1))$, and the strict inequality in Eq. (\ref{eq: gdm_direct_taylor_extended}) holds. Consequently, there exists some small enough positive real $\varepsilon$, such that Eq. (\ref{eq: gdm_direct_taylor_extended}) holds for all time in $[1,t+\alpha+\varepsilon)$, and we are able to extend the feasible set where Eq. (\ref{eq: gdm_direct_taylor_extended}) holds.
%  The formal description of the generalized lemma and its corresponding proof is deferred to Appendix \ref{appen: gdm_sum_finite}.

By Lemma \ref{lem: update_rule_loss_gdm}, we have that $  \xi(t)$ is monotonously decreasing by gap $\frac{1-C_1}{\eta}\Vert \bw(t+1)-\bw(t)\Vert^2$. As $ \xi(1)=\mL(\bw(1))$ is a finite number, we have $\sum_{t=1}^{\infty}\Vert \bw(t+1)-\bw(t)\Vert^2<\infty$. By that $(1-\beta)\eta\nabla \mL(\bw(t))=(\bw(t+1)-\bw(t))-\beta(\bw(t)-\bw(t-1))$, it immediately follows that $\sum_{t=1}^{\infty}$ $\Vert \nabla \mL (\bw(t))\Vert^2<\infty$.

% \vspace{-7pt}
%  { The proof of Corollary \ref{lem: sum_of_squared_gradient_converge} can be found in  Appendix \ref{appen: gdm_sum_finite}. Resulted from Corollary \ref{lem: sum_of_squared_gradient_converge}, we have $ \Vert \nabla \mL(\bw(t))\Vert^2\rightarrow0$ as $t\rightarrow \infty$, which by the exponential-tailed loss assumption further leads to $\lim_{t\rightarrow \infty} \mL(\bw(t))=0$.}

 \textbf{Stage II. Parameter Dynamics.} The goal of this stage is to characterize the dynamics of the parameter and show that GDM asymptotically converges (in direction) to the max-margin solution $\hbw$. To see this, we define a residual term $\br(t)\overset{\triangle}{=}\bw(t) -\ln (t)\hbw-\tbw$ with some constant vector $\tbw$ (specified in Appendix \ref{subsubsec: form of the solution_gdm}). If we can show the norm of $\br(t)$ is bounded over the iterations, we complete the proof as $\ln (t)\hbw$ will then dominates the dynamics of $\bw(t)$.
 
 For simplicity, we use the continuous dynamics approximation of GDM \cite{sun2019non} to demonstrate why $\br(t)$ is bounded:
\small
\begin{equation}
    \frac{\beta}{1-\beta}\frac{\mathrm{d}^2 \bw(t)}{\mathrm{d} t^2}+ \frac{\mathrm{d} \bw(t)}{\mathrm{d} t}+\eta \nabla \mL(\bw(t))=0.\label{gdmflow}
\end{equation}
\normalsize
%Specifically
We start by directly examining the evolution of $\Vert \br(t)\Vert$, i.e.,
\begin{small}
 \begin{align*}
   &\frac{1}{2}\Vert\br(T)\Vert^2-\frac{1}{2}\Vert\br(1)\Vert^2=\int_{1}^T\frac{1}{2} \frac{\mathrm{d}\Vert \br(s) \Vert^2}{\mathrm{d} s} \mathrm{d} s
   \\
   =&\int_{1}^T\left\langle \br(s), - \eta\nabla \mL(\bw(s))-\frac{1}{s} \hbw\right\rangle\mathrm{d} s+\int_{1}^T\frac{\beta}{1-\beta}\left\langle \br(s), -\frac{\mathrm{d}^2 \bw(s)}{\mathrm{d} s^2}\right\rangle \mathrm{d} s,
 \end{align*}
\end{small}
which by integration by part leads to
\begin{tiny}
\begin{equation*}
    \operatorname{RHS}=\int_{1}^T\left\langle \br(s), - \eta\nabla \mL(\bw(s))-\frac{1}{s} \hbw\right\rangle\mathrm{d} s+\frac{\beta}{1-\beta}\left(\left\langle \br(T), -\frac{\mathrm{d} \bw(T)}{\mathrm{d} t}\right\rangle-\left\langle \br(1), -\frac{\mathrm{d} \bw(1)}{\mathrm{d} t}\right\rangle+\int_{1}^T\left\langle  \frac{\mathrm{d} \br(s)}{\mathrm{d} s}, \frac{\mathrm{d} \bw(s)}{\mathrm{d} s}\right\rangle \mathrm{d} s\right),
\end{equation*}
\end{tiny}
We then check the terms one by one:
\begin{itemize}
    \item $\int_{1}^T\left\langle \br(s), - \eta\nabla \mL(\bw(s))-\frac{1}{s} \hbw\right\rangle\mathrm{d} s$: This term also occurs in the analysis of GD \cite{soudry2018implicit}, and has been proved to be bounded;
    \item $\langle \br(T), -\frac{\mathrm{d} \bw(T)}{\mathrm{d} t}\rangle$: as shown in Stage I, $\frac{\mathrm{d} \bw(T)}{\mathrm{d} t}\rightarrow 0$ (i.e., $\bw(T)-\bw(T-1)$ $\rightarrow 0$ in the discrete case)  as $T\rightarrow \infty$. Thus, this term is $\boldsymbol{o}(\Vert \br(T)\Vert)$;
    \item $\int_{1}^T\left\langle  \frac{\mathrm{d} \br(s)}{\mathrm{d} s}, \frac{\mathrm{d} \bw(s)}{\mathrm{d} s}\right\rangle \mathrm{d} s$: finite due to $\frac{\mathrm{d} \br(s)}{\mathrm{d} s}=\frac{\mathrm{d} \bw(s)}{\mathrm{d} s}-\frac{1}{s}\hbw$ and  $\int_{1}^\infty\Vert\frac{\mathrm{d} \bw(s)}{\mathrm{d} s} \Vert^2 \mathrm{d} s$ is finite (i.e., $\sum_{t=1}^{\infty}$ $ \Vert \bw(t+1)-\bw(t) \Vert^2<\infty$ in the discrete case)  by Stage I.
\end{itemize}

Putting them together, we show that $\Vert\br(T)\Vert^2+\boldsymbol{o}(\Vert\br(T)\Vert)$ is upper bounded over the iterations, which immediately leads to that $\Vert\br(T)\Vert$ is bounded. Applying similar methodology to the discrete update rule, we have the following lemma (the proof can be found in Appendix \ref{subsubsec: form of the solution_gdm}). 
% \begin{lemma}
% Under the assumptions in Theorem 2, we have that for the continuous flow defined in Eq.(\ref{gdmflow}), $\frac{1}{2}\Vert\br(T)\Vert^2+$ $\frac{\beta}{1-\beta}\langle \br(T), -\frac{\mathrm{d} \bw(T)}{\mathrm{d} t}\rangle$ is upper bounded, which then leads to that $\frac{1}{2}\Vert\br(T)\Vert^2$ is upper bounded as $\frac{\mathrm{d} \bw(T)}{\mathrm{d} t} \rightarrow 0$.
% \end{lemma}
% \textit{Proof Sketch:}
% As the results from Stage I and \cite{soudry2018implicit} can be interpreted in the continuous limit as $\int_{1}^\infty\langle  \frac{\mathrm{d} \bw(s)}{\mathrm{d} s},$ $ \frac{\mathrm{d} \bw(s)}{\mathrm{d} s}\rangle \mathrm{d} s<\infty$ and $\int_{1}^\infty\langle \br(s), - \eta\nabla \mL(\bw(s))-\frac{1}{s} \hbw\rangle\mathrm{d} s<\infty$, we have 
% $\frac{1}{2}\Vert\br(T)\Vert^2+$ $\frac{\beta}{1-\beta}\langle \br(T), -\frac{\mathrm{d} \bw(T)}{\mathrm{d} t}\rangle$ is upper bounded, which then leads to that $\frac{1}{2}\Vert\br(T)\Vert^2$ is upper bounded as $\frac{\mathrm{d} \bw(T)}{\mathrm{d} t} \rightarrow 0$. The above methodology leads to the following claim for the discrete case (see Appendix \ref{subsubsec: form of the solution_gdm} for proofs):

\begin{lemma}
\label{claim: direction_convergence}
Define potential function $g:\mathbb{Z}^{+}\rightarrow\mathbb{R}$ as 
\begin{small}
\begin{align*}
    g(t)\overset{\triangle}{=}\frac{1}{2}\Vert \br(t)\Vert^2+\frac{\beta}{1-\beta}\langle \br(t),\bw(t)-\bw(t-1)\rangle.
\end{align*}
\end{small}
$g(t)$ is upper bounded, which further indicates $\Vert \br(t) \Vert $ is upper bounded.
\end{lemma}

\begin{remark}
Our technique for analyzing GDM here is essentially more complex and elaborate than that for GD in \cite{soudry2018implicit} due to the historical information of gradients GDM. The approach in \cite{soudry2018implicit} cannot be directly applied. It is worth mentioning that we provide a more easy-to-check condition for whether $\br(t)$ is bounded, i.e., "is $g(t)$ upper-bounded?". This condition can be generalized for other momentum-based implicit regularization analyses. E.g., SGDM and Adam later in this paper.
\end{remark}

\section{Tackle the difficulty brought by random sampling}
\label{subsec: main_sgdm}
In this section, we analyze the implicit regularization of SGDM. Parallel to GDM, we establish the following implicit regularization result for SGDM:
\begin{theorem}
\label{thm: sgdm_main}
Let Assumption \ref{assum: separable}, \ref{assum: exponential-tailed}, and \ref{assum: smooth}. (\textbf{S}) hold. Let $\beta \in [0,1)$ and $
     \eta< \frac{1}{\frac{H\beta\sigmax^3}{\sqrt{Nb}\gamma (1-\beta)}+\frac{H\sigmax^4}{2b\gamma^2}}$. Then, with arbitrary initialization point $\bw(1)$, SGDM (w/. r) satisfies
$\bw(t)-\ln(t) \hbw$ is bounded as $t\rightarrow \infty$ and $\lim_{t\rightarrow\infty}\frac{\bw(t)}{\Vert \bw(t) \Vert} =\frac{\hbw}{\Vert \hbw \Vert}$, almost surely (a.s.).
\end{theorem}
Similar to the GDM case, Theorem \ref{thm: sgdm_main} shows that the implicit regularization of SGDM under this setting is consistent with SGD (c.f. \cite{nacson2019stochastic} for the implicit regularization of SGD). This matches the observations in practice (c.f. Section \ref{sec: discussion} for details), and is later supported by our experiments (e.g., Figure \ref{fig: main_text}). We add two remarks on the learning rate upper bound and extension to SGDM (w/. r).

\begin{remark}[On the learning rate] Firstly, our learning rate upper bound $\frac{1}{\frac{H\beta\sigmax^3}{\sqrt{Nb}\gamma (1-\beta)}+\frac{H\sigmax^4}{2b\gamma^2}}$ exactly matches that of SGD $2\frac{b\gamma^2}{H\sigma^4_{max}}$ \cite{nacson2019stochastic} when $\beta=0$, and matches that of SGD in terms of the order of $\sigmax$, $H$, and $b$ when $\beta\ne 0$. This indicates our analysis is tight. Secondly, as the bound is monotonously increasing with respect to batch size $b$, Theorem \ref{thm: sgdm_main} also sheds light on the learning rate tuning, i.e., the larger the batch size is, the larger the learning rate is. 

\end{remark}

\begin{remark}(On SGDM (w/o. r))
\label{remark: mini_batch SGDM}
Theorem \ref{thm: sgdm_main} can be similarly extended to SGDM (w/o. r). We defer the detailed description of the corresponding theorem together with the proof to Appendix \ref{appen: mini_SGDM}.
\end{remark}

% \textcolor{red}{****Qi makes revision for the following part.****}

Next, we show the proof sketch for Theorem 3. The proof also contains two stages, where Stage II is similar to that for GDM. However, we highlight that \textbf{Stage I for SGDM is not a trivial extension of that for GDM and has its merit for other optimization analyses for SGDM}. %Despite the Stage II analysis being similar, the Stage I analysis of GDM no longer works for SGDM. 
Specifically, the methodology used to construct GDM's potential function fails for SGDM due to the random sampling. We defer a detailed discussion to Appendix \ref{subsec: explanation}. We then need to find a proper potential function for SGDM.
Inspired by SGD's simple update rule, %we start by examining the update rule of SGD, which only \textbf{contains the gradient information of the current step} and allows a simple potential function $\mL(\bw(t))$. Therefore, if the update rule of SGDM can also be 
we rearrange the update rule of SGDM such that only the gradient information of the current step is contained, i.e., %a simple potential function could be similarly constructed. Such an intuition motivates us to rearrange the update rule of SGDM as
\begin{small}
\begin{equation*}
    \frac{\bw(t+1)-\beta \bw(t)}{1-\beta}=\frac{\bw(t)-\beta \bw(t-1)}{1-\beta}-\eta \nabla \mL_{\bB(t)}(\bw(t)).
\end{equation*}
\end{small}
By defining $\bu(t)\triangleq \frac{\bw(t)-\beta \bw(t-1)}{1-\beta}=\bw(t)+\frac{\beta}{1-\beta}(\bw(t)-\bw(t-1))$, we have that $\bu(t)$ is close to $\bw(t)$ (differs by order of one-step update $\bw(t)-\bw(t-1)$), and the update rule of $\bu(t)$ only contains the current-step-gradient information $ \nabla \mL(\bw(t))$. We then select potential function as $\mL(\bu(t))$, and a simple Taylor's expansion directly leads to:
\begin{equation*}
    \mE[\mL(\bu(t+1))|\mF_t]\approx \mL(\bu(t))-\eta\langle \nabla \mL(\bw(t)), \nabla \mL(\bu(t)) \rangle\approx \mL(\bu(t))-\eta\Vert \nabla \mL(\bw(t))\Vert^2,
\end{equation*}
i.e., $\mL(\bu(t))$ is a proper potential function. We formalize the above discussion into the following lemma.

\begin{lemma}
\label{lem: sgdm_loss_update}
Let all conditions in Theorem \ref{thm: sgdm_main} hold. Then, there exists a positive constant $C_2$, such that
\small
\begin{equation*}
    \mE[ \mL(\bu(t+1))]\le \mL(\bu(1))-\sum_{s=1}^t C_2 \eta \mE\Vert \nabla \mL(\bw(s)) \Vert^2.
\end{equation*}
\normalsize
\end{lemma}
By letting $T\rightarrow \infty$ in the second claim of Lemma \ref{lem: sgdm_loss_update}, we have $\sum_{s=1}^\infty \mE\Vert \nabla \mL(\bw(s)) \Vert^2<\infty$, which is indeed what we need in Stage I.

\section{Analyze the effect of preconditioners}
\label{subsec: main_adam}

\subsection{Implicit regularization of deterministic Adam}
This section presents the implicit regularization of deterministic Adam, i.e., Adam without random sampling.

\begin{theorem}
\label{thm: adam_main}
Let Assumption \ref{assum: separable}, \ref{assum: exponential-tailed}, and \ref{assum: smooth}. (\textbf{D}) hold. Let $1>\beta_2>\beta_1^4\ge 0$, and the learning rate $\eta$ is a small enough constant (The upper bound of learning rate is complex, and we defer it to Appendix \ref{subsec: choice_learning_rate}).  Then, with arbitrary initialization point $\bw(1)$, {  deterministic Adam} (Eq.~\ref{eq: def_adam}) satisfies that 
$\bw(t)-\ln(t) \hbw$ is bounded as $t\rightarrow \infty$, and $\lim_{t\rightarrow\infty}\frac{\bw(t)}{\Vert \bw(t) \Vert} =\frac{\hbw}{\Vert \hbw \Vert}$.
\end{theorem}

\begin{remark}[On the $\beta_1$ and $\beta_2$ range]
Almost all existing literature assume a time-decaying hyperparameter choice of $\beta_1$ or $\beta_2$ (c.f., \cite{kingma2014adam,chen2018convergence}). On the other hand, our result proves that deterministic Adam converges with constant settings of $\beta_1$ and $\beta_2$, which agrees with the practical use.
\end{remark}
%\begin{remark}[On the "Contradictory" with existing experimental results]
\begin{remark}[(Discussion on the results in \cite{soudry2018implicit})]
\cite{soudry2018implicit} observe that, on a synthetic dataset, the direction of the output parameter by Adam still does not converge to the max-margin direction after $2\times 10^6$ iterations (but is getting closer). At the same time, GD seems to converge to the max-margin direction. This seems to contradict with Theorem \ref{thm: adam_main}. We reconduct the experiments (please refer to Appendix \ref{appen:ill} for details) and found (1). such a phenomenon occurs because a large learning rate is selected, using which GD will also stick in a direction close to the max-margin direction (but not equal); (2). the constructed synthetic dataset is ill-posed, as the non-support data is larger than the support data by order of magnitude (which is also rare in practice). On a well-posed dataset, we observe that both Adam and GD will converge to the max-margin solution rapidly (Figure \ref{fig: main_text}).
\end{remark}

We simply introduce the proof idea here and put the full proof in Appendix \ref{subsec: implicit_bias_adam}. The proof of deterministic Adam is very similar to that of GDM with minor changes in potential functions. Specifically, $\xi(t)$ in Lemma \ref{lem: update_rule_loss_gdm} is changed to
\begin{small}
\begin{equation*}
    \xi(t)\triangleq\mL(\bw(t))+\frac{1}{2}\frac{1-\beta_1^{t-1}}{\eta(1-\beta_1)}\left\Vert \sqrt[4]{\varepsilon\mathds{1}_d+\hbnu(t-1)}\odot (\bw(t)-\bw(t-1))\right\Vert ^2,
\end{equation*}
\end{small}
and $g(t)$ in Lemma \ref{claim: direction_convergence} is changed to
\begin{small}
\begin{align*}
   &g(t)\overset{\triangle}{=}\left\langle \br(t),(1-\beta_1^{t-1}) \sqrt{\varepsilon\mathds{1}_d+\hbnu(t-1)}\odot\left(\bw(t)-\bw(t-1)\right)\right\rangle
   \frac{\beta_1}{1-\beta_1}+\frac{\sqrt{\varepsilon}}{2}\Vert \br(t)\Vert^2.
\end{align*}
\end{small}
The rest of the proof then flows similarly to the GDM case.

\subsection{What if random sampling is added?}
\label{randomsam}
We have obtained the implicit regularization for GDM, SGDM, and deterministic Adam. One may wonder whether the implicit regularization of stochastic Adam can be obtained. Unfortunately, the gap can not be closed yet. This is because an implicit regularization analysis requires the knowledge of the loss dynamics, little of which, however, has been ever known even for stochastic RMSProp (i.e., Adam with $\beta_1=0$ in Eq. (\ref{eq: def_adam})) with constant learning rates. Specifically, the main difficulty lies in bounding the change of conditioner $\frac{1}{\sqrt{\varepsilon\mathds{1}_d+\hbnu(t)}}$ across iterations, which is required to make the drift term $\langle \nabla \mL(\bw(t)), \mE (\bw(t+1)-\bw(t))\rangle$ (derived by Taylor's expansion of the epoch start from $Kt$) negative  to ensure a non-increasing loss. 

On the other hand, if we adopt decaying learning rates $\eta_t=\frac{1}{\sqrt{t}}$, \cite{shi2021rmsprop} shows $\beta_2$ close enough to $1$, the following equation holds for stochastic RMSProp (w/o. r) (recall that $K\triangleq\frac{N}{b}$ is the epoch size)
\small
\begin{equation}
\label{eq: shi_main}
    \sum_{t=0}^T \frac{1}{\sqrt{t+1}}\left\Vert \nabla \mL\left(\bw\left(Kt\right)\right)\right\Vert=\mathcal{O}(\ln T).
\end{equation}
\normalsize

Based on this result, we have the following theorem for stochastic RMSProp (w/o. r):
\begin{theorem}
\label{thm: ib_rms}
Let Assumptions \ref{assum: separable}, \ref{assum: exponential-tailed}, and \ref{assum: smooth}. (S) hold. Let $\beta_2$ be close enough to $1$.
Then, with arbitrary initialization point
$\bw(1)$, stochastic RMSProp (w/o. r) converges to the $L^2$ max-margin solution.
\end{theorem}

\begin{remark}[On the decaying learning rate]
The decaying learning rate is a "stronger" setting compared to the constant learning rate, both in the sense that GDM, SGDM, and {  deterministic Adam} can be shown to converge to the max-margin solution following the same routine as Theorems \ref{thm: gdm_main}, \ref{thm: sgdm_main}, and \ref{thm: adam_main}, and in the sense that we usually adopt constant learning rate in practice.
\end{remark}

The proof is on the grounds of a novel characterization of the loss convergence rate derived from Eq. (\ref{eq: shi_main}), and readers can find the details in Appendix \ref{appen: random_shuffling_rmsprop}. 
Furthermore, the proof can be easily extended to the Stochastic Adaptive Heavy Ball (SAHB) algorithm  (proposed by \cite{tao2021role}), which can be viewed as a momentum version of RMSProp (but different from Adam) with the following update rule
\small
\begin{equation*}
    \bw(t+1)=\bw(t)-\bom(t),\bom(t)=\eta_t(1-\beta_1)\frac{\nabla \mL_{\bB(t)}(\bw(t))}{\sqrt{\varepsilon \mathds{1}_d+\hbnu (t)}}+\beta_1(\bw(t)-\bw(t-1)).
\end{equation*}
\normalsize

The proof of Theorem \ref{thm: ib_rms} can be easily extended to SAHB, based on the fact that $\bu(t)=\frac{\bw(t)-\beta_1\bw(t-1)}{1-\beta_1}$ has a simple update, i.e.,
\small
\begin{gather*}
\nonumber
    \bu(t+1)=-\eta_t\frac{\nabla \mL_{\bB(t)}(\bw(t))}{\sqrt{\varepsilon \mathds{1}_d+\hbnu (t)}}+\bu(t).
\end{gather*}  
\normalsize
and $\mL(\bu(t))$ can be used as a potential function just as the SGDM case.
\begin{figure}[htbp]
\centering
\vspace{-1mm}

\begin{minipage}{1.0\textwidth}
\centering
\begin{subfigure}{.45\textwidth}
          \centering
          \includegraphics[width=0.9\textwidth]{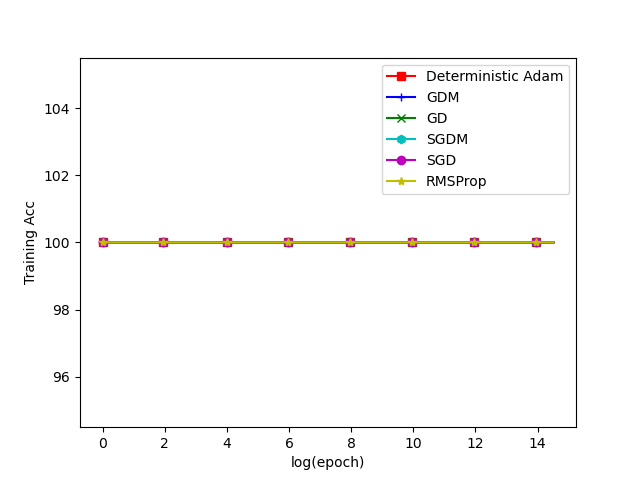}
          \caption{Training Accuracy}
        \end{subfigure}%
        \hspace{0.4mm}
\begin{subfigure}{.45\textwidth}
          \centering
          \includegraphics[width=0.9\textwidth]{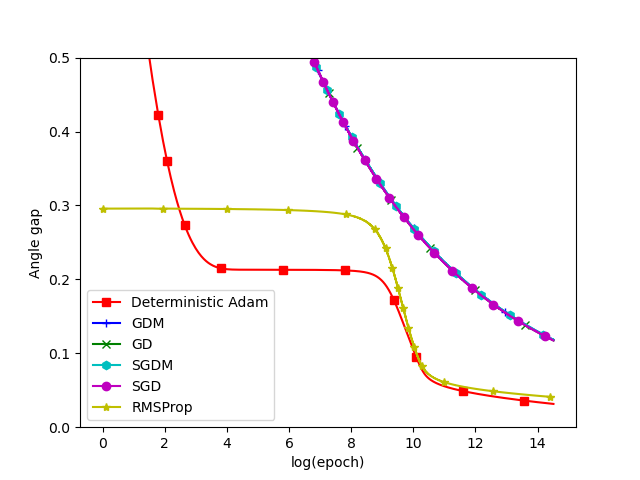}
          \caption{Angle gap between the parameter direction and the max margin solution}
\end{subfigure}%
\end{minipage}%
\centering
\caption{Comparison of the implicit regularization of (S)GD, (S)GDM, deterministic Adam and stochastic RMSProp. We use the synthetic dataset in \cite{soudry2018implicit} with learning rate $0.1$. Figure (b) shows (1). all the optimizers converge to the max margin solution, and (2). the asymptotic behaviors with \& without momentum are similar. The experimental observation support our theoretical results.}
\label{fig: main_text}
\end{figure}

\section{Discussions}
\label{sec: discussion}

\textbf{Consistency with the Experimental Results.} We conduct experiments to verify our theoretical findings. Specifically, we (1). run GD, GDM, SGD, SGDM, and Adam on a synthetic dataset to observe their implicit regularization; (2). run GD and Adam on ill-posed dataset proposed in \cite{soudry2018implicit} to verify Theorem \ref{thm: adam_main}; (3). run SGD and SGDM on neural networks to classify the MNIST dataset and compare their implicit regularization. The experimental observations stand with our theoretical results. Furthermore, it is worth mentioning that experimental phenomenons that adding momentum will not change the implicit regularization have also been observed by existing literature \cite{soudry2018implicit,nacson2019convergence,wang2021implicit}. 

\textbf{Influence of hyperparameters on convergence rates.} Our results can be further extended to provide a precise characterization of the influence of the hyperparameters $\eta$ and $\beta$ on the convergence rate of (S)GDM. Specifically, in Appendix \ref{appen: precise}, we show that the asymptotic convergence rate of (S)GDM is $C\frac{1}{\eta}\frac{1}{t}$, where $C$ is some constant independent of $\beta$ and $\eta$. Therefore, increase $\eta$ can lead to a faster convergence rate (with learning rate requirements in Theorems \ref{thm: gdm_main} and \ref{thm: sgdm_main} satisfied). However, changing $\beta$ does not affect the convergence rate, which is also observed in our experiments (e.g., Figure \ref{fig: main_text}). While this seems weird, as replacing gradient with momentum in deep learning often accelerates the training, we hypothesize that such an acceleration appears as the landscape of neural networks is highly non-convex and thus can not be observed in the case considered by this paper.

\textbf{Gap Between The Linear Model and Deep Neural Networks.} While our results only hold for the linear classification problem, extending the results to the deep neural networks is possible. Specifically, existing literature \cite{lyu2019gradient,wang2021implicit} provide a framework for deriving implicit regularization for deep homogeneous neural networks. However, the approach in \cite{lyu2019gradient,wang2021implicit} can not be trivially applied to the momentum-based optimizers, as their proofs require the specific gradient-based updates to lower bound a smoothed margin (c.f., Theorem 4.1, \cite{lyu2019gradient}). It remains an exciting work to see how our results can be expanded to GDM
and Adam for deep neural networks.

%% file: Tex Files/Conclusion.tex
\section{Conclusion}
This paper studies the implicit regularization of momentum-based optimizers in linear classification with exponential-tailed loss. Our results indicate that for SGD and the deterministic version of Adam, adding momentum will not influence the implicit regularization, and the direction of the parameter converges to the $L^2$ max-margin solution. Our theoretical results stand with existing experimental observations, and developed techniques such as the potential functions may inspire the analyses on other momentum-based optimizers. Motivated by the results and techniques for linear cases in this paper, it has the potential to extend them to the homogeneous neural network in the future. Another topic left for future work is to derive the implicit regularization of constant learning rate stochastic Adam. As discussed in Section~\ref{randomsam}, this topic is non-trivial, and it requires new techniques and assumptions to be developed.

%% file: Tex Files/Appendix.tex
\appendix
\newpage
\begin{center}
    {\Large \textbf{Supplementary Materials for \\``{   Does Momentum  Change the Implicit Regularization on Separable Data?''}}}
\end{center}
\section{Preparations}
\label{appen: preparation}
This section collect definitions and lemmas which will be used throughout the proofs.
\subsection{Characterization of the max-margin solution}
\label{subsec: characterization_max_margin}
This section collects several commonly-used characterization of the max-margin solution from \cite{nacson2019stochastic} and \cite{soudry2018implicit}.

To start with, we define support vectors and support set, which are two common terms in margin analysis. Recall that in the main text, we assume that without the loss of generality, $\by_i=1$, $\forall i\in \{1,\cdots,N\}$.

\begin{definition}[Support vectors and support set] For any $i\in [N]$, $\bx_i$ is called a support vector of the dataset $\bS$, if  {  
\begin{equation*}
    \langle \bx_i, \hbw \rangle=1.
\end{equation*}
Correspondingly, $\bx_i$ is called a non-support vector if $ \langle \bx_i, \hbw \rangle>1$. The support set of $\bS$ is then defined as
\begin{equation*}
    \bS_{s}=\{\bx\in \bS:\langle \bx, \hbw \rangle=1 \}.
\end{equation*}
}
\end{definition}

The following lemma delivers $\hbw$ as an linear combination of support vectors.

\begin{lemma}[Lemma 12, \cite{soudry2018implicit}]
\label{lem: structure of max-margin vector}
For almost every datasets $S$, there exists a unique vector $\bv=(\bv_1,\cdots,\bv_N)$, such that $\hbw$ can be represented as 
\begin{equation}
\label{eq: structure of hbw}
    \hbw= \sum_{i=1}^N \bv_i \bx_i,
\end{equation}
where $\bv$ satisfies $\bv_i=0$ if $\bx_i\notin\bS_{s}$, and $\bv_i>0$ if $\bx_i\in\bS_{s}$. Furthermore, the size of $\tbS_{s}$ is at most $d$.
\end{lemma}

By Lemma \ref{lem: structure of max-margin vector}, we further have the following corollary:
\begin{corollary}
\label{coro:definition_of_tbw}
For almost every datasets $S$, the unique $\bv$ given by Lemma \ref{lem: structure of max-margin vector} further satisfies that for any positive
constant ${   C_3}$, there exists a non-zero vector {   $\tbw$}, such that, $\bx_i\in\bS_{s}$, we have
\begin{equation}
\label{eq: represent_v}
    {   C_3} e^{-\langle \bx_i,\tbw\rangle}=\bv_i.
\end{equation}
\end{corollary}
\begin{proof}
For almost every datasets $ \bS$, any subsets with size $d$ of $\bS$ is linearly independent. Since $\tbS_{s}$ has size no larger than $d$ (by Lemma \ref{lem: structure of max-margin vector}), and Eq. (\ref{eq: represent_v}) is equivalent to linear equations, the proof is completed.
\end{proof}

For the stochastic case, we will also need the following lemma when we calculate the form of parameter at time $t$.

\begin{lemma}[Lemma 5, \cite{nacson2019stochastic}]
\label{lem: define_n_t}
Let $\bB(s)$ be the random subset used in SGDM (w/. r). Almost surely, there exists a vector $\cbw$
\begin{equation*}
    \frac{N}{b} \sum_{s=1}^{t-1} \frac{1}{s} \sum_{\bx_i \in\bB(s)\cap\bS_s} \bv_i \bx_i=\ln \left(\frac{bt}{N}\right) \hbw + \bn(t)+\cbw,
\end{equation*}
where $\bn(t)$ satisfies $\Vert \bn(t) \Vert = \bo(t^{-0.5+\varepsilon})$ for any $\varepsilon>0$, and $\Vert \bn(t+1)-\bn(t)\Vert =O(t^{-1})$. As for SGDM (w/o. r),  the a.s. condition can be removed.

\end{lemma}

\subsection{Preparations of the optimization analysis}
\label{subsec: preparation_optimization}
This section collects technical lemmas which will be used in latter proofs. We begin with a lemma bounding the smooth constants if the loss is bounded.

\begin{lemma}
\label{lem: smooth_guarantee}
If loss $\ell$ satisfies (D) in Assumption \ref{assum: smooth}, then for any $\bw_0$, if $\mL(\bw)\le \mL(\bw_0)$, then we have $\mL$ is $\sigmax^2H_{s_0}$ smooth at point $\bw$, where $s_0=\ell^{-1} (N\mL(\bw_0))$. Furthermore, $\mL$ is globally $\sigmax^2H_{s_0}$ smooth over the set $\{\bw:\mL(\bw)\le \mL(\bw_0)\}$.
\end{lemma}
\begin{proof}
Since $\ell$ is positive, we have $\forall i\in [N]$,
\begin{equation*}
    \frac{\tell(\bw,\bx_i)}{N}< \frac{\sum_{j=1}^N\tell(\bw,\bx_j)}{N}=\mL(\bw)\le \mL(\bw_0),
\end{equation*}
which leads to $\tell(\bw,\bx_i)< N\mL(\bw_0)$, and $\ell$ is $H_{s_0}$ smooth at $\langle\bw, \bx_i\rangle$.

Furthermore, since $\nabla_{\bw} \tell(\bw,\bx_i)=\nabla_{\bw} \ell(\langle\bw,\bx_i\rangle)=\ell'(\langle\bw,\bx_i\rangle)\bx_i$, for any two parameters $\bw_1$ and $\bw_2$ close enough to $\bw$,
\begin{align*}
    &\Vert \nabla_{\bw} \mL (\bw_1)-\nabla_{\bw} \mL(\bw_2)\Vert = \left\Vert \sum_{\bx\in\bS}(\ell'(\langle \bw_1, \bx\rangle)-\ell'(\langle \bw_2, \bx\rangle))\bx\right\Vert
    \\
    \le &\sigmax\sqrt{\sum_{\bx\in\bS}(\ell'(\langle \bw_1, \bx\rangle)-\ell'(\langle \bw_2, \bx\rangle))^2}\le \sigmax H_{s_0}\sqrt{\sum_{\bx\in\bS}(\langle \bw_1-\bw_2, \bx\rangle)^2}
    \\
    \le & \sigmax^2H_{s_0}\Vert \bw_1-\bw_2\Vert.
\end{align*}

Now if $\bw_1$ and $\bw_2$ both belong to $\{\bw: \mL(\bw)\le \mL(\bw_0)\}$, we have
for any $\bx_i\in \bS$, $\langle \bw_1, \bx_i\rangle>\ell^{-1}(N\mL(\bw_0))$, and $\langle \bw_2, \bx_i\rangle>\ell^{-1}(N\mL(\bw_0))$. Following the same routine as the locally smooth proof, we complete the second argument.

The proof is completed.
\end{proof}

Based on Assumption \ref{assum: exponential-tailed}, we also have the following lemma characterizing the relationship between loss $\ell$ and its derivative $\ell'$ when $x$ is large enough.

\begin{lemma}
\label{lem: equivalence_loss_derivative}
Let loss $\ell$ satisfy Assumption \ref{assum: exponential-tailed}. Then, there exists an large enough $x_0$ and a positive real $K$, such that, $\forall x>x_0$, we have
\begin{equation*}
    -\frac{1}{4} \ell'(x)\le \ell (x)\le -4\ell'(x).
\end{equation*}
\end{lemma}

\begin{proof}
    By Assumption \ref{assum: exponential-tailed}, there exists a large enough $x_0$, such that $\forall x>x_0$, we have
    \begin{equation}
    \label{eq: derivative_ell_bound}
        \frac{1}{2} e^{-x}\le -\ell'(x)\le 2 e^{-x}.
    \end{equation}
    
    On the other hand, as $\lim_{t\rightarrow \infty} \ell(x)=0$, we have
    \begin{align*}
        \ell(x)=\int_{s=x}^{\infty} -\ell'(s) \mathrm{d} s,
    \end{align*}
    which by Eq. (\ref{eq: derivative_ell_bound}) leads to 
    \begin{equation*}
       \frac{1}{2}e^{-x}=\frac{1}{2}\int_{x}^{\infty} e^{-s} \mathrm{d} s\le  \ell(x)\le 2\int_{x}^{\infty} e^{-s} \mathrm{d} s=2e^{-x}.
    \end{equation*}
    
    The proof is completed.
\end{proof}

By Lemma \ref{lem: equivalence_loss_derivative}, we immediately get the following corollary:
\begin{corollary}
\label{coro:gradient_small_equivalence}
Let loss $\ell$ satisfy Assumption \ref{assum: exponential-tailed}. Then, there exist positive reals $C_{g}$ and $C_{l}$, such that, for any $\bw\in \mathbb{R}^d$ satisfying either $\Vert \nabla \mL(\bw) \Vert\le C_{g}$ or $\mL(\bw)\le C_l$, we have
\begin{equation*}
    \frac{\gamma}{4} \mL(\bw)\le \Vert \nabla\mL(\bw)\Vert \le 4\mL(\bw).
\end{equation*}
\end{corollary}
\begin{proof}
We start with the case $\Vert \nabla \mL(\bw) \Vert\le C_{g}$. By simple calculation, we have
    \begin{equation}
    \label{eq: coro_gradient_small_equivalence_mid_1}
        \Vert \nabla\mL(\bw)\Vert= \frac{1}{N} \left\Vert \sum_{i=1}^N \ell'  (\langle \bw,\bx_i \rangle)\bx_i \right\Vert\le -\frac{\sigmax}{N} \sum_{i=1}^N \ell'  (\langle \bw,\bx_i \rangle),
    \end{equation}
    and
    \begin{equation}
    \label{eq: coro_gradient_small_equivalence_mid_2}
         \Vert \nabla\mL(\bw)\Vert \Vert \hbw\Vert \ge \langle \nabla \mL(\bw),\hbw\rangle\ge -\frac{1}{N}\sum_{i=1}^N \ell'(\langle \bw,\bx_i\rangle).
    \end{equation}
    By Assumption \ref{assum: exponential-tailed}, we have there exists a constant $C_g'$, s.t., any $x$ with $-\ell'(x)>C_g'$ satisfies $x>x_0$. Let $C_g=\frac{C_g'\gamma}{N}$. We then have if $ \Vert \nabla\mL(\bw)\Vert\le C_g$, then $\langle \bw,\bx_i\rangle>x_0$ ($\forall i$), and thus $4\ell(\langle \bw,\bx_i\rangle)\ge -\ell'(\langle \bw,\bx_i\rangle)\ge \frac{1}{4} \ell(\langle \bw,\bx_i\rangle) $. Combing Eqs. (\ref{eq: coro_gradient_small_equivalence_mid_1}) and  (\ref{eq: coro_gradient_small_equivalence_mid_2}), we then have
    \begin{equation*}
       4\mL(\bw) =4\sum_{i=1}^N \ell(\langle \bw,\bx_i\rangle)\ge\Vert \nabla\mL(\bw)\Vert \ge \frac{\gamma}{4N}\sum_{i=1}^N \ell(\langle \bw,\bx_i\rangle)=\frac{\gamma}{4} \mL(\bw).
    \end{equation*}
    
    Similarly, as for the case $\mL(\bw)\le C_l$, we have there exists a constant   $C_l'$, s.t., any $x$ with $\ell(x)< C_l'$ satisfies $x>x_0$. Let $C_l=\frac{C_l\sigmax}{N}$ and the rest of the proof follows the same routine as the first case.
    
    The proof is completed.
\end{proof}

The following lemma bridges the second moment of $\nabla \mL_{\bB(t)}$ with its squared first moment.

\begin{lemma}
\label{lem: first_second}
Let the dataset $\bS$ satisfies the separable assumption \ref{assum: separable}. Let $\bB$  be a random subset of $\bS$ with size $b$ sampled independently and uniformly without replacement. Then, at any point $\bw$, we have
\begin{equation*}
   \Vert \nabla \mL(\bw) \Vert^2\le\mathbb{E}_{\bB}\left[\Vert \nabla \mL_{\bB}(\bw)  \Vert^2\right]\le  \frac{N\sigmax^2}{\gamma^2b} \Vert \nabla \mL(\bw) \Vert^2.
\end{equation*}
\end{lemma}
\begin{proof}
   To start with, notice that
   \begin{equation*}
       \Vert \nabla \mL(\bw) \Vert =\Vert \mE_{\bB} \nabla\mL_{\bB} (\bw) \Vert \le \mE_{\bB}\Vert  \nabla\mL_{\bB} (\bw) \Vert.
   \end{equation*} Therefore, the first inequality can be directly obtained by Cauchy-Schwartz's inequality. To prove the second inequality, we first calculate the explicit form of $\nabla \mL_{\bB}(\bw)$.
    \begin{align*}
        &\Vert \nabla \mL_{\bB}(\bw) \Vert^2
        = \frac{1}{b^2}\left\Vert \sum_{\bx\in \bB}\nabla \tell (\bw,\bx) \right\Vert^2 
        =\frac{1}{b^2}\left\Vert \sum_{\bx\in \bB} \ell' (\langle \bw,\bx\rangle)\bx \right\Vert^2
        \le\frac{\sigmax^2}{b^2} \sum_{\bx\in \bB}  \ell' (\langle \bw,\bx\rangle)^2.
    \end{align*}
    
    Therefore, 
    \begin{align}
        \mathbb{E}_{\bB}\Vert \nabla \mL_{\bB}(\bw) \Vert^2
        \le \frac{\sigmax^2}{Nb} \sum_{\bx\in \bS}  \ell' (\langle \bw,\bx\rangle)^2
   \label{eq: second_momen}
        \le \frac{\sigmax^2}{Nb} \left(\sum_{\bx\in \bS}  \ell' (\langle \bw,\bx\rangle)\right)^2.
    \end{align}
    
    On the other hand,
    \begin{align*}
        &\Vert \nabla \mL(\bw) \Vert
        =\frac{1}{N} \left\Vert \sum_{\bx\in \bS} \ell' (\langle \bw,\bx\rangle)\bx \right\Vert 
        \\
        \ge &\frac{1}{N} \left \langle \sum_{\bx\in \bS}\ell' (\langle \bw,\bx\rangle)\bx, -\frac{\hbw}{\Vert \hbw \Vert} \right\rangle
        \overset{(\star)}{\ge} \frac{\gamma}{N}\sum_{\bx\in \bS} \ell' (\langle \bw,\bx\rangle)  
    \end{align*}
    where {   Eq. ($\star$) is due to $\forall \bx\in \bS  $}, $\langle\bx, -\hbw\rangle\ge 1$ and $\ell'<0$.
    
    Therefore, 
    \begin{align}    \label{eq: first_momen}
        \Vert \nabla \mL(\bw) \Vert^2
        \ge  \frac{\gamma^2}{N^2}\left(\sum_{\bx\in \bS} \ell' (\langle \bw,\bx\rangle) \right)^2.
    \end{align}
    
    The proof is completed by putting Eqs. (\ref{eq: second_momen}) and  (\ref{eq: first_momen}) together.
\end{proof}

In the following lemma, we show the updates of GDM, Adam, and SGDM are all non-zero.
\begin{lemma}
\label{lem: non-zero}
 Regardless of GDM, Adam, or SGDM, the updates of all steps are non-zero, i.e.,
    \begin{equation*}
        \Vert \bw(t+1)-\bw(t) \Vert>0, \forall t>1. 
    \end{equation*}
\end{lemma}
\begin{proof}
We start with the alternative forms of the update rule of GDM, Adam, and SGDM using the gradients along the trajectory respectively. For GDM, by Eq. (\ref{eq: def_gdm}), the update rule can be written as 
\begin{equation}
\label{eq: equivalent form gdm}
    \bw(t+1)-\bw(t)=-\eta(1-\beta)\left(\sum_{s=1}^t \beta^{t-s}\nabla \mL(\bw(s))\right).
\end{equation}

Similarly, the update rule of SGDM can be written as
\begin{equation}
    \label{eq: equivalent form sgdm}
    \bw(t+1)-\bw(t)=-\eta(1-\beta)\left(\sum_{s=1}^t \beta^{t-s}\nabla \mL_{\bB(s)}(\bw(s))\right),
\end{equation}
while the update rule of Adam can be given as
\begin{equation}
    \label{eq: equivalent form adam}
    \bw(t+1)-\bw(t)=-\eta\frac{\sum_{s=1}^t \frac{1-\beta_1}{1-\beta_1^s} \beta_1^{t-s}\nabla \mL(\bw(s))}{\sqrt{\varepsilon \mathbf{1}_d+\sum_{s=1}^t \frac{1-\beta_2}{1-\beta_2^s} \beta_2^{t-s}(\nabla \mL(\bw(s)))^2 }}.
\end{equation}

On the other hand, by the definition of empirical risk $\mL$, the gradient of $\mL$ at point $\bw$ can be given as
\begin{equation}
\label{eq: form of gradient}
    \nabla \mL(\bw)=\frac{\sum_{i=1}^N\ell'(\langle \bw,\bx_i\rangle) \bx_i}{N}.
\end{equation}
By Eq. (\ref{eq: form of gradient}) and Eq. (\ref{eq: equivalent form gdm}), we further have for GDM,
\begin{equation}
\label{eq: gdm_detailed_expansion}
    \bw(t+1)-\bw(t)=-\eta(1-\beta)\left(\sum_{s=1}^t \beta^{t-s}\frac{\sum_{i=1}^N\ell'(\langle \bw(s),\bx_i\rangle) \bx_i}{N}\right).
\end{equation}

By Assumption \ref{assum: separable}, there exists a non-zero parameter $\hbw$, such that, $\langle \hbw, \bx_i\rangle >0$, $\forall i$. Therefore, by executing inner product between  Eq. (\ref{eq: gdm_detailed_expansion}) and $\hbw$, we have 
\begin{align*}
    &\Vert \bw(t+1)-\bw(t)\Vert \Vert \hbw\Vert \ge\langle \bw(t+1)-\bw(t), \hbw\rangle 
    \\
    =& -(1-\beta)\eta\left(\sum_{s=1}^t \beta^{t-s}\frac{\sum_{i=1}^N\ell'(\langle \bw(s),\bx_i\rangle) \langle \bx_i, \hbw\rangle}{N}\right)\overset{(*)}>0,
\end{align*}
where Eq. $(*)$ is due to $\ell'<0$. This complete the proof for GDM.

Similarly, for SGDM, we have 
\begin{equation*}
    \Vert \bw(t+1)-\bw(t)\Vert \Vert \hbw\Vert \ge -\eta(1-\beta)\left(\sum_{s=1}^t \beta^{t-s}\frac{\sum_{(\bx,\by)\in \bB}\ell'(\langle \bw(s),\by\bx\rangle) \langle \by\bx, \hbw\rangle}{b}\right)>0,
\end{equation*}
which completes the proof of SGDM.

For Adam, we have 
\begin{align*}
    &\Vert \bw(t+1)-\bw(t)\Vert \left\Vert \hbw \odot \sqrt{\varepsilon \mathbf{1}_d+\sum_{s=1}^t \frac{1-\beta_2}{1-\beta_2^s} \beta_2^{t-s}(\nabla \mL(\bw(s)))^2 }\right \Vert 
    \\
    \ge & -\left\langle  \hbw \odot \sqrt{\varepsilon \mathbf{1}_d+\sum_{s=1}^t \frac{1-\beta_2}{1-\beta_2^s} \beta_2^{t-s}(\nabla \mL(\bw(s)))^2 }, \eta\frac{\sum_{s=1}^t \frac{1-\beta_1}{1-\beta_1^s} \beta_1^{t-s}\nabla \mL(\bw(s))}{\sqrt{\varepsilon \mathbf{1}_d+\sum_{s=1}^t \frac{1-\beta_2}{1-\beta_2^s} \beta_2^{t-s}(\nabla \mL(\bw(s)))^2 }} \right\rangle
    \\
    =&\left\langle  \hbw , \eta\sum_{s=1}^t \frac{1-\beta_1}{1-\beta_1^s} \beta_1^{t-s}\nabla \mL(\bw(s) \right\rangle
    \\
    =&-\eta\left(\sum_{s=1}^t\frac{1-\beta_1}{1-\beta_1^s} \beta_1^{t-s}\frac{\sum_{i=1}^N\ell'(\langle \bw(s),\bx_i\rangle) \langle \bx_i, \hbw\rangle}{N}\right)>0,
\end{align*}
which completes the proof of Adam.

The proof is completed.
\end{proof}

\section{Implicit regularization of GD/SGD with momentum}
\label{appen: implicit_bias_gd_sgd}
This section collects the proof of the implicit regularization of gradient descent with momentum and stochastic gradient descent with momentum. The analyses of this section hold for almost every dataset, and the "almost every" constraint is further moved in Section \ref{appen: every_data_set}.
\subsection{Implicit regularization of GD with Momentum}
\label{appen: implicit_bias_gd}

This section collects the proof of Theorem \ref{thm: gdm_main}.

\subsubsection{Proof of the sum of squared gradients converges}
\label{appen: gdm_sum_finite}
To begin with, we will prove the sum of squared norm of gradients along the trajectory is finite for gradient descent with momentum. To see this, we first define
the continuous-time update rule as
\begin{equation*}
    \bw(t+\alpha)-\bw(t)=\alpha (\bw(t+1)-\bw(t)), \forall t\in\mathbb{Z}^{+}, \forall \alpha\in[0,1].
\end{equation*}

We then prove a generalized case of Lemma \ref{lem: update_rule_loss_gdm} for any $\bw(t+\alpha)$.
\begin{lemma}[Lemma \ref{lem: update_rule_loss_gdm}, extended]
\label{lem: update_rule_loss_gdm_extended}
Let all conditions in Theorem \ref{thm: gdm_main} hold. We then have
\begin{align}
\nonumber
    \mL(\bw(t))+\frac{\beta}{2\eta(1-\beta)} \Vert \bw(t)-\bw(t-1)\Vert^2
    \ge& \mL(\bw(t+\alpha))+\frac{\beta}{2\eta(1-\beta)}\alpha^2 \Vert \bw(t+1)-\bw(t)\Vert^2
    \\
\label{eq: gdm_direct_taylor_extended}
    &+\frac{(1-{   C_1})\alpha^2}{\eta}\Vert \bw(t+1)-\bw(t)\Vert^2,
\end{align}
where ${   C_1}$ is a positive real such that $\eta=2\frac{N}{H_{s_0}\sigma_{max}^2} {   C_1}$ and $s_0\overset{\triangle}{=}\ell^{-1}(N\mL(\bw_1))$.
\end{lemma}

\begin{proof}[Proof of Lemma \ref{lem: update_rule_loss_gdm_extended}]

 We prove this lemma by reduction to absurdity.

Concretely, let $t^*$ be the smallest positive integer time such that there exists an $\alpha\in [0,1]$, such that Eq. (\ref{eq: gdm_direct_taylor_extended}) doesn't hold. Let $\alpha^*=\inf\{\alpha\in [0,1]: Eq.~(\ref{eq: gdm_direct_taylor_extended})~doesn't~hold~for~(t^*,\alpha)\}$. By continuity, Eq. (\ref{eq: gdm_direct_taylor_extended}) holds for $(t^*,\alpha^*)$.

We further divide the proof into two cases depending on the value of $\alpha^*$.

\textbf{Case 1: $\alpha^*=0$}: For any $ t^*>t\ge 1$, we have Eq. (\ref{eq: gdm_direct_taylor_extended}) holds for $(t,1)$. Specifically, we have 
\begin{equation*}
    \mL(\bw(t))+\frac{\beta}{2(1-\beta)\eta} \Vert \bw(t)-\bw(t-1)\Vert^2
    \ge \mL(\bw(t+1))+\frac{\beta}{2(1-\beta)\eta} \Vert \bw(t+1)-\bw(t)\Vert^2,
\end{equation*}
which further leads to 
\begin{equation*}
      \mL(\bw(1))=\mL(\bw(1))+\frac{\beta}{2(1-\beta)\eta} \Vert \bw(1)-\bw(0)\Vert^2
    \ge \mL(\bw(t^*))+\frac{\beta}{2(1-\beta)\eta} \Vert \bw(t^*)-\bw(t^*-1)\Vert^2.
\end{equation*}
Since $\frac{\beta}{2\eta} \Vert \bw(t^*)-\bw(t^*-1)\Vert^2$ is non-negative, we have
\begin{equation*}
      \mL(\bw(1))
    \ge \mL(\bw(t^*)).
\end{equation*}

By Lemma \ref{lem: smooth_guarantee}, we have $\mL$ is $H_{s_0}$ smooth at $\bw(t^*)$. Therefore, by Taylor's expansion for $\mL$ at point $\bw(t^*)$, we have for small enough $\alpha>0$
\begin{small}
\begin{align}
\nonumber
    &\mL(\bw(t^*+\alpha))
    \\
\nonumber
    \le & \mL(\bw(t^*))+\langle\nabla \mL(\bw(t^*)), \bw(t^*+\alpha)-\bw(t^*) \rangle+\frac{H_{s_0}\sigmax^2}{2N}\Vert \bw(t^*+\alpha)-\bw(t^*)\Vert^2
    \\
\nonumber
    = &\mL(\bw(t^*))+\alpha\langle\nabla \mL(\bw(t^*)), \bw(t^*+1)-\bw(t^*) \rangle+\frac{H_{s_0}\alpha^2\sigmax^2}{2N}\Vert \bw(t^*+1)-\bw(t^*)\Vert^2
    \\
\nonumber
    \overset{(*)}{=}& \mL(\bw(t^*))+\alpha\left\langle\frac{1}{(1-\beta)\eta} (\beta(\bw(t^*)-\bw(t^*-1))-(\bw(t^*+1)-\bw(t^*))), \bw(t^*+1)-\bw(t^*) \right\rangle
     \\
\nonumber
    +&\frac{H_{s_0}\alpha^2\sigmax^2}{2N}\Vert \bw(t^*+1)-\bw(t^*)\Vert^2
   \end{align}
    \begin{align}
\nonumber
    =&\mL(\bw(t^*))+\frac{\alpha\beta}{(1-\beta)\eta}\left\langle (\bw(t^*)-\bw(t^*-1), \bw(t^*+1)-\bw(t^*) \right\rangle+\left(\frac{H_{s_0}\alpha^2\sigmax}{2N}-\frac{\alpha}{(1-\beta)\eta}\right)\Vert \bw(t^*+1)-\bw(t^*)\Vert^2
    \\
\nonumber
    \overset{(**)}{\le}& \mL(\bw(t^*))+\frac{\alpha\beta}{2(1-\beta)\eta}\Vert (\bw(t^*)-\bw(t^*-1)\Vert^2+\frac{\alpha\beta}{2(1-\beta)\eta}\Vert (\bw(t^*+1)-\bw(t^*)\Vert^2 
    \\
\nonumber
    +&\left(\frac{H_{s_0}\alpha^2\sigmax^2}{2N}-\frac{\alpha}{(1-\beta)\eta}\right)\Vert \bw(t^*+1)-\bw(t^*)\Vert^2
    \\
\nonumber
    = & \mL(\bw(t^*))+  {\frac{\alpha\beta}{2(1-\beta)\eta}\Vert (\bw(t^*)-\bw(t^*-1)\Vert^2+\left(\frac{\alpha\beta}{2(1-\beta)\eta}-\frac{\alpha}{(1-\beta)\eta}+\frac{H_{s_0}\alpha^2\sigmax^2}{2N}\right)\Vert (\bw(t^*+1)-\bw(t^*)\Vert^2} 
    \\
    \nonumber
        = & \mL(\bw(t^*))+  {\frac{\beta}{2(1-\beta)\eta}\Vert (\bw(t^*)-\bw(t^*-1)\Vert^2-\frac{(1-\alpha)\beta}{2(1-\beta)\eta}\Vert \bw(t^*)-\bw(t^*-1)\Vert^2}
     \\
     \nonumber
       &  {+ \left(\frac{\alpha\beta}{2(1-\beta)\eta}-\frac{\alpha}{(1-\beta)\eta}+\frac{H_{s_0}\alpha^2\sigmax^2}{2N}\right)\Vert (\bw(t^*+1)-\bw(t^*)\Vert^2} 
    \\
    \nonumber
    \overset{(\diamond)}{\le}&  \mL(\bw(t^*))+\frac{\beta}{2(1-\beta)\eta} \Vert \bw(t^*)-\bw(t^*-1)\Vert^2
     -\frac{\beta}{2(1-\beta)\eta}\alpha^2 \Vert \bw(t^*+1)-\bw(t^*)\Vert^2
    \\
    \label{eq: disturbance at t}
    &-\frac{(1-{   C_1})\alpha^2}{\eta}\Vert \bw(t^*+1)-\bw(t^*)\Vert^2,
\end{align}
\end{small}
where Eq. $(*)$ is due to a simple rearrangement of the update rule of gradient descent with momentum (Eq. (\ref{eq: def_gdm})), i.e.,
\begin{equation}
    \label{eq: rearrange_gdm}
    \nabla\mL(\bw(t))=\frac{1}{(1-\beta)\eta}(\beta (\bw(t)-\bw(t-1))-(\bw(t+1)-\bw(t))), \forall t\ge 1,
\end{equation}
 Inequality $(**)$ is due to Cauchy Schwarz's inequality and arithmetic-geometric average inequality, and Inequality $(\diamond)$ is due to 
{   \begin{align*}
    & -\frac{(1-\alpha)\beta}{2(1-\beta)\eta}\Vert (\bw(t^*)-\bw(t^*-1)\Vert^2+ \left(\frac{\alpha\beta}{2(1-\beta)\eta}-\frac{\alpha}{(1-\beta)\eta}+\frac{H_{s_0}\alpha^2\sigmax^2}{2N}\right)\Vert (\bw(t^*+1)-\bw(t^*)\Vert^2
    \\
    =&-\frac{(1-\alpha)\beta}{2(1-\beta)\eta}\Vert (\bw(t^*)-\bw(t^*-1)\Vert^2+\mathcal{O}(\alpha)
    \\
    \le&
     -\frac{\beta}{2(1-\beta)\eta}\alpha^2 \Vert \bw(t^*+1)-\bw(t^*)\Vert^2
    -\frac{(1-{   C_1})\alpha^2}{\eta}\Vert \bw(t^*+1)-\bw(t^*)\Vert^2.
 \end{align*}
 
 Here the inequality is due to that $-\frac{(1-\alpha)\beta}{2\eta}\Vert (\bw(t^*)-\bw(t^*-1)\Vert^2$ tend to $-\frac{\beta}{2\eta}\Vert (\bw(t^*)-\bw(t^*-1)\Vert^2$ as $\alpha$ tend to zero, which is a negative constant by Lemma \ref{lem: non-zero}, and $ -\frac{\beta}{2(1-\beta)\eta}\alpha^2 \Vert \bw(t^*+1)-\bw(t^*)\Vert^2
    -\frac{(1-{   C_1})\alpha^2}{\eta}\Vert \bw(t^*+1)-\bw(t^*)\Vert^2$ is $\mathcal{O}(\alpha^2)$.
 }

Eq. (\ref{eq: disturbance at t}) indicates Eq. (\ref{eq: gdm_direct_taylor_extended}) holds at $(t^*,\alpha)$ for $\alpha>0$ is small enough, which contradicts to $\alpha^*=0$.

\textbf{Case 2: $\alpha^*\ne 0$}: Same as $\textbf{Case 1}$, we have for any $1\le t<t^*$,
\begin{equation*}
    \mL(\bw(t))+\frac{\beta}{2(1-\beta)\eta} \Vert \bw(t)-\bw(t-1)\Vert^2
    \ge \mL(\bw(t+1))+\frac{\beta}{2(1-\beta)\eta} \Vert \bw(t+1)-\bw(t)\Vert^2,
\end{equation*}
which further leads to 
\begin{equation}
\label{eq: loss_bound_t}
      \mL(\bw(1))
    \ge \mL(\bw(t^*))+\frac{\beta}{2\eta}\Vert \bw(t^*)-\bw(t^*-1)\Vert^2.
\end{equation}
On the other hand, by the definition of $\alpha^*$, we have for any $0\le\alpha<\alpha^*$, we have Eq. (\ref{eq: gdm_direct_taylor_extended}) holds for $(t^*,\alpha)$, which by continuity further leads to Eq. (\ref{eq: gdm_direct_taylor_extended}) holds for $(t^*,\alpha^*)$. Therefore, $\alpha^*<1$, otherwise, Eq. (\ref{eq: gdm_direct_taylor_extended}) holds for $(t^*,\alpha)$, $ \forall\alpha \in [0,1]$ which contradicts the definition of $t^*$.

Combining Eq. (\ref{eq: gdm_direct_taylor_extended}) with $(t^*,\alpha)$ and Eq. (\ref{eq: loss_bound_t}), we further have
\begin{equation*}
    \mL(\bw(1))\ge\mL(\bw(t^*+\alpha))+\frac{\beta}{2(1-\beta)\eta}\alpha^2 \Vert \bw(t+1)-\bw(t)\Vert^2
    +\frac{1-{   C_1}}{2{   C_1}}\alpha^2 \Vert \bw(t+1)-\bw(t)\Vert^2,
\end{equation*}
Consequently, for any $\alpha\in[0,\alpha^*]$
\begin{equation*}
    \mL(\bw(1))\ge \mL(\bw(t^*+\alpha)),
\end{equation*}
and by Lemma \ref{lem: smooth_guarantee}, we then have $\mL$ is $\frac{H_{s_0}\sigmax^2}{N}$ smooth at $\bw(t^*+\alpha)$, which further by Taylor's expansion leads to 
\begin{small}
\begin{align*}
    &\mL(\bw(t^*+\alpha^*))
    \\
    \le  & \mL(\bw(t^*))+\langle\nabla \mL(\bw(t^*)), \bw(t^*+\alpha^*)-\bw(t^*) \rangle+\frac{H_{s_0}\sigmax^2}{2N}\Vert \bw(t^*+\alpha^*)-\bw(t^*)\Vert^2
    \\
    \overset{(\circ)}{\le}& \mL(\bw(t^*))+\frac{\alpha^*\beta}{2(1-\beta)\eta}\Vert \bw(t^*)-\bw(t^*-1)\Vert^2+\frac{\alpha^*\beta}{2(1-\beta)\eta}\Vert \bw(t^*+1)-\bw(t^*)\Vert^2 
    \\
\nonumber
    +&\left(\frac{H_{s_0}(\alpha^*)^2\sigmax^2}{2N}-\frac{\alpha^*}{(1-\beta)\eta}\right)\Vert \bw(t^*+1)-\bw(t^*)\Vert^2
    \\
    \nonumber
    =&\mL(\bw(t^*))+\frac{\alpha^*\beta}{2(1-\beta)\eta}\Vert \bw(t^*)-\bw(t^*-1)\Vert^2 
    +\left(\frac{H_{s_0}(\alpha^*)^2\sigmax^2}{2N}-\frac{{  \alpha^*}(2-\beta)}{2(1-\beta)\eta}\right)\Vert \bw(t^*+1)-\bw(t^*)\Vert^2
    \\
    \overset{(\bullet)}{=}& \mL(\bw(t^*))+\frac{\alpha^*\beta}{2(1-\beta)\eta}\Vert \bw(t^*)-\bw(t^*-1)\Vert^2 
    +\left(\frac{{   C_1}(\alpha^*)^2}{\eta}-\frac{\alpha^*(2-\beta)}{2(1-\beta)\eta}\right)\Vert \bw(t^*+1)-\bw(t^*)\Vert^2
    \\
    \overset{(\ast)}{<}&\mL(\bw(t^*))+\frac{\beta}{2(1-\beta)\eta}\Vert \bw(t^*)-\bw(t^*-1)\Vert^2 
    -\frac{(\alpha^*)^2\beta}{2(1-\beta)\eta}\Vert \bw(t^*+1)-\bw(t^*)\Vert^2
    \\
    -&\frac{(1-{   C_1})(\alpha^*)^2}{\eta}\Vert \bw(t^*+1)-\bw(t^*)\Vert^2
    % \\
    %  \overset{(\square)}{<}&\mL(\bw(t^*))+\frac{\alpha\beta}{2\eta}\Vert \bw(t^*)-\bw(t^*-1)\Vert^2 
    % -\frac{\beta}{2\eta}\Vert \bw(t^*+1)-\bw(t^*)\Vert^2,
\end{align*}
\end{small}
where Eq. ($\circ$) follows the same routine as \textbf{Case 1}, Eq. ($\bullet$) is due to the definition of $\eta$ and ${   C_1}$, and Eq. ($\ast$) is due to $\alpha^*< 1$, and $\Vert \bw(t^*+1)-\bw(t^*)\Vert^2>0$ (given by Lemma \ref{lem: non-zero}).

By the continuity of $\mL$, for any small enough $\delta>0$, Eq. (\ref{eq: gdm_direct_taylor_extended}) holds for $(t^*,\alpha^*+\delta)$, which contradicts to the definition of $\alpha^*$.

The proof is completed.
\end{proof}

By Lemma \ref{lem: update_rule_loss_gdm}, one can easily obtain the sum of the squared norms of the updates across the trajectory converges. 

\begin{corollary}
\label{coro: update_rule_bounded}
Let all conditions in Theorem \ref{thm: gdm_main} hold. We have 
\begin{equation}
\label{eq: update_finite}
    \sum_{t=1}^\infty\Vert \bw(t+1)-\bw(t)\Vert^2<\infty.
\end{equation}

Consequentially, we have
\begin{equation*}
    \Vert \bw(t) \Vert =\mathcal{O}(\sqrt{t}).
\end{equation*}
\end{corollary}

\begin{proof}
By Lemma \ref{lem: update_rule_loss_gdm}, we have 
\begin{align*}
    \mL(\bw(t))+\frac{\beta}{2(1-\beta)\eta} \Vert \bw(t)-\bw(t-1)\Vert^2
    -& \left(\mL(\bw(t+1))+\frac{\beta}{2(1-\beta)\eta} \Vert \bw(t+1)-\bw(t)\Vert^2\right)
    \\
    &\ge\frac{1-{   C_1}}{\eta} \Vert \bw(t+1)-\bw(t)\Vert^2,
\end{align*}
which by summing over $t$ further leads to
\begin{equation*}
    \mL(\bw(1))\ge\mL(\bw(1))-\left(\mL(\bw(t+1))+\frac{\beta}{2(1-\beta)\eta} \Vert \bw(t+1)-\bw(t)\Vert^2\right)\ge \frac{1-{   C_1}}{\eta}\sum_{s=1}^t\Vert \bw(s+1)-\bw(s)\Vert^2.
\end{equation*}

Taking $t\rightarrow \infty$ leads to 
\begin{equation*}
    \sum_{s=1}^\infty\Vert \bw(s+1)-\bw(s)\Vert^2<\infty.
\end{equation*}

By triangle inequality, we further have
\begin{align*}
    \Vert \bw(t) \Vert \le & \sum_{s=1}^t \Vert \bw(s+1)-\bw(s)\Vert+ \Vert \bw(1)\Vert 
    \\
    \overset{(\star)}{\le} &\sqrt{t\left(\sum_{s=1}^t \Vert \bw(s+1)-\bw(s)\Vert^2\right)}+ \Vert \bw(1)\Vert =\mathcal{O}(\sqrt{t}),
\end{align*}
where Eq. ($\star$) is due to Cauchy-Schwartz's inequality.

The proof is completed.
\end{proof}

By the negative derivative of the loss and the separable data, we can finally prove the sum of squared gradient converges.
\begin{corollary}
\label{lem: sum_of_squared_gradient_converge}
 Let all conditions in Theorem \ref{thm: gdm_main} hold. We have, $\sum_{t=1}^{\infty}\Vert \nabla \mL(\bw(t))\Vert^2<\infty$.
\end{corollary}
\begin{proof}
By Eq. (\ref{eq: gdm_detailed_expansion}), we have
\begin{align}
\nonumber
\Vert \bw(t+1)-\bw(t) \Vert^2 
=& \eta^2\left\Vert(1-\beta) \sum_{s=1}^t \beta^{t-s}\frac{\sum_{i=1}^N\ell'(\langle \bw(s),\bx_i\rangle) \bx_i}{N} \right\Vert^2 
\\\nonumber
= &\eta^2(1-\beta)^2\left\Vert \sum_{s=1}^t \beta^{t-s}\frac{\sum_{i=1}^N\ell'(\langle \bw(s),\bx_i\rangle) \bx_i}{N} \right\Vert^2 {  \frac{
\Vert \hbw \Vert^2}{\Vert \hbw \Vert^2}}
\\\nonumber
\overset{(*)}{\ge} & \eta^2{  \gamma^2}(1-\beta) \left\langle \hbw ,\sum_{s=1}^t \beta^{t-s}\frac{\sum_{i=1}^N\ell'(\langle \bw(s),\bx_i\rangle) \bx_i}{N}\right\rangle^2
\\\nonumber
\overset{(**)}{\ge} & \eta^2\gamma^2(1-\beta)^2 \left( \sum_{s=1}^t \beta^{t-s}\frac{\sum_{i=1}^N\ell'(\langle \bw(s),\bx_i\rangle) }{N}\right)^2
\\ \nonumber
\ge & \eta^2\gamma^2(1-\beta)^2 \left(  \frac{\sum_{i=1}^N\ell'(\langle \bw(t),\bx_i\rangle) }{N}\right)^2
\\\nonumber
\overset{(\bullet)}{\ge} &\frac{\eta^2\gamma^2(1-\beta)^2}{\sigmax^2} \left\Vert  \frac{\sum_{i=1}^N\ell'(\langle \bw(t),\bx_i\rangle)  \bx_i}{N}\right\Vert^2
\\\label{eq: update_to_gradient}
=&\frac{\eta^2\gamma^2(1-\beta)^2}{\sigmax^2} \left\Vert  \nabla \mL(\bw(t))\right\Vert^2,
\end{align}
where Inequality $(*)$ is due to Cauchy-Schwartz's inequality, Inequality $(**)$ is due to $\ell'(s)< 0$, $\forall s\in \mathbb{R}$  and $\langle \hbw, \bx_i\rangle\ge \gamma$, $\forall i \in [N]$, and Inequality ($\bullet$) is due to the definition of $\sigmax$. By combining Eq. (\ref{eq: update_finite}) and Eq. (\ref{eq: update_to_gradient}), we complete the proof.
\end{proof}

By the exponential-tailed assumption of the loss (Assumption \ref{assum: exponential-tailed}), we further have the following corollary.

\begin{corollary}
\label{coro: time_asymptotic_loss}
Let all conditions in Theorem \ref{thm: gdm_main} hold. Then, $\lim_{t\rightarrow\infty}\left\Vert  \nabla \mL(\bw(t))\right\Vert=0$, and 
\begin{equation*}
    \lim_{  t\rightarrow\infty } \langle\bw(t) ,\bx_i\rangle=\infty, \forall i.
\end{equation*}
 Consequently, there exists an large enough time $t_0$, such that, $\forall t>t_0$, $\forall i$, we have $ \langle\bw(t) ,\bx_i\rangle>0$, and
 \begin{align*}
         -\ell'(\langle \bw(t),\bx_i\rangle)\le (1+e^{-\mu_+ \langle \bw(t),\bx_i\rangle})e^{-\langle \bw(t),\bx_i\rangle},
        \\
        -\ell'(\langle \bw(t),\bx_i\rangle)\ge (1-e^{-\mu_- \langle \bw(t),\bx_i\rangle})e^{-\langle \bw(t),\bx_i\rangle}.
    \end{align*}
 
\end{corollary}

\subsubsection{Parameter dynamics}
\label{subsubsec: form of the solution_gdm}
To prove Theorem \ref{thm: gdm_main}, we only need to show $\bw(t)-\ln(t)\hbw$ $(t\ge 1)$ has bounded norm for any iteration $t>0$. Letting ${  C_3}=\frac{\eta}{N}$ in Corollary \ref{coro:definition_of_tbw}, we obtain an constant vector $\tbw$ satisfying Eq. (\ref{eq: represent_v}). Define 
\begin{equation}
    \label{eq: def_r_gdm}
    \br(t)\overset{\triangle}{=}\bw(t)-\ln(t)\hbw-\tbw.
\end{equation} As $\tbw$ is a constant vector, that $\bw(t)-\ln(t)\hbw$ $(t\ge 1)$ has bounded norm is equivalent to $\br(t)$ has bounded norm. As discussed in the main body of the paper, we then propose an equivalent proposition of $\Vert \br(t) \Vert$ is bounded, and further prove this proposition is fulfilled. Specifically, we have
\begin{lemma}
\label{lem:gdm_r_t_bounded}
Let all conditions in Theorem \ref{thm: gdm_main} hold. Then, $\Vert \br(t) \Vert$ is bounded if and only if the function $g(t)$ is upper bounded, where $g:\mathbb{Z}^{+}\rightarrow\mathbb{R}$ is defined as
\begin{small}
\begin{align}
    g(t)\overset{\triangle}{=}&\frac{1}{2}\Vert \br(t)\Vert^2+\frac{\beta}{1-\beta}\langle \br(t),\bw(t)-\bw(t-1)\rangle
    \label{eq: definition_of_g}
    -\frac{\beta}{1-\beta}\sum_{\tau=2}^t\langle \br(\tau)-\br(\tau-1),\bw(\tau)-\bw(\tau-1)\rangle.
\end{align}
\end{small}
Furthermore, for almost every dataset, we have $\sum_{t=1}^{\infty} (g(t+1)-g(t))$ is upper bounded.
\end{lemma}

As the proof is rather complex, we separate it into two sub-lemmas. We first prove $\Vert \br(t) \Vert$ is bounded if and only if function $g(t)$ is upper bounded.

\begin{lemma}[First argument in Lemma \ref{lem:gdm_r_t_bounded}]
\label{lem:gdm_r_t_bounded_first}
Let all conditions in Theorem \ref{thm: gdm_main} hold. Then, $\Vert \br(t) \Vert$ is bounded if and only if function $g(t)$ is upper bounded.
\end{lemma}
\begin{proof}
We start the proof by showing that $A_1(t)\overset{\triangle}{=} \sum_{\tau=2}^t\langle \br(\tau)-\br(\tau-1),\bw(\tau)-\bw(\tau-1)\rangle$ has bounded absolute value.

By the definition of $\br(t)$ , we have 
\begin{align*}
    \br(t)-\br(t-1)=&\bw(t)-\bw(t-1)-\ln\left(\frac{t}{t-1}\right)\hbw,
\end{align*}
which further indicates
\begin{equation*}
    A_1(t)=\sum_{\tau=2}^t \left\langle \bw(\tau)-\bw(\tau-1)-\ln\left(\frac{\tau}{\tau-1}\right)\hbw,\bw(\tau)-\bw(\tau-1)\right\rangle.
\end{equation*}
Therefore, the absolute value of $A_1(t)$ can be bounded as 
\begin{align*}
    \vert A_1(t) \vert = &\left\vert \sum_{\tau=2}^t \left\langle \bw(\tau)-\bw(\tau-1)-\ln\left(\frac{\tau}{\tau-1}\right)\hbw,\bw(\tau)-\bw(\tau-1)\right\rangle \right\vert
\\
    \le &\sum_{\tau=2}^t  \left \vert\left\langle \bw(\tau)-\bw(\tau-1)-\ln\left(\frac{\tau}{\tau-1}\right)\hbw,\bw(\tau)-\bw(\tau-1)\right\rangle\right\vert \\
    \le &\sum_{\tau=2}^t  \left \Vert \bw(\tau)-\bw(\tau-1)\right\Vert^2+\sum_{\tau=2}^t\left\vert \left\langle\ln\left(\frac{\tau}{\tau-1}\right)\hbw,\bw(\tau)-\bw(\tau-1)\right\rangle\right\vert
    \\
    \le &\sum_{\tau=2}^t  \left \Vert \bw(\tau)-\bw(\tau-1)\right\Vert^2+\sum_{\tau=2}^t\left\Vert \ln\left(\frac{\tau}{\tau-1}\right)\hbw\right\Vert \left\Vert\bw(\tau)-\bw(\tau-1)\right\Vert
    \\
    \overset{(\star)}{\le} & \frac{3}{2}\sum_{\tau=2}^t  \left \Vert \bw(\tau)-\bw(\tau-1)\right\Vert^2+\frac{1}{2}\sum_{\tau=2}^t \left\Vert\ln\left(\frac{\tau}{\tau-1}\right)\hbw\right\Vert^2
    \\
    \overset{(\circ)}{<}& \infty,
\end{align*}
where Inequality $(\star)$ is due to the Inequality of arithmetic and geometric means, and Inequality $(\circ)$ is due to Corollary \ref{coro: update_rule_bounded} and $\ln\frac{\tau}{\tau-1}=\mO(\frac{1}{\tau})$.

Therefore, $g(t)$ is upper bounded is then equivalent to $\frac{1}{2}\Vert \br(t)\Vert^2+\frac{\beta}{1-\beta}\langle \br(t),\bw(t)-\bw(t-1)\rangle$ is upper bounded. Now if $\frac{1}{2}\Vert \br(t)\Vert^2+\frac{\beta}{1-\beta}\langle \br(t),\bw(t)-\bw(t-1)\rangle$ is upper bounded, we will prove $\Vert \br(t) \Vert $ is bounded by reduction to absurdity.

{  Suppose that $\Vert \br(t)\Vert$ has unbounded norm.} By Corollary \ref{coro: update_rule_bounded}, we have $\lim_{t\rightarrow \infty} \Vert \bw(t)-\bw(t-1)\Vert=0$, and there exists a large enough time $T$, such that $ \Vert \bw(t)-\bw(t-1)\Vert<1$ for any $t\ge T$. On the other hand, since $\br(t)$ is unbounded from above, there exists an increasing time sequence $k_i>T$, $i\in \mathbb{Z}^+$, such that 
\begin{equation*}
    \lim_{i\rightarrow\infty} \Vert \br(k_i)\Vert=\infty.
\end{equation*}

Therefore, we have 
\begin{align*}
    &\varliminf_{i\rightarrow\infty} \frac{1}{2}\Vert \br(k_i) \Vert^2+ \frac{\beta}{1-\beta}\langle \br(k_i),\bw(k_i)-\bw(k_i-1)\rangle
\end{align*}
\begin{align*}
    \ge& \varliminf_{i\rightarrow\infty}\frac{1}{2}\Vert \br(k_i) \Vert^2- \frac{\beta}{1-\beta}\left\Vert\br(k_i)\right\Vert\left\Vert\bw(k_i)-\bw(k_i-1)\right\Vert
    \\
    \ge& \varliminf_{i\rightarrow\infty}\frac{1}{2}\Vert \br(k_i) \Vert^2- \frac{\beta}{1-\beta}\left\Vert\br(k_i)\right\Vert=\infty,
\end{align*}
which leads to contradictory, and completes the proof of necessity.

On the other hand, if $\Vert\br(t)\Vert$ is upper bounded, since $\Vert\bw(t)-\bw(t-1)\Vert$ is also upper bounded, we have $\frac{1}{2}\Vert \br(t) \Vert^2+ \frac{\beta}{1-\beta}\langle \br(t),\bw(t)-\bw(t-1)\rangle$ is upper bounded, which completes the proof of sufficiency.

The proof is completed.
\end{proof}

Therefore, the last piece of this puzzle is to prove $g(t)$ is upper bounded $\forall t>0$.
\begin{lemma}[Second argument in Lemma \ref{lem:gdm_r_t_bounded}]
\label{lem: bounded_g_gdm}
Let all conditions in Theorem \ref{thm: gdm_main} hold.  Then, for almost every dataset, we have that $g(t)$ is upper bounded.
\end{lemma}
\begin{proof}
We start the proof by calculating $g(t+1)-g(t)$. For any $t\ge 2$, we have 
\begin{align*}
    g(t+1)-g(t)=&\frac{1}{2}\Vert \br(t+1)-\br(t)\Vert^2+ \langle \br(t),\br(t+1)-\br(t)\rangle+ \frac{\beta}{1-\beta}\langle \br(t+1),\bw(t+1)-\bw(t)\rangle
    \\
    -&\frac{\beta}{1-\beta}\langle \br(t),\bw(t)-\bw(t-1)\rangle-\frac{\beta}{1-\beta} \langle \br(t+1)-\br(t),\bw(t+1)-\bw(t)\rangle
    \\
    =&\frac{1}{2}\Vert \br(t+1)-\br(t)\Vert^2+ \langle \br(t),\br(t+1)-\br(t)\rangle+\frac{\beta}{1-\beta}\langle \br(t),\bw(t+1)+\bw(t-1)-2\bw(t)\rangle.
\end{align*}

On the other hand, by simply rearranging the update rule Eq. (\ref{eq: def_gdm}), we have
\begin{equation}
    \label{eq: rearrange_gdm_2}
    \frac{\beta}{1-\beta}(\bw(t+1)+\bw(t-1)-2\bw(t))=-\eta\nabla \mL (\bw(t))-(\bw(t+1)-\bw(t)),
\end{equation}
which further indicates
\begin{align*}
     &g(t+1)-g(t)
     \\
     =&\frac{1}{2}\Vert \br(t+1)-\br(t)\Vert^2+ \langle \br(t),\br(t+1)-\br(t)\rangle+\left\langle \br(t),-\eta\nabla \mL (\bw(t))-(\bw(t+1)-\bw(t))\right\rangle
     \\
     =&\frac{1}{2}\Vert \br(t+1)-\br(t)\Vert^2+\left\langle \br(t), -\ln\left(\frac{t+1}{t}\right)\hbw-\eta\nabla \mL (\bw(t))\right\rangle.
\end{align*}
Denote $A_2(t)= \Vert \br(t+1)-\br(t)\Vert^2$, and $A_3(t)=\left\langle \br(t), -\ln\left(\frac{t+1}{t}\right)\hbw-\eta\nabla \mL (\bw(t))\right\rangle$. We then prove respectively $\sum_{t=1}^{\infty} A_2(t)$ and $\sum_{t=1}^{\infty} A_3(t)$ are upper bounded.

First of all, by definition of $\br(t)$ Eq.(\ref{eq: def_r_gdm}), we have 
\begin{align}
\nonumber
    \sum_{t=1}^{\infty} A_2(t)=& \sum_{t=1}^{\infty} \left(\Vert \bw(t+1)-\bw(t) \Vert^2+ \ln\left(\frac{t+1}{t}\right)^2\Vert \hbw\Vert^2-2\ln \left(\frac{t+1}{t}\right)\langle \bw(t+1)-\bw(t),\hbw\rangle\right)
    \\
\label{eq: A_2 finite}
    \le &2\sum_{t=1}^{\infty} \left(\Vert \bw(t+1)-\bw(t) \Vert^2+ \ln\left(\frac{t+1}{t}\right)^2\Vert \hbw\Vert^2\right)\overset{(\bullet)}{<}\infty,
\end{align}
where Eq. ($\bullet$) is due to Lemma \ref{coro: update_rule_bounded} and $\ln \left(\frac{t+1}{t}\right)=\mathcal{O}(\frac{1}{t})$.

Then we only need to prove $\sum_{t=1}^\infty A_3(t) <\infty$. 

To begin with, by adding one additional term $\frac{1}{t} \hbw$ into $A_3$, we have
\begin{align*}
    A_3(t)=& \left\langle \br(t),\frac{1}{t}\hbw -\ln\left(\frac{t+1}{t}\right)\hbw\right\rangle+\left\langle \br(t),-\frac{1}{t}\hbw-\eta\nabla \mL (\bw(t))\right\rangle.
\end{align*}

On the one hand, by Corollary \ref{coro: update_rule_bounded}, $\Vert \bw(t)\Vert =\mathcal{O}(\sqrt{t})$, which further leads to \begin{equation*}
    \Vert \br(t) \Vert=\Vert \bw(t) \Vert + \ln (t)\Vert \hbw \Vert+\Vert \hbw\Vert=\mathcal{O}(\sqrt{t})
\end{equation*}
By {  $\frac{1}{t}-\ln \frac{t+1}{t}=\mathcal{O}\left(\frac{1}{t^{2}}\right)$}, we have
\begin{equation}
\label{eq: 1 over t difference}
    \left\langle \br(t), \frac{1}{t}\hbw -\ln\left(\frac{t+1}{t}\right)\hbw \right\rangle=\mathcal{O}\left(\frac{1}{t^{\frac{3}{2}}}\right).
\end{equation}

On the other hand, by direct calculation of the gradient, we have
\begin{align*}
    &\left\langle \br(t),-\frac{1}{t}\hbw-\eta\nabla \mL (\bw(t))\right\rangle
    \\
    =&\left\langle \br(t),-\frac{1}{t}\hbw-\frac{\eta}{N}\sum_{i=1}^N \ell'(\langle \bw(t), \bx_i\rangle)\bx_i\right\rangle
    \\
    \overset{(\star)}{=}&\frac{1}{N}\left\langle \br(t),-\frac{1}{t}\eta\sum_{\bx_i\in\bS_{s}} e^{-\langle \tbw, \bx_i\rangle} \bx_i -\eta\sum_{i=1}^N \ell'(\langle \bw, \bx_i\rangle)\bx_i\right\rangle
    \\
    =& \frac{1}{N}\left\langle \br(t),-\eta\sum_{\bx_i\in\bS_{s}} \left(\frac{1}{t}e^{-\langle \tbw, \bx_i\rangle}+\ell'(\langle \bw(t), \bx_i\rangle)\right) \bx_i\right\rangle -{  \frac{1}{N}}\left\langle\br(t),\eta\sum_{\bx_i\notin\bS_{s}} \ell'(\langle \bw(t), \bx_i\rangle)\bx_i\right\rangle,
\end{align*}
where Eq. ($\star$) is due to the definition of $\tbw$ (Eq. (\ref{eq: represent_v}) with ${   C_3}=\eta/N$).

Denote 
\begin{equation*}
    A_4(t)=-\left\langle\br(t),\eta\sum_{\bx_i\notin\bS_{s}} \ell'(\langle \bw(t), \bx_i\rangle)\bx_i\right\rangle,
\end{equation*} and
\begin{equation*}
    A_5(t)=\left\langle \br(t),-\eta\sum_{\bx_i\in\bS_{s}} \left(\frac{1}{t}e^{-\langle \tbw, \bx_i\rangle}+\ell'(\langle \bw(t), \bx_i\rangle)\right) \bx_i\right\rangle.
\end{equation*} We then analysis these two terms respectively. As for $A_4(t)$, due to $\ell'<0$, we have 
\begin{align*}
    A_4(t)\le -\eta\left\langle\br(t),\sum_{\bx_i\notin\bS_{s},\langle\br(t),\bx_i\rangle>0} \ell'(\langle \bw(t), \bx_i\rangle)\bx_i\right\rangle.
\end{align*}
By Corollary \ref{coro: time_asymptotic_loss}, we further have $\forall t>t_0$
\begin{equation*}
    -\ell'(\langle\bw(t), \bx_i\rangle)\le (1+e^{-\mu_+ \langle \bw(t),\bx_i\rangle})e^{-\langle \bw(t),\bx_i\rangle}\le 2e^{-\langle \bw(t),\bx_i\rangle},
\end{equation*}
which further indicates
\begin{align*}
    A_4(t)\le& -\eta\sum_{\bx_i\notin\bS_{s},\langle\br(t),\bx_i\rangle>0}\ell'(\langle \bw(t), \bx_i\rangle)\left\langle\br(t), \bx_i\right\rangle
    \\
    \le &\eta\sum_{\bx_i\notin\bS_{s},\langle\br(t),\bx_i\rangle>0}2e^{-\langle \bw(t),\bx_i\rangle}\left\langle\br(t), \bx_i\right\rangle
    \\
    =&\eta\sum_{\bx_i\notin\bS_{s},\langle\br(t),\bx_i\rangle>0}2e^{-\langle \br(t)+\ln t \hbw+\tbw,\bx_i\rangle}\left\langle\br(t), \bx_i\right\rangle
    \\
    \le &\eta\left(\max_{i}e^{\langle-\tbw,\bx_i\rangle}\right)\sum_{\bx_i\notin\bS_{s},\langle\br(t),\bx_i\rangle>0}2e^{-\langle \br(t)+\ln t \hbw,\bx_i\rangle}\left\langle\br(t), \bx_i\right\rangle
    \\
    \overset{(\circ)}{\le }&\eta\frac{\left(\max_{i}e^{\langle-\tbw,\bx_i\rangle}\right)}{t^{\theta}}\sum_{\bx_i\notin\bS_{s},\langle\br(t),\bx_i\rangle>0}2e^{-\langle \br(t),\bx_i\rangle}\left\langle\br(t), \bx_i\right\rangle
    \\
    \overset{(\diamond)}{\le }&\eta\frac{\left(\max_{i}e^{\langle-\tbw,\bx_i\rangle}\right)}{t^{\theta}}2N,
\end{align*}
where $\theta$ in Eq. $(\circ)$ is defined as 
\begin{equation}
\label{eq: def_theta}
    \theta=\min_{\bx_i\notin\bS_{s}} \langle \bx_i, \hbw \rangle>1.
\end{equation}

As $\sum_{t=1}^{\infty} \frac{1}{t^{\theta}}<\infty$, we have 
\begin{equation}
\label{eq: A_4 bounded}
    \sum_{t=1}^{\infty} A_4(t)<\infty\footnote{In this paper, for a real series $\{r_i\}_{i=1}^{\infty}$, we use $\sum_{i=1}^{\infty} r_i<\infty$ representing $\sum_{i=1}^{T} r_i$ is uniformly upper bounded for any $T$.}.
\end{equation}

For each term $\left\langle \br(t),-\eta \left(\frac{1}{t}e^{-\langle \tbw, \bx_i\rangle}+\ell'(\langle \bw, \bx_i\rangle)\right) \bx_i\right\rangle$ ($\bx_i\notin\bS_{s}$) in $A_5(t)$, we divide the analysis into two parts depending on the sign of $\langle \br(t), \bx_i \rangle$. 

\textbf{Case 1: $\langle \br(t), \bx_i \rangle \ge 0$.} By Corollary \ref{coro: time_asymptotic_loss}, we have  
\begin{align*}
    &\left\langle \br(t),-\eta \left(\frac{1}{t}e^{-\langle \tbw, \bx_i\rangle}+\ell'(\langle \bw, \bx_i\rangle)\right) \bx_i\right\rangle
    \\
    =&-\eta \left(\frac{1}{t}e^{-\langle \tbw, \bx_i\rangle}+\ell'(\langle \bw, \bx_i\rangle)\right)\left\langle \br(t), \bx_i\right\rangle
    \\
    \le &\eta \left(-\frac{1}{t}e^{-\langle \tbw, \bx_i\rangle}+(1+e^{-\mu_+ \langle \bw(t),\bx_i\rangle})e^{-\langle \bw(t),\bx_i\rangle}\right)\left\langle \br(t), \bx_i\right\rangle
    \\
    \overset{(\diamond)}{=}&\eta \left(-\frac{1}{t}e^{-\langle \tbw, \bx_i\rangle}+(1+e^{-\mu_+ \langle \br(t)+\ln t \hbw+\tbw,\bx_i\rangle})e^{-\langle \br(t)+\ln t \hbw+\tbw,\bx_i\rangle}\right)\left\langle \br(t), \bx_i\right\rangle,
\end{align*}
where Eq. ($\diamond$) is due to the definition of $\br(t)$ (Eq. (\ref{eq: def_r_gdm}).

Since $\langle \br(t), \bx_i\rangle\ge 0$, we further have
\begin{align*}
    &\left\langle \br(t),-\eta \left(\frac{1}{t}e^{-\langle \tbw, \bx_i\rangle}+\ell'(\langle \bw, \bx_i\rangle)\right) \bx_i\right\rangle
    \\
    \le & \eta \left(-\frac{1}{t}e^{-\langle \tbw, \bx_i\rangle}+(1+e^{-\mu_+ \langle\ln t \hbw+\tbw,\bx_i\rangle})e^{-\langle \br(t)+\ln t \hbw+\tbw,\bx_i\rangle}\right)\left\langle \br(t), \bx_i\right\rangle
    \\
    \overset{(\Box)}{=}&\eta \left(-\frac{1}{t}e^{-\langle \tbw, \bx_i\rangle}+\frac{1}{t}(1+t^{-\mu_{+}}e^{-\mu_+ \langle\tbw,\bx_i\rangle})e^{-\langle \br(t)+\tbw,\bx_i\rangle}\right)\left\langle \br(t), \bx_i\right\rangle
    \\
    =&\eta \frac{1}{t}e^{-\langle \tbw, \bx_i\rangle}\left(-1+(1+t^{-\mu_{+}}e^{-\mu_+ \langle\tbw,\bx_i\rangle})e^{-\langle \br(t),\bx_i\rangle}\right)\left\langle \br(t), \bx_i\right\rangle,
\end{align*}
where Eq. ($\Box$) is due to $\langle \hbw, \bx_i \rangle=1$, $\forall \bx_i \in\bS_{s}$.

Specifically, 
\begin{align*}
    &-1+(1+t^{-\mu_{+}}e^{-\mu_+ \langle\tbw,\bx_i\rangle})e^{-\langle \br(t),\bx_i\rangle}
    \\
    =&-1+e^{-\langle \br(t),\bx_i\rangle}+t^{-\mu_{+}}e^{-\mu_+ \langle\tbw,\bx_i\rangle}e^{-\langle \br(t),\bx_i\rangle}
    \\
    \le &t^{-\mu_{+}}e^{-\mu_+ \langle\tbw,\bx_i\rangle}e^{-\langle \br(t),\bx_i\rangle}.
\end{align*}
Therefore, 
\begin{align*}
    &\eta \frac{1}{t}e^{-\langle \tbw, \bx_i\rangle}\left(-1+(1+t^{-\mu_{+}}e^{-\mu_+ \langle\tbw,\bx_i\rangle})e^{-\langle \br(t),\bx_i\rangle}\right)\left\langle \br(t), \bx_i\right\rangle
    \\
    \le &\eta \frac{1}{t}e^{-\langle \tbw, \bx_i\rangle}\left(t^{-\mu_{+}}e^{-\mu_+ \langle\tbw,\bx_i\rangle}e^{-\langle \br(t),\bx_i\rangle}\right)\left\langle \br(t), \bx_i\right\rangle
    \\
    \le & \frac{\eta}{(1-\beta)e} \frac{1}{t^{1+\mu_{+}}}e^{-(1+\mu_{+})\langle \tbw, \bx_i\rangle}=\mathcal{O}\left(\frac{1}{t^{1+\mu_{+}}}\right).
\end{align*}
% Specifically, if $\langle \br(t), \bx_i\rangle\le t^{-0.5\mu_{+}}$, since 
% \begin{align*}
%     &-1+(1+t^{-\mu_{+}}e^{-\mu_+ \langle\tbw,\bx_i\rangle})e^{-\langle \br(t),\bx_i\rangle}
%     \\
%     \le & -1+1+t^{-\mu_{+}}e^{-\mu_+ \langle\tbw,\bx_i\rangle}\rightarrow 0
% \end{align*}
% is upper bounded, which further indicates
% \begin{equation*}
%     \frac{\eta}{1-\beta} \frac{1}{t}e^{-\langle \tbw, \bx_i\rangle}\left(-1+(1+t^{-\mu_{+}}e^{-\mu_+ \langle\tbw,\bx_i\rangle})e^{-\langle \br(t),\bx_i\rangle}\right)\left\langle \br(t), \bx_i\right\rangle =\mathcal{O}\left(\frac{1}{t^{1+0.5\mu_{+}}}\right).
% \end{equation*}

% If  $\langle \br(t), \bx_i\rangle> t^{-0.5\mu_{+}}$
\textbf{Case 2: $\langle \br(t), \bx_i \rangle < 0$.} Similar to \textbf{Case 1.}, in this case we have
\begin{align*}
    &\left\langle \br(t),-\eta \left(\frac{1}{t}e^{-\langle \tbw, \bx_i\rangle}+\ell'(\langle \bw, \bx_i\rangle)\right) \bx_i\right\rangle
    \\
    \le &\eta \left(-\frac{1}{t}e^{-\langle \tbw, \bx_i\rangle}+\left(1-e^{-\mu_{-} \langle \bw(t),\bx_i\rangle}\right)e^{-\langle \bw(t),\bx_i\rangle}\right)\left\langle \br(t), \bx_i\right\rangle
    \\
    =&\eta \left(-\frac{1}{t}e^{-\langle \tbw, \bx_i\rangle}+\left(1-e^{-\mu_{-} \langle \br(t)+\ln t \hbw+\tbw,\bx_i\rangle}\right)e^{-\langle \br(t)+\ln t \hbw+\tbw,\bx_i\rangle}\right)\left\langle \br(t), \bx_i\right\rangle
    \\
    =&\eta\frac{1}{t}e^{-\langle \tbw, \bx_i\rangle} \left(-1+\left(1-e^{-\mu_{-} \langle \br(t)+\ln t \hbw+\tbw,\bx_i\rangle}\right)e^{-\langle \br(t),\bx_i\rangle}\right)\left\langle \br(t), \bx_i\right\rangle.
\end{align*}

Specifically, if $\langle \br(t), \bx_i\rangle\ge -t^{-0.5 \mu_{-}}$, 
\begin{align*}
    &\left\vert \eta\frac{1}{t}e^{-\langle \tbw, \bx_i\rangle} \left(-1+(1-e^{-\mu_{-} \langle \br(t)+\ln t \hbw+\tbw,\bx_i\rangle})e^{-\langle \br(t),\bx_i\rangle}\right)\left\langle \br(t), \bx_i\right\rangle\right\vert
    \\
    =&\left\vert \eta\frac{1}{t}e^{-\langle \tbw, \bx_i\rangle} \left(-1+(1-t^{-\mu_{-}}e^{-\mu_{-} \langle \br(t)+\tbw,\bx_i\rangle})e^{-\langle \br(t),\bx_i\rangle}\right)\left\langle \br(t), \bx_i\right\rangle\right\vert
    \\
    \le& \eta \frac{1}{t^{1+0.5\mu_{-}}}e^{-\langle \tbw, \bx_i\rangle} \left\vert -1+\left(1-t^{-\mu_{-}}e^{-\mu_{-} \langle \br(t)+\tbw,\bx_i\rangle}\right)e^{-\langle \br(t),\bx_i\rangle}\right\vert
    \\
    \overset{(\dagger)}{=}&\mathcal{O}\left(\frac{1}{t^{1+0.5\mu_{-}}}\right),
\end{align*}
where Eq. ($\dagger$) is due to if $\langle \br(t), \bx_i\rangle\ge -t^{-0.5 \mu_{-}}$,
\begin{equation*}
    \lim_{t\rightarrow \infty }\left\vert -1+\left(1-t^{-\mu_{-}}e^{-\mu_{-} \langle \br(t)+\tbw,\bx_i\rangle}\right)e^{-\langle \br(t),\bx_i\rangle}\right\vert =0.
\end{equation*}

If $-2\le \langle \br(t), \bx_i\rangle< -t^{-0.5 \mu_{-}}$, we have 
\begin{align*}
    &\eta\frac{1}{t}e^{-\langle \tbw, \bx_i\rangle} \left(-1+\left(1-e^{-\mu_{-} \langle \br(t)+\ln t \hbw+\tbw,\bx_i\rangle}\right)e^{-\langle \br(t),\bx_i\rangle}\right)\left\langle \br(t), \bx_i\right\rangle
    \\
    =&\eta\frac{1}{t}e^{-\langle \tbw, \bx_i\rangle} \left(-1+\left(1-\frac{1}{t^{\mu_{-}}}e^{-\mu_{-} \langle \br(t)+\tbw,\bx_i\rangle}\right)e^{-\langle \br(t),\bx_i\rangle}\right)\left\langle \br(t), \bx_i\right\rangle
    \\
    \le &\eta\frac{1}{t}e^{-\langle \tbw, \bx_i\rangle} \left(-1+\left(1-\frac{e^{2\mu_{-}}}{t^{\mu_{-}}}e^{-\mu_{-} \langle \tbw,\bx_i\rangle}\right)e^{-\langle \br(t),\bx_i\rangle}\right)\left\langle \br(t), \bx_i\right\rangle.
\end{align*}
Therefore, when $t$ is large enough, $1-\frac{e^{2\mu_{-}}}{t^{\mu_{-}}}e^{-\mu_{-} \langle \tbw,\bx_i\rangle}>0$, which by $e^{-\langle \br(t),\bx_i\rangle}\ge 1-\langle \br(t),\bx_i\rangle$ leads to 
\begin{align*}
    &\eta\frac{1}{t}e^{-\langle \tbw, \bx_i\rangle} \left(-1+\left(1-\frac{e^{2\mu_{-}}}{t^{\mu_{-}}}e^{-\mu_{-} \langle \tbw,\bx_i\rangle}\right)e^{-\langle \br(t),\bx_i\rangle}\right)\left\langle \br(t), \bx_i\right\rangle
    \\
    \le & \eta\frac{1}{t}e^{-\langle \tbw, \bx_i\rangle} \left(-1+\left(1-\frac{e^{2\mu_{-}}}{t^{\mu_{-}}}e^{-\mu_{-} \langle \tbw,\bx_i\rangle}\right)\left(1-\langle \br(t),\bx_i\rangle\right)\right)\left\langle \br(t), \bx_i\right\rangle
    \\
     \le & \eta\frac{1}{t}e^{-\langle \tbw, \bx_i\rangle} \left(-1+\left(1-\frac{e^{2\mu_{-}}}{t^{\mu_{-}}}e^{-\mu_{-} \langle \tbw,\bx_i\rangle}\right)\left(1+\frac{1}{t^{0.5\mu_{-}}}\right)\right)\left\langle \br(t), \bx_i\right\rangle
     \\
     =&\eta\frac{1}{t}e^{-\langle \tbw, \bx_i\rangle} \left(\frac{1}{t^{0.5\mu_{-}}}+\boldsymbol{o}\left(\frac{1}{t^{0.5\mu_{-}}}\right)\right)\left\langle \br(t), \bx_i\right\rangle<0.
\end{align*}

If $-2> \langle \br(t), \bx_i\rangle$,
\begin{align*}
    &\eta\frac{1}{t}e^{-\langle \tbw, \bx_i\rangle} \left(-1+\left(1-e^{-\mu_{-} \langle \br(t)+\ln t \hbw+\tbw,\bx_i\rangle}\right)e^{-\langle \br(t),\bx_i\rangle}\right)\left\langle \br(t), \bx_i\right\rangle
    \\
    =&\eta\frac{1}{t}e^{-\langle \tbw, \bx_i\rangle} \left(-1+\left(1-e^{-\mu_{-} \langle\bw(t),\bx_i\rangle}\right)e^{-\langle \br(t),\bx_i\rangle}\right)\left\langle \br(t), \bx_i\right\rangle.
\end{align*}
For large enough $t$,  $ 1-e^{-\mu_{-} \langle\bw(t),\bx_i\rangle}>\frac{1}{2}$, and 
\begin{align*}
    &\eta\frac{1}{t}e^{-\langle \tbw, \bx_i\rangle} \left(-1+\left(1-e^{-\mu_{-} \langle\bw(t),\bx_i\rangle}\right)e^{-\langle \br(t),\bx_i\rangle}\right)\left\langle \br(t), \bx_i\right\rangle
    \\
    \le &\eta\frac{1}{t}e^{-\langle \tbw, \bx_i\rangle} \left(-1+\left(1-e^{-\mu_{-} \langle\bw(t),\bx_i\rangle}\right)e^{2}\right)\left\langle \br(t), \bx_i\right\rangle
    \\
    \le &\eta\frac{1}{t}e^{-\langle \tbw, \bx_i\rangle} \left(-1+\frac{e^{2}}{2}\right)\left\langle \br(t), \bx_i\right\rangle<0.
\end{align*}

Therefore, in \textbf{Case 2.}, for large enough $t$, we have 
\begin{equation*}
    \left\langle \br(t),-\eta \left(\frac{1}{t}e^{-\langle \tbw, \bx_i\rangle}+\ell'(\langle \bw, \bx_i\rangle)\right) \bx_i\right\rangle\le \mathcal{O}\left(\frac{1}{t^{1+0.5\mu_{-}}}\right).
\end{equation*}

Combining \textbf{Case 1.} and \textbf{Case 2.}, we conclude that 
\begin{equation*}
    A_5(t)\le \mathcal{O}\left(\frac{1}{t^{1+0.5\mu_{+}}}\right),
\end{equation*}
which further yields 
\begin{equation}
\label{eq: A_5 finite}
    \sum_{t=1}^{\infty} A_5(t)< \infty.
\end{equation}

Combining Eq. (\ref{eq: A_4 bounded}) and Eq. (\ref{eq: A_5 finite}), we conclude that $\sum_{t=1}^{\infty} A_3(t)<\infty$, which together with Eq. (\ref{eq: A_2 finite}) yields $\sum_{t=2}^{\infty} g(t+1)-g(t)<\infty$, and completes the proof.
\end{proof}

We are now ready to prove Theorem \ref{thm: gdm_main}.

\begin{proof}[Proof of Theorem \ref{thm: gdm_main}] By Lemma \ref{lem: bounded_g_gdm}, we have $g(t)$ is upper bounded. Therefore, by Lemma \ref{lem:gdm_r_t_bounded}, we have $\Vert \br(t) \Vert $ is bounded, which further indicates $\Vert\bw(t)-\ln(t)\hbw \Vert$ is  bounded. 

Therefore, the direction of $\bw(t)$ can be calculated as 
\begin{align*}
    \frac{\bw(t)}{\Vert \bw(t) \Vert }=&\frac{\ln(t)\hbw}{\Vert \bw(t)\Vert }+\frac{\bw(t)-\ln(t)\hbw}{\Vert \bw(t)\Vert }
=\frac{\ln(t)\hbw}{\Vert \ln(t) \hbw+\bw(t)-\ln(t) \hbw\Vert }+\frac{\bw(t)-\ln(t)\hbw}{\Vert \bw(t)\Vert }
\\
=&\frac{\hbw}{\left\Vert  \hbw+\frac{\bw(t)-\ln(t) \hbw}{\ln t}\right\Vert }+\frac{\bw(t)-\ln(t)\hbw}{\Vert \bw(t)\Vert }\rightarrow \frac{\hbw}{\Vert \hbw \Vert} ~(as~ t\rightarrow\infty).
\end{align*}

The proof is completed.

\end{proof}

\subsection{Implicit regularization of SGDM}\label{appen: sgdm_proof}
This section collects the proof of Theorem \ref{thm: sgdm_main}. Following the same framework as Appendix \ref{appen: implicit_bias_gd}, we will first prove that the sum of the squared gradient norms along the trajectory is finite. One may expect $ \mL(\bw(t))+\frac{\beta}{2\eta} \Vert \bw(t)-\bw(t-1)\Vert^2$ is a Lyapunov function of SGDM. However, due to the randomness of the update rule of SGDM, $\mL(\bw(t))+\frac{\beta}{2\eta} \Vert \bw(t)-\bw(t-1)\Vert^2$ may no longer decrease (we will show this in the end of Appendix \ref{appen: sgdm_proof}, please see Appendix \ref{subsec: explanation} for explanation).

\subsubsection{Loss dynamics}
Recall that in the main text, we define $\bu(t)$ as
\begin{equation}
    \label{eq: alter_sgdm}
   \bu(t)=\frac{\bw(t)-\beta\bw(t-1)}{1-\beta},
\end{equation}
where the update of $\bu(t)$ is given by $\bu(t+1)=\bu(t)-\eta\nabla \mL_{\bB(t)} (\bw(t))$.
We then prove Lemma \ref{lem: sgdm_loss_update}, which indicates $\mL(\bu(t))$ is a proper choice of Lyapunov function.

\begin{proof}[Proof of Lemma \ref{lem: sgdm_loss_update}]
    We start the proof by applying the Taylor's expansion of $\mL$ at the point $\bu(t)$ to the point $\bu(t+1)$. Concretely, by Assumption \ref{assum: smooth}. (S), we have
    \begin{equation*}
        \mL(\bu(t+1))\le \mL(\bu(t))+\langle \bu(t+1)-\bu(t),\nabla \mL(\bu(t))\rangle+\frac{H\sigmax^2}{2N} \Vert \bu(t+1)-\bu(t) \Vert^2,
    \end{equation*}
    which by Eq. (\ref{eq: alter_sgdm}) leads to 
    \begin{equation}
    \label{eq: loss_sgdm_mid_1}
        \mL(\bu(t+1))\le \mL(\bu(t))-\eta\langle \nabla \mL_{\bB(t)}(\bw(t)),\nabla \mL(\bu(t))\rangle+\frac{H\sigmax^2\eta^2}{2N} \Vert \nabla \mL_{\bB(t)}(\bw(t)) \Vert^2.
    \end{equation}

    Taking the expectation of Eq. (\ref{eq: loss_sgdm_mid_1}) with respect to $\bw(t+1)$ conditioning on $\mF_t$ (recall that $\mF_t$ is the sub-sigma algebra over the mini-batch sampling, such that $\forall t\in\mathbb{N}$, $\bw(t)$ is adapted with respect to the sigma algebra flow $\mF_t$), we have 
    \begin{align}
    \nonumber
        &\mE[ \mL(\bu(t+1))|\mF_t]
        \\
            \nonumber
        \overset{(\star)}{=}&
         \mE_{\bB(t)}[ \mL(\bu(t+1))|\mF_t]
         \\
             \nonumber
         \le &\mE_{\bB(t)}\left[\left.\mL(\bu(t))-\eta\langle \nabla \mL_{\bB(t)}(\bw(t)),\nabla \mL(\bu(t))\rangle+\frac{H\eta^2\sigmax^2}{2N} \Vert \nabla \mL_{\bB(t)}(\bw(t)) \Vert^2\right|\mF_t\right]
         \\
             \nonumber
         \overset{(\circ)}{=}&\mL(\bu(t))-\eta\langle \nabla \mL(\bw(t)),\nabla \mL(\bu(t))\rangle+\frac{H\sigmax^2\eta^2}{2N} \mE_{\bB(t)}\left[\Vert \nabla \mL_{\bB(t)}(\bw(t))\Vert^2\right]
         \\
         \label{eq: expectation_updates}
         {  \overset{(\bullet)}{\le}}&\mL(\bu(t))-\eta\langle \nabla \mL(\bw(t)),\nabla \mL(\bu(t))\rangle+\frac{H\eta^2\sigmax^4}{2b\gamma^2} \Vert \nabla \mL(\bw(t)) \Vert^2,
    \end{align}
    where Eq. $(\star)$ is due to that $\bw(t+1)$ is uniquely determined by $\bB(t)$ given $\{\bw(s)\}_{s=1}^t$,  Eq. $(\circ)$ is due to $\bu(t)$ is uniquely determined by $\{\bw(s)\}_{s=1}^t$, and {  Inequality.} $(\bullet)$ is due to Lemma \ref{lem: first_second}.
    
    Therefore, we have
    \begin{align*}
        &\mathbb{E}[ \mL(\bu(t+1))|\mF_t]
        \\
        \le &\mL(\bu(t))-\eta\langle \nabla \mL(\bw(t)),\nabla \mL(\bu(t))\rangle+\frac{H\eta^2\sigmax^4}{2b\gamma^2} \Vert \nabla \mL(\bw(t)) \Vert^2
        \\
        = &\mL(\bu(t))-\eta\langle \nabla \mL(\bw(t)),\nabla \mL(\bw(t))\rangle+\eta\langle \nabla \mL(\bw(t)),\nabla \mL(\bw(t))-\nabla \mL(\bu(t))\rangle+\frac{H\eta^2\sigmax^4}{2b\gamma^2} \Vert \nabla \mL(\bw(t)) \Vert^2
        \\
        =&\mL(\bu(t))-\eta \left(1-\frac{H\eta\sigmax^4}{2b\gamma^2}\right)\Vert \nabla \mL(\bw(t)) \Vert^2+\left\langle\eta \nabla\mL(\bw(t)),\nabla \mL(\bw(t)) -\nabla \mL(\bu(t))\right\rangle
        \\
        \le &\mL(\bu(t))-\eta \left(1-\frac{H\eta\sigmax^4}{2b\gamma^2}\right)\Vert \nabla \mL(\bw(t)) \Vert^2+\frac{1}{2\lambda}\left\Vert\eta \nabla\mL(\bw(t))\right\Vert^2+\frac{\lambda}{2}\left \Vert \nabla \mL(\bw(t)) -\nabla \mL(\bu(t))\right \Vert^2
        \\
        =& \mL(\bu(t))-\eta \left(1-\left(\frac{1}{2\lambda}+\frac{H\sigmax^4}{2b\gamma^2}\right)\eta\right)\Vert \nabla \mL(\bw(t)) \Vert^2+\frac{\lambda}{2}\left \Vert \nabla \mL(\bw(t)) -\nabla \mL(\bu(t))\right \Vert^2,
    \end{align*}
    where $\lambda$ is a positive constant that will be specified latter.
    
    By Assumption \ref{assum: smooth}. (S), $\ell$ is $H$-smooth, which further leads to
    \begin{align}
    \nonumber
        &\left \Vert \nabla \mL(\bw(t)) -\nabla \mL(\bu(t))\right \Vert^2
        \\
          \nonumber
          \le & \frac{H^2\sigmax^4}{N^2} \Vert \bw(t)-\bu(t) \Vert^2
          \overset{(\square)}{=}   \frac{H^2\beta^2\sigmax^4}{N^2(1-\beta)^2} \Vert \bw(t)-\bw(t-1) \Vert^2
          \\
          \nonumber
          =&\frac{H^2\beta^2\sigmax^4}{N^2} \left\Vert \sum_{s=1}^{t-1} \eta \beta^{t-1-s} \nabla \mL_{\bB(s)}(\bw(s))\right\Vert^2
          \\
          \nonumber
         \overset{(\diamond)}{\le} &\frac{H^2\beta^2\eta^2\sigmax^4}{N^2} \left(\sum_{s=1}^{t-1} \beta^{t-1-s}\left\Vert  \nabla \mL_{\bB(s)}(\bw(s))\right\Vert\right)^2
         \\\nonumber
         \overset{(\clubsuit)}{\le } &\frac{H^2\beta^2\eta^2\sigmax^4}{N^2} \left(\sum_{s=1}^{t-1} \beta^{t-1-s}\left\Vert  \nabla \mL_{\bB(s)}(\bw(s))\right\Vert^2\right)\left(\sum_{s=1}^{t-1} \beta^{t-1-s}\right)
         \\         \label{eq: difference_gradient}
         \le&
         \frac{H^2\beta^2\eta^2\sigmax^4}{N^2(1-\beta)} \left(\sum_{s=1}^{t-1} \beta^{t-1-s}\left\Vert  \nabla \mL_{\bB(s)}(\bw(s))\right\Vert^2\right),
    \end{align}
    where Inequality $(\square)$ is due to $\beta(\bw(t)-\bw(t-1))=(1-\beta)(\bu(t)-\bw(t))$ by Eq. (\ref{eq: alter_sgdm}), Inequality $(\diamond)$ is due to triangular inequality, and Inequality ($\clubsuit$) is due to Cauchy-Schwartz Inequality.
    
    Combining Eqs. (\ref{eq: expectation_updates}) and (\ref{eq: difference_gradient}), we have
    \begin{align*}
        &\mE[ \mL(\bu(t+1))|\mF_t]
        \\
        \le& \mL(\bu(t))-\eta \left(1-\left(\frac{1}{2\lambda}+\frac{H\sigmax^4}{2b\gamma^2}\right)\eta\right)\Vert \nabla \mL(\bw(t)) \Vert^2+ \frac{\lambda H^2\beta^2\eta^2\sigmax^4}{2N^2(1-\beta)} \left(\sum_{s=1}^{t-1} \beta^{t-1-s}\left\Vert  \nabla \mL_{\bB(s)}(\bw(s))\right\Vert^2\right),
    \end{align*}
    which by taking expectation with respect to $\mF_t$ leads to
    \begin{small}
    \begin{align*}
        &\mE[ \mL(\bu(t+1))]
        \\
        \le &\mE\mL(\bu(t))-\eta \left(1-\left(\frac{1}{2\lambda}+\frac{H\sigmax^4}{2b\gamma^2}\right)\eta\right)\mE\Vert \nabla \mL(\bw(t)) \Vert^2+\mE\frac{\lambda H^2\beta^2\eta^2\sigmax^4}{2N^2(1-\beta)} \left(\sum_{s=1}^{t-1} \beta^{t-1-s}\left\Vert  \nabla \mL_{\bB(s)}(\bw(s))\right\Vert^2\right)
        \\
        = &\mE\mL(\bu(t))-\eta \left(1-\left(\frac{1}{2\lambda}+\frac{H\sigmax^4}{2b\gamma^2}\right)\eta\right)\mE\Vert \nabla \mL(\bw(t)) \Vert^2+\frac{\lambda H^2\beta^2\eta^2\sigmax^4}{2N^2(1-\beta)} \left(\sum_{s=1}^{t-1} \beta^{t-1-s}\mE\left\Vert  \nabla \mL_{\bB(s)}(\bw(s))\right\Vert^2\right)
        \\
        \le&\mE\mL(\bu(t))-\eta \left(1-\left(\frac{1}{2\lambda}+\frac{H\sigmax^4}{2b\gamma^2}\right)\eta\right)\mE\Vert \nabla \mL(\bw(t)) \Vert^2+\frac{\lambda H^2\beta^2\eta^2\sigmax^6}{2Nb\gamma^2(1-\beta)} \left(\sum_{s=1}^{t-1} \beta^{t-1-s}\mE\left\Vert  \nabla \mL(\bw(s))\right\Vert^2\right),
    \end{align*}
    \end{small}
where the last inequality is due to
 Lemma \ref{lem: first_second}. Lettting $\lambda=\frac{\sqrt{Nb}\gamma (1-\beta)}{H\beta\sigmax^3}$ then leads to
 \begin{small}
 \begin{align*}
     &\mE[ \mL(\bu(t+1))]
     \\
     \le &\mE\mL(\bu(t))-\eta \left(1-\left(\frac{H\beta\sigmax^3}{2\sqrt{Nb}\gamma (1-\beta)}+\frac{H\sigmax^4}{2b\gamma^2}\right)\eta\right)\mE\Vert \nabla \mL(\bw(t)) \Vert^2+\frac{ H\beta\eta^2\sigmax^3}{2\sqrt{Nb}\gamma} \left(\sum_{s=1}^{t-1} \beta^{t-1-s}\mE\left\Vert  \nabla \mL(\bw(s))\right\Vert^2\right).
 \end{align*}
  
 \end{small}
 
 By the learning rate upper bound
 $
     \eta\le \frac{1}{\frac{H\beta\sigmax^3}{\sqrt{Nb}\gamma (1-\beta)}+\frac{H\sigmax^4}{2b\gamma^2}}$,
 summing the above inequality over $t$ then leads to 
    \begin{align*}
        &\mE[ \mL(\bu(T+1))]
        \\
        \le & \mL(\bu(1))-\eta\sum_{t=1}^T \left(1-\left(\frac{H\beta\sigmax^3}{\sqrt{Nb}\gamma (1-\beta)}+\frac{H\sigmax^4}{2b\gamma^2}\right)\eta\right)\mE\Vert \nabla \mL(\bw(t)) \Vert^2
        \\
        =&\mL(\bu(1))-\eta C_2\sum_{t=1}^T\mE\Vert \nabla \mL(\bw(t)) \Vert^2,
    \end{align*}
    where $C_2\triangleq\left(1-\left(\frac{H\beta\sigmax^3}{\sqrt{Nb}\gamma (1-\beta)}+\frac{H\sigmax^4}{2b\gamma^2}\right)\eta\right)$.
    
    The proof is completed.
    % \begin{align*}
    %     &\left\Vert \sum_{s=1}^{t-1} \beta^{t-1-s} \nabla \mL_{\bB(s)}(\bw(s))\right\Vert
    %     =\left\Vert \sum_{s=1}^{t-1} \beta^{t-1-s} \left(\sum_{\bx\in \tbS} \ell'(\langle \bw(s),\bx\rangle)\bx\right)\right\Vert
    %     \\
    %     \ge &\left\langle  \sum_{s=1}^{t-1} \beta^{t-1-s} \left(\sum_{\bx\in \tbS} \ell'(\langle \bw(s),\bx\rangle)\bx\right), -\gamma \hbw\right\rangle
    %     \ge \gamma \sum_{s=1}^{t-1} \beta^{t-1-s} \left(\sum_{\bx\in \tbS} \ell'(\langle \bw(s),\bx\rangle)\right)
    %     \\
    %     \ge &\gamma \sum_{s=1}^{t-1} \beta^{t-1-s} \Vert \nabla \mL(\bw(s))\Vert.
    % \end{align*}
\end{proof}

As $\mL(\bu(1))$ is upper bounded, we have the following corollary given by Lemma \ref{lem: sgdm_loss_update}.
 
\begin{corollary}
\label{coro: sgdm_gradient_finite}
Let all conditions in Theorem \ref{thm: sgdm_main} hold. Then, we have 
\begin{equation}
\label{eq: sum_finite_sgdm}
    \sum_{t=1}^{\infty} \mE\Vert \nabla \mL(\bw(t)) \Vert^2<\infty.
\end{equation}
Consequently, 
\begin{equation*}
     \sum_{t=1}^{\infty} \Vert \nabla \mL(\bw(t)) \Vert^2<\infty
\end{equation*}
and 
\begin{equation*}
    \langle\bw(t),\bx \rangle \rightarrow \infty, \forall \bx\in \tbS
\end{equation*}
hold almost surely.
\end{corollary}
\begin{proof}
    By Lemma \ref{lem: sgdm_loss_update}, we have for any $T>1$, 
    \begin{equation*}
        \sum_{t=1}^T C_2\eta \mE\Vert \nabla \mL(\bw(t)) \Vert^2\le \mL(\bu(1))- \mE[\mL(\bu(T+1))]\le \mL(\bu(1))<\infty,
    \end{equation*}
    which completes the proof of Eq. (\ref{eq: sum_finite_sgdm}). The rest of claims follows immediately by Fubini's Theorem and Assumption \ref{assum: exponential-tailed}.
    
    The proof is completed.
\end{proof}

\subsubsection{Parameter dynamics}
\label{appen: orthogonal_sgdm}
Similar to the case of GDM, we define $\tbw$ as the solution of Eq. (\ref{eq: represent_v}) with ${   C_3}=\frac{\eta}{(1-\beta)N}$. We also let $\bn(t)$ be given by Lemma \ref{lem: define_n_t}, and define $\br(t)$ in this case as
\begin{equation}
\label{eq: r_sgdm}
    \br(t)\overset{\triangle}{=}\bw(t)-\ln (t)\hbw-\tbw-\bn(t).
\end{equation}

As $\tbw$ is a constant vector, and $\Vert \bn(t) \Vert \rightarrow 0$  as $t\rightarrow \infty$, we have $\bw(t)-\ln(t)\hbw$ has bounded norm if and only if $ \Vert \br(t) \Vert$ is upper bounded. Similar to the GDM case, we have the following equivalent condition of that $\Vert \br(t) \Vert $ is bounded.
\begin{lemma}
\label{lem:sgdm_r_t_bounded}
Let all conditions in Theorem \ref{thm: sgdm_main} hold. Then, $\Vert \br(t) \Vert$ is bounded almost surely if and only if function $g(t)$ is upper bounded almost surely, where $g:\mathbb{Z}^{+}\rightarrow\mathbb{R}$ is defined as
\begin{small}
\begin{equation}
    \label{eq: definition_of_g_sgdm}
    g(t)\overset{\triangle}{=}\frac{1}{2}\Vert \br(t)\Vert^2+\frac{\beta}{1-\beta}\langle \br(t),\bw(t)-\bw(t-1)\rangle-\frac{\beta}{1-\beta}\sum_{\tau=2}^t\langle \br(\tau)-\br(\tau-1),\bw(\tau)-\bw(\tau-1)\rangle.
\end{equation}
\end{small}
\end{lemma}

\begin{proof}
    To begin with, we prove that almost surely $\vert \sum_{\tau=2}^t\langle \br(\tau)-\br(\tau-1),\bw(\tau)-\bw(\tau-1)\rangle\vert$ is upper bounded for any $t$. By Corollary \ref{coro: sgdm_gradient_finite}, we have almost surly
    \begin{equation*}
        \sum_{t=1}^{\infty} \Vert \nabla \mL(\bw(t)) \Vert^2< \infty.
    \end{equation*}
    On the other hand, for any $\bw$, we have 
    \begin{align*}
    &\Vert \nabla \mL_{\bB(t)}(\bw) \Vert =\frac{1}{b} \left\Vert \sum_{\bx\in \bB(t)} \ell'(\langle \bw, \bx\rangle) \bx \right\Vert 
    \\
    \le & -\frac{\sigmax}{b} \sum_{\bx \in \bB(t)} \ell'(\langle \bw, \bx\rangle)
    < -\frac{\sigmax}{b} \sum_{\bx \in \bS} \ell'(\langle \bw, \bx\rangle)
    \\
    \le &-\frac{\sigmax}{b} \sum_{\bx \in \bS} \ell'(\langle \bw, \bx\rangle) \langle {  \hbw},\bx \rangle
    \le \frac{N\sigmax}{b}\left\Vert \frac{1}{N}\sum_{\bx \in \bS} \ell'(\langle \bw, \bx\rangle)\bx\right\Vert \Vert {  \hbw} \Vert
    \\
    = &\frac{N\sigmax}{b\gamma}\left\Vert \nabla \mL(\bw)\right\Vert.
    \end{align*}
    
    Therefore, we have almost surely, 
    \begin{equation*}
        \sum_{t=1}^{\infty} \left\Vert \nabla \mL_{\bB(t)}(\bw(t)) \right\Vert^2< \infty,
    \end{equation*}
    which further leads to almost surely
    \begin{align*}
        &\sum_{t=1}^{\infty} \Vert \bw(t+1)-\bw(t)\Vert^2
        \le  \eta^2(1-\beta)^2\sum_{t=1}^{\infty} \left\Vert \sum_{s=1}^{t} \beta^{t-s} \nabla \mL_{\bB(s)}(\bw(s)) \right\Vert^2 
        \\
        \le & \eta^2(1-\beta)^2\sum_{t=1}^{\infty}  \left(\sum_{s=1}^{t} \beta^{t-s} \left\Vert\nabla \mL_{\bB(s)}(\bw(s)) \right\Vert \right)^2
        \\
        \le &\eta^2(1-\beta)^2\sum_{t=1}^{\infty}  \left(\sum_{s=1}^{t} \beta^{t-s} \left\Vert\nabla \mL_{\bB(s)}(\bw(s)) \right\Vert^2 \right)\left(\sum_{s=1}^{t} \beta^{t-s}  \right)
        \\
        \le & \eta^2\sum_{s=1}^{\infty} \left\Vert\nabla \mL_{\bB(s)}(\bw(s)) \right\Vert^2  <\infty.
    \end{align*}
    
    By the definition of $\br(t)$ (Eq. (\ref{eq: r_sgdm})), we further have
    \begin{align*}
        &\left\vert \sum_{\tau=2}^t\langle \br(\tau)-\br(\tau-1),\bw(\tau)-\bw(\tau-1)\rangle\right\vert
        \\
        \le &\sum_{\tau=2}^t\vert \langle \br(\tau)-\br(\tau-1),\bw(\tau)-\bw(\tau-1)\rangle\vert
        \\
        =&\sum_{\tau=2}^t\left\vert \left\langle \bw(\tau)-\bw(\tau-1)-\ln\left(\frac{\tau+1}{\tau}\right)-(\bn(\tau)-\bn(\tau-1)),\bw(\tau)-\bw(\tau-1)\right\rangle\right\vert
        \\
        \le &\sum_{\tau=2}^t \Vert \bw(\tau)-\bw(\tau-1)\Vert^2+\sum_{\tau=2}^t \left\vert \left\langle -\ln\left(\frac{\tau+1}{\tau}\right)-(\bn(\tau)-\bn(\tau-1)),\bw(\tau)-\bw(\tau-1)\right\rangle\right\vert
        \\
        \le &\frac{3}{2}\sum_{\tau=2}^t \Vert \bw(\tau)-\bw(\tau-1)\Vert^2+\frac{1}{2}\sum_{\tau=2}^t \left\Vert  -\ln\left(\frac{\tau+1}{\tau}\right)-(\bn(\tau)-\bn(\tau-1))\right\Vert^2
        \\
        \overset{(\star)}{\le}& \frac{3}{2}\sum_{\tau=2}^t \Vert \bw(\tau)-\bw(\tau-1)\Vert^2+\frac{1}{2}\sum_{\tau=2}^t  \mathcal{O}\left(\frac{1}{\tau}\right)^2<\infty,
    \end{align*}
    where Inequality $(\star)$ is due to $\Vert\bn(\tau)-\bn(\tau-1)\Vert=\mathcal{O}(\frac{1}{
    \tau})$ and $\ln \frac{\tau+1}{\tau}=\mathcal{O}(\frac{1}{\tau})$.
    
    Therefore, $g(t)$ is upper bounded almost surely is equivalent to $\frac{1}{2} \Vert \br(t)\Vert^2+\frac{\beta}{1-\beta}\langle \br(t),\bw(t)-\bw(t-1)\rangle$ is upper bounded, which can be shown to be equivalent with $\Vert \br(t) \Vert$ is bounded following the same routine as Lemma \ref{lem:gdm_r_t_bounded}.
    
    The proof is completed.
\end{proof}

As the case of GDM, we only need to prove $g(t)$ is upper bounded to complete the proof of Theorem \ref{thm: sgdm_main}.

\begin{lemma}
\label{lem: sgdm_ib}
Let all conditions in Theorem \ref{thm: sgdm_main} hold. Then, for almost every dataset, we have $g(t)$ is upper bounded.
\end{lemma}

\begin{proof}
    Following the same routine as Lemma \ref{lem:gdm_r_t_bounded}, we have
    \begin{align*}
         &g(t+1)-g(t)
     \\
     =&\frac{1}{2}\Vert \br(t+1)-\br(t)\Vert^2+ \langle \br(t),\br(t+1)-\br(t)\rangle+\left\langle \br(t),-\eta\nabla \mL_{\bB(t)} (\bw(t))-(\bw(t+1)-\bw(t))\right\rangle,
    \end{align*}
    where $\sum_{t=1}^{\infty} \Vert \br(t+1)-\br(t) \Vert^2$ is upper bounded.
    
    On the other hand, by the definition of $\br(t)$ (Eq. (\ref{eq: r_sgdm})), we have
    \begin{align*}
        &\br(t+1)-\br(t)
        \\
        =&\bw(t+1)-\bw(t)-\ln\left(\frac{t+1}{t}\right)\hbw-\bn(t+1)+\bn(t),
    \end{align*}
    while by Lemma \ref{lem: define_n_t},
    \begin{align*}
        \frac{N}{b}\frac{1}{t}\sum_{i: \bx_i\in \bB(t)\cap \bS_s} \bv_i \bx_i=\ln\left(\frac{t+1}{t}\right)\hbw+\bn(t+1)-\bn(t).
    \end{align*}
    
    Combining the above two equations, we further have
    \begin{equation*}
        \br(t+1)-\br(t)=\bw(t+1)-\bw(t)-\frac{N}{bt}\sum_{i: \bx_i\in \bB(t)\cap \bS_s} \bv_i \bx_i,
    \end{equation*}
    which further indicates
    \begin{equation*}
        g(t+1)-g(t)=\frac{1}{2}\Vert \br(t+1)-\br(t)\Vert^2+\left\langle \br(t), -\eta\nabla \mL_{\bB(t)} (\bw(t))-\frac{N}{bt}\sum_{i: \bx_i\in \bB(t)\cap \bS_s} \bv_i\bx_i\right\rangle.
    \end{equation*}
    
    Therefore, we only need to prove $\sum_{t=1}^\infty\langle \br(t), -\eta\nabla \mL (\bw(t))-\frac{N}{bt}\sum_{i: \bx_i\in \bB(t)\cap \bS_s} \bv_i\bx_i\rangle<\infty$. By directly applying the form of $\nabla \mL(\bw(t))$, we have 
    \begin{align*}
        &\left\langle \br(t), -\eta\nabla \mL_{\bB(t)} (\bw(t))-\frac{N}{bt}\sum_{i: \bx_i\in \bB(s)\cap \bS_s} \bv_i\bx_i\right\rangle
        \\
        =&\left\langle \br(t), -\frac{\eta}{(1-\beta)b}\sum_{i: \bx_i\in \bB(s)} \ell'(\langle \bw(t),\bx_i\rangle)\bx_i-\frac{N}{bt}\sum_{i: \bx_i\in \bB(s)\cap \bS_s} \bv_i\bx_i\right\rangle
        \\
        =&\frac{\eta}{  (1-\beta)b}\left\langle \br(t), -\sum_{i: \bx_i\in \bB(s)} \ell'(\langle \bw(t),\bx_i\rangle)\bx_i-\frac{N(1-\beta)}{\eta}\frac{1}{t}\sum_{i: \bx_i\in \bB(s)\cap \bS_s} \bv_i\bx_i\right\rangle
        \\
        =&\frac{\eta}{  (1-\beta)b}\left\langle \br(t), -\sum_{i: \bx_i\in \bB(s)} \ell'(\langle \bw(t),\bx_i\rangle)\bx_i-\frac{1}{t}\sum_{i: \bx_i\in \bB(s)\cap \bS_s} e^{-\langle \tbw,\bx_i\rangle}\bx_i\right\rangle
        \\
        =&\frac{\eta}{  (1-\beta)b}\sum_{i: \bx_i\in \bB(s)\cap \bS_s}\left(- \ell'(\langle \bw(t),\bx_i\rangle)-\frac{1}{t} e^{-\langle \tbw,\bx_i\rangle}\right)\langle \br(t), \bx_i\rangle
        \\
        +&\frac{\eta}{  (1-\beta)b}\sum_{i: \bx_i\in \bB(s)\cap \bS_s^c}\left\langle \br(t),  -\ell'(\langle \bw(t),\bx_i\rangle)\bx_i\right\rangle.
    \end{align*}
    Let $A_6(t)= \sum_{i: \bx_i\in \bB(s)\cap \bS_s}\left(- \ell'(\langle \bw(t),\bx_i\rangle)-\frac{1}{t} e^{-\langle \tbw,\bx_i\rangle}\right)\langle \br(t), \bx_i\rangle$, and $A_7(t)=\sum_{i: \bx_i\in \bB(s)\cap \bS_s^c}$ $\left\langle \br(t),  -\ell'(\langle \bw(t),\bx_i\rangle)\bx_i\right\rangle$. We will investigate these two terms respectively.
    
    As $\left\langle \bw(t),\bx\right\rangle\rightarrow\infty$, $\forall \bx\in \bS$, a.s., we have a.s., there exists a large enough time $t_0$, s.t., $\forall t\ge t_0$,  $\forall \bx\in \bS$,
    \begin{align*}
         -\ell'(\langle \bw(t),\bx\rangle)\le (1+e^{-\mu_+ \langle \bw(t),\bx\rangle})e^{-\langle \bw(t),\bx\rangle},
        \\
        -\ell'(\langle \bw(t),\bx_i\rangle)\ge (1-e^{-\mu_- \langle \bw(t),\bx\rangle})e^{-\langle \bw(t),\bx\rangle},
        \\
        \langle \bx,\bw(t)\rangle>0.
    \end{align*}
    
    Therefore, 
    \begin{align*}
        A_7(t)\le &\sum_{i: \bx_i\in \bB(s)\cap \bS_s^c}-\ell'(\langle \bw(t),\bx_i\rangle)\langle \br(t),  \bx_i\rangle \mathds{1}_{\langle \br(t), \bx_i \rangle\ge 0}
        \\
        \le &\sum_{i: \bx_i\in \bB(s)\cap \bS_s^c}(1+e^{-\mu_+ \langle \bw(t),\bx_i\rangle})e^{-\langle \bw(t),\bx_i\rangle}\langle \br(t),  \bx_i\rangle \mathds{1}_{\langle \br(t), \bx_i \rangle\ge 0}
        \\
        \le &2\sum_{i: \bx_i\in \bB(s)\cap \bS_s^c}e^{-\langle \br(t)+\ln(t)\hbw+\tbw+\bn(t),\bx_i\rangle}\langle \br(t),  \bx_i\rangle \mathds{1}_{\langle \br(t), \bx_i \rangle\ge 0}
        \\
        \overset{(\star)}{\le}&2\sum_{i: \bx_i\in \bB(s)\cap \bS_s^c}\frac{1}{t^{\theta}}e^{-\langle \tbw+\bn(t),\bx_i\rangle}e^{-\langle \br(t),\bx_i\rangle}\langle \br(t),  \bx_i\rangle \mathds{1}_{\langle \br(t), \bx_i \rangle\ge 0}
        \\
        \overset{(\dagger)}{\le} &\frac{2}{e}\frac{1}{t^{\theta}}\sum_{i: \bx_i\in \bB(s)\cap \bS_s^c}e^{-\langle \tbw+\bn(t),\bx_i\rangle} \mathds{1}_{\langle \br(t), \bx_i \rangle\ge 0}
        \\
        \overset{(\circ)}{=}& \mathcal{O}\left(\frac{1}{t^{\theta}}\right),
    \end{align*}
    where Inequality. $(\star)$ is due the definition of $\theta$ (Eq. (\ref{eq: def_theta})), {  Inequality. ($\dagger$) is due to $e^{-\langle \br(t),\bx_i\rangle}\langle \br(t),  \bx_i\rangle\le e^{-1}$ }, and Eq. $(\circ)$ is due to   $ \lim_{t\rightarrow\infty} e^{-\langle \tbw+\bn(t),\bx_i\rangle}=e^{-\langle \tbw,\bx_i\rangle}$. Thus, 
    \begin{equation*}
        \sum_{t=1}^{\infty }A_7(t)<\infty.
    \end{equation*}
    
    On the other hand, $A_6(t)$ can be rewritten as 
    \begin{align*}
        A_6(t)=&\sum_{i: \bx_i\in \bB(s)\cap \bS_s}\left(- \ell'(\langle \bw(t),\bx_i\rangle)-\frac{1}{t} e^{-\langle \tbw,\bx_i\rangle}\right)\langle \br(t), \bx_i\rangle\mathds{1}_{\langle\br(t),\bx_i\rangle\ge 0}
        \\
        +&\sum_{i: \bx_i\in \bB(s)\cap \bS_s}\left(- \ell'(\langle \bw(t),\bx_i\rangle)-\frac{1}{t} e^{-\langle \tbw,\bx_i\rangle}\right)\langle \br(t), \bx_i\rangle\mathds{1}_{\langle\br(t),\bx_i\rangle< 0}.
    \end{align*}
    
    If $\langle\br(t),\bx_i\rangle\ge 0$,  we have for $\varepsilon<0.5$,
    \begin{align*} 
        &\left(- \ell'(\langle \bw(t),\bx_i\rangle)-\frac{1}{t} e^{-\langle \tbw,\bx_i\rangle}\right)\langle \br(t), \bx_i\rangle
        \\
        \le &\left(\left(1+e^{-\mu_+ \langle \bw(t),\bx_i\rangle}\right)e^{-\langle \bw(t),\bx_i\rangle}-\frac{1}{t} e^{-\langle \tbw,\bx_i\rangle}\right) \langle \br(t), \bx_i\rangle
        \\
        =& \left(\left(1+e^{-\mu_+ \langle \br(t)+\ln(t)\hbw+\tbw+\bn(t),\bx_i\rangle}\right)e^{-\langle \br(t)+\ln(t)\hbw+\tbw+\bn(t),\bx_i\rangle}-\frac{1}{t} e^{-\langle \tbw,\bx_i\rangle}\right) \langle \br(t), \bx_i\rangle
        \\
        =&\left(\left(1+e^{-\mu_+ \langle \br(t)+\ln(t)\hbw+\tbw+\bn(t),\bx_i\rangle}\right)e^{-\langle \br(t)+\bn(t),\bx_i\rangle}- 1\right) \frac{1}{t}\langle \br(t), \bx_i\rangle e^{-\langle \tbw,\bx_i\rangle}
        \\
        \le &\left(\left(1+\frac{1}{t^{\mu_{+}}}e^{-\mu_+ \langle \tbw+\bn(t),\bx_i\rangle}\right)e^{-\langle \br(t)+\bn(t),\bx_i\rangle}- 1\right) \frac{1}{t}\langle \br(t), \bx_i\rangle e^{-\langle \tbw,\bx_i\rangle}
        \\
        \overset{(\bullet)}{=}& \left(\left(1+\mathcal{O}\left(\frac{1}{t^{\mu_{+}}}\right)\right)\left(1+\mathcal{O}\left( \frac{1}{t^{0.5-\varepsilon}}\right)\right)e^{-\langle \br(t),\bx_i\rangle}- 1\right) \frac{1}{t}\langle \br(t), \bx_i\rangle e^{-\langle \tbw,\bx_i\rangle}
        \\
        =&\left(e^{-\langle \br(t),\bx_i\rangle}- 1\right) \frac{1}{t}\langle \br(t), \bx_i\rangle e^{-\langle \tbw,\bx_i\rangle}+\frac{1}{t}\mathcal{O}\left(\frac{1}{t^{\min\{\mu_{+},0.5-\varepsilon\}}}\right)e^{-\langle \br(t), \bx_i\rangle}\langle \br(t), \bx_i\rangle e^{-\langle \tbw,\bx_i\rangle}
        \\
        \le & \frac{1}{t}\mathcal{O}\left(\frac{1}{t^{\min\{\mu_{+},0.5-\varepsilon\}}}\right)e^{-\langle \br(t), \bx_i\rangle}\langle \br(t), \bx_i\rangle e^{-\langle \tbw,\bx_i\rangle}
        \\
        \overset{(\diamond)}{=}&\mathcal{O}\left(\frac{1}{t^{\min\{1+\mu_{+},1.5-\varepsilon\}}}\right),
    \end{align*}
    where Eq. $(\bullet)$ is due to $\bn(t)=\mathcal{O}(\frac{1}{t^{0.5{  -}\varepsilon}})$, and Eq. ($\diamond$) is due to $e^{-\langle \br(t), \bx_i\rangle}\langle \br(t), \bx_i\rangle \le \frac{1}{e}$.
    
    On the other hand, if $\langle \br(t), \bx_i\rangle <0$, we have
    \begin{align*}
        &\left(- \ell'(\langle \bw(t),\bx_i\rangle)-\frac{1}{t} e^{-\langle \tbw,\bx_i\rangle}\right)\langle \br(t), \bx_i\rangle
        \\
        \le &\left(\left(1-e^{-\mu_{-} \langle \bw(t),\bx_i\rangle}\right)e^{-\langle \bw(t),\bx_i\rangle}-\frac{1}{t} e^{-\langle \tbw,\bx_i\rangle}\right) \langle \br(t), \bx_i\rangle
        \\
        =&\frac{1}{t}e^{-\langle \tbw, \bx_i\rangle} \left(-1+\left(1-e^{-\mu_{-} \langle \bw(t),\bx_i\rangle}\right)e^{-\langle \br(t)+\bn(t),\bx_i\rangle}\right)\left\langle \br(t), \bx_i\right\rangle
    \end{align*}

Specifically, if $\langle \br(t), \bx_i\rangle\ge -t^{-0.5\min\{\mu_{-},0.5\}}$, 
\begin{align*}
    &\left\vert \frac{1}{t}e^{-\langle \tbw, \bx_i\rangle} \left(-1+\left(1-e^{-\mu_{-} \langle \bw(t),\bx_i\rangle}\right)e^{-\langle \br(t)+\bn(t),\bx_i\rangle}\right)\left\langle \br(t), \bx_i\right\rangle\right\vert 
    \\
    \le & \frac{1}{  t^{1+0.5\min\{\mu_{-},0.5\}}}e^{-\langle \tbw, \bx_i\rangle}\left\vert-1+\left(1-e^{-\mu_{-} \langle \bw(t),\bx_i\rangle}\right)e^{-\langle \br(t)+\bn(t),\bx_i\rangle}\right\vert
    \\
    \overset{(\square)}{=} &\mathcal{O}\left(\frac{1}{t^{1+0.5\min\{\mu_{-},0.5\}}}\right),
\end{align*}
{  where Eq. ($\square$) is due to that as $\langle \bw(t),\bx_i\rangle \rightarrow\infty$ and $t^{-0.5\min\{\mu_{-},0.5\}}\rightarrow 0$ as $t\rightarrow\infty$, there exists a large enough time $T$, s.t., $\forall t>T$, under the circumstance $0>\langle \br(t), \bx_i\rangle\ge -t^{-0.5\min\{\mu_{-},0.5\}}$, $e^{-\langle \br(t)+\bn(t),\bx_i\rangle}<1$ and $e^{-\mu_{-} \langle \bw(t),\bx_i\rangle}<1$.}

If $-2\le \langle\br(t),\bx_i\rangle<-t^{-0.5\min\{\mu_{-},0.5\}}$, then, for large enough $t$,  $\vert\langle \bx_i, \bn(t)\rangle\vert<2$, $1-\frac{e^{{  \mu_{-}}(-\langle \tbw,\bx_i\rangle+4)}}{t^{\mu_{-}}}>0$, and 
\begin{align*}
    & \frac{1}{t}e^{-\langle \tbw, \bx_i\rangle} \left(-1+\left(1-e^{-\mu_{-} \langle \bw(t),\bx_i\rangle}\right)e^{-\langle \br(t)+\bn(t),\bx_i\rangle}\right)\left\langle \br(t), \bx_i\right\rangle
    \\
    = &\frac{1}{t}e^{-\langle \tbw, \bx_i\rangle} \left(-1+\left(1-e^{-\mu_{-} \langle \br(t)+\ln(t)\hbw+\tbw+\bn(t),\bx_i\rangle}\right)e^{-\langle \br(t)+\bn(t),\bx_i\rangle}\right)\left\langle \br(t), \bx_i\right\rangle
    \\
    = &\frac{1}{t}e^{-\langle \tbw, \bx_i\rangle} \left(-1+\left(1-\frac{e^{-{  \mu_{-}}\langle \tbw,\bx_i\rangle}}{t^{\mu_{-}}}e^{-\mu_{-} \langle \br(t)+\bn(t),\bx_i\rangle}\right)e^{-\langle \br(t)+\bn(t),\bx_i\rangle}\right)\left\langle \br(t), \bx_i\right\rangle
    \\
    \le &\frac{1}{t}e^{-\langle \tbw, \bx_i\rangle} \left(-1+\left(1-\frac{e^{{  \mu_{-}}(-\langle \tbw,\bx_i\rangle+4)}}{t^{\mu_{-}}}\right)e^{- \langle \br(t)+\bn(t),\bx_i\rangle}\right)\left\langle \br(t), \bx_i\right\rangle
    \\
    \le &\frac{1}{t}e^{-\langle \tbw, \bx_i\rangle} \left(-1+\left(1-\frac{e^{{  \mu_{-}}(-\langle \tbw,\bx_i\rangle+4)}}{t^{\mu_{-}}}\right)\left(1- \langle \br(t)+\bn(t),\bx_i\rangle\right)\right)\left\langle \br(t), \bx_i\right\rangle
    \\
    \le &\frac{1}{t}e^{-\langle \tbw, \bx_i\rangle} \left(-1+\left(1-\frac{e^{{  \mu_{-}}(-\langle \tbw,\bx_i\rangle+4)}}{t^{\mu_{-}}}\right)\left(1+t^{-0.5\min\{\mu_{-},0.5\}}- \langle \bn(t),\bx_i\rangle\right)\right)\left\langle \br(t), \bx_i\right\rangle
    \\
    =&\frac{1}{t}e^{-\langle \tbw, \bx_i\rangle} \left(-1+\left(1-\frac{e^{{  \mu_{-}}(-\langle \tbw,\bx_i\rangle+4)}}{t^{\mu_{-}}}\right)\left(1+t^{-0.5\min\{\mu_{-},0.5\}} +\boldsymbol{o}\left(t^{-0.5\min\{\mu_{-},0.5\}}\right)\right)\right)\left\langle \br(t), \bx_i\right\rangle
    \\
    = &\frac{1}{t}e^{-\langle \tbw, \bx_i\rangle} \left(-1+1+t^{-0.5\min\{\mu_{-},0.5\}}+\boldsymbol{o}\left(t^{-0.5\min\{\mu_{-},0.5\}}\right)\right)\left\langle \br(t), \bx_i\right\rangle<0.
\end{align*}

If $-2> \langle\br(t),\bx_i\rangle$, then for large enough time $t$, $e^{-\langle \br(t)+\bn(t),\bx_i\rangle}\ge e^{\frac{3}{2}}$, $1-e^{-\mu_{-} \langle \bw(t),\bx_i\rangle}\ge e^{-\frac{1}{2}}$, and
\begin{align*}
    &\frac{1}{t}e^{-\langle \tbw, \bx_i\rangle} \left(-1+\left(1-e^{-\mu_{-} \langle \bw(t),\bx_i\rangle}\right)e^{-\langle \br(t)+\bn(t),\bx_i\rangle}\right)\left\langle \br(t), \bx_i\right\rangle
    \\
    \le &\frac{1}{t}e^{-\langle \tbw, \bx_i\rangle} \left(-1+e\right)\left\langle \br(t), \bx_i\right\rangle<0.
\end{align*}

Conclusively, if $ \langle\br(t),\bx_i\rangle<0$, for large enough $t$, we have 
\begin{equation*}
    \left(- \ell'(\langle \bw(t),\bx_i\rangle)-\frac{1}{t} e^{-\langle \tbw,\bx_i\rangle}\right)\langle \br(t), \bx_i\rangle\le \mathcal{O}\left(\frac{1}{t^{1+0.5\min\{\mu_{-},0.5\}}}\right),
\end{equation*}
which further indicates, for large enough $t$, we have
\begin{equation*}
    A_6(t)\le \max\left\{\mathcal{O}\left(\frac{1}{t^{1+0.5\min\{\mu_{-},0.5\}}}\right),\mathcal{O}\left(\frac{1}{t^{\min\{1+\mu_{+},1.5-\varepsilon\}}}\right)\right\},
\end{equation*}
which indicates
\begin{equation*}
    \sum_{t=1}^{\infty} A_6(t)<\infty.
\end{equation*}

Therefore, 
\begin{align*}
    &\sum_{t=1}^{\infty} \left(g(t+1)-g(t)\right)
    \\
    =&\sum_{t=1}^{\infty}  \left( \frac{1}{2}\Vert \br(t+1)-\br(t)\Vert^2+\left\langle \br(t), -\eta\nabla \mL_{\bB(t)} (\bw(t))-\frac{N}{bt}\sum_{i: \bx_i\in \bB(t)\cap \bS_s} \bv_i\bx_i\right\rangle\right)
    \\
    =&\sum_{t=1}^{\infty}  \left( \frac{1}{2}\Vert \br(t+1)-\br(t)\Vert^2+\eta A_6(t)+\eta A_7(t)\right)
    \\
    <&\infty.
\end{align*}

The proof is completed.
\end{proof}
\subsubsection{Explanation for proper lyapunov function}
\label{subsec: explanation}
Based on the success of applying Lyapunov function $\mL(\bw(t))+\frac{\beta}{2\eta}\Vert \bw(t)-\bw(t-1)\Vert^2$ to analyze gradient descent with momentum, it is natural to try to extend this routine to analyze stochastic gradient descent with momentum. However, in this section, we will show such Lyapunov function is not proper to analyze SGDM as this will put constraints on the range of the momentum rate $\beta$. Specifically, at any step $t$, since the loss $\mL$ is $\frac{H\sigmax^2}{N}$ smooth at $\bw(t)$, we can expand the loss $\mL$ in the same way as the GDM case:
\begin{align*}
    \mL(\bw(t+1))\le& \mL(\bw(t))+\langle \bw(t+1)-\bw(t),\nabla \mL(\bw(t))\rangle+\frac{H\sigmax^2}{2N} \Vert \bw(t+1)-\bw(t) \Vert^2.
\end{align*}

By taking expectation with respect to $\bw(t+1)$ conditioning on $\{\bw(s)\}_{s=1}^t$ for both sides, we further obtain
\begin{align*}
     &\mathbb{E}\left[\mL(\bw(t+1))\left|\mF_t\right.\right]
     \\
     \le& \mL(\bw(t))+\langle\mathbb{E} \left[\bw(t+1)-\bw(t)\left|\mF_t\right.\right],\nabla \mL(\bw(t))\rangle+\frac{H\sigmax^2}{2N} \mathbb{E} \left[\Vert \bw(t+1)-\bw(t) \Vert^2\left|\mF_t\right.\right]
     \\
     \overset{(\star)}{=}&\mL(\bw(t))+\frac{1}{(1-\beta)\eta}\left\langle\mathbb{E} \left[\bw(t+1)-\bw(t)\left|\mF_t\right.\right],\beta\left(\bw(t)-\bw(t-1)\right)-\mathbb{E}\left[\bw(t+1)-\bw(t)\left|\mF_t\right.\right]\right\rangle
     \\
     &+\frac{H\sigmax^2}{2N} \mathbb{E} \left[\Vert \bw(t+1)-\bw(t) \Vert^2\left|\mF_t\right.\right]
     \\
     =&\mL(\bw(t))+\frac{\beta}{(1-\beta)\eta}\left\langle\left(\bw(t)-\bw(t-1)\right),\mathbb{E}\left[\bw(t+1)-\bw(t)\left|\mF_t\right.\right]\right\rangle
     \\
     &+\frac{H\sigmax^2}{2N} \mathbb{E} \left[\Vert \bw(t+1)-\bw(t) \Vert^2\left|\mF_t\right.\right]-\frac{1}{(1-\beta)\eta}\left\Vert\mathbb{E} \left[\bw(t+1)-\bw(t)\left|\mF_t\right.\right]\right\Vert^2
     \\
     \le &\mL(\bw(t))+\frac{\beta}{2(1-\beta)\eta}\left\Vert\bw(t)-\bw(t-1)\right\Vert^2+\frac{\beta}{2(1-\beta)\eta}\left\Vert\mathbb{E}\left[\bw(t+1)-\bw(t)\left|\mF_t\right.\right]\right\Vert^2
     \\
     &+\frac{H\sigmax^2}{2N} \mathbb{E} \left[\Vert \bw(t+1)-\bw(t) \Vert^2\left|\mF_t\right.\right]-\frac{1}{(1-\beta)\eta}\left\Vert\mathbb{E} \left[\bw(t+1)-\bw(t)\left|\mF_t\right.\right]\right\Vert^2,
\end{align*}
where Eq. $(\star)$ is becasue $\mathbb{E}\left[\bw(t+1)-\bw(t)\left|\mF_t\right.\right]=-(1-\beta)\eta \nabla \mL(\bw(t))+\beta\left(\bw(t)-\bw(t-1)\right)$ due to the definition of SGDM (Eq. (\ref{eq: def_sgdm})). Rearranging the above inequality and taking expectations of both sides with respect to $\{\bw(s)\}_{s=1}^t$ leads to 
\begin{align}
\nonumber
    &\mathbb{E}\left[\mL(\bw(t+1))\right]+\frac{2-\beta}{2(1-\beta)\eta}\mE\left\Vert\mathbb{E} \left[\bw(t+1)-\bw(t)\left|\mF_t\right.\right]\right\Vert^2
    \\
\nonumber
    -&\frac{H\sigmax^2}{2N} \mE \left[\Vert \bw(t+1)-\bw(t) \Vert^2\right]
    \\
\label{eq: lyapunov_fake}
    \le & \mE\mL(\bw(t))+\frac{\beta}{2(1-\beta)\eta}\mE\left\Vert\bw(t)-\bw(t-1)\right\Vert^2.
\end{align}

On the other hand, we wish to obtain some positive constant $\alpha$ from Eq. (\ref{eq: lyapunov_fake}), such that (at least),
\begin{align}
\nonumber
    &\mathbb{E}\left[\mL(\bw(t+1))\right]+\alpha\mE\left\Vert\bw(t+1)-\bw(t)\right\Vert^2.
    \\
\label{eq: requirement_lyapunov}
    \le & \mE\mL(\bw(t))+\alpha\mE\left\Vert\bw(t)-\bw(t-1)\right\Vert^2,
\end{align}
which requires to lower bound $\mE\left\Vert\mathbb{E} \left[\bw(t+1)-\bw(t)\left|\mF_t\right.\right]\right\Vert^2$ by $\mE$ $\left\Vert\ \bw(t+1)-\bw(t)\right\Vert^2$. However, in general cases, $\mE\left\Vert\mathbb{E} \left[\bw(t+1)-\bw(t)\left|\mF_t\right.\right]\right\Vert^2$ is only upper bounded by $\mE$ $\left\Vert\ \bw(t+1)-\bw(t)\right\Vert^2$ (Holder's Inequality), although in our case, $\left\Vert\mathbb{E} \left[\bw(t+1)-\bw(t)\left|\mF_t\right.\right]\right\Vert^2$ can be bounded as
\begin{small}
\begin{align*}
    &\left\Vert\mathbb{E} \left[\bw(t+1)-\bw(t)\left|\mF_t\right.\right]\right\Vert^2
    \\
    =&\left\Vert -(1-\beta)\eta\nabla \mL(\bw(t))+\beta\left(\bw(t)-\bw(t-1)\right)\right\Vert^2
    \\
    =&\left\Vert -(1-\beta)\eta\nabla \mL(\bw(t))\right\Vert^2+\left\Vert\beta\left(\bw(t)-\bw(t-1)\right)\right\Vert^2+2\beta(1-\beta)\eta\langle \bw(t)-\bw(t-1),-\nabla \mL(\bw(t))\rangle,
\end{align*}
\end{small}
while by the separability of the dataset and that the loss is non-increasing, $\mathbb{E}\left[\left\Vert\ \bw(t+1)-\bw(t)\right\Vert^2\vert \mF_t\right]$ can be bounded as
\begin{small}
\begin{align}
\nonumber
    &\mathbb{E}\left[\left\Vert\ \bw(t+1)-\bw(t)\right\Vert^2\vert \mF_t\right]
    \\
%     \\
% \nonumber
%     =&\mathbb{E}_{\bB(t)}\left\Vert -(1-\beta)\eta \nabla \mL_{\bB(t)}(\bw(t)) +\beta(\bw(t)-\bw(t-1))\right\Vert^2
%     \\
% \nonumber
%     =&\mathbb{E}_{\bB(t)} \left\Vert -(1-\beta)\eta \nabla \mL_{\bB(t)}(\bw(t))\right\Vert^2+ \left\Vert \beta(\bw(t)-\bw(t-1))\right\Vert^2+2(1-\beta)\eta\beta \mathbb{E}_{\bB(t)}\langle - \nabla \mL_{\bB(t)}(\bw(t)), \bw(t)-\bw(t-1)\rangle
%     \\
% \nonumber
%     =&\mathbb{E}_{\bB(t)} \left\Vert -(1-\beta)\eta \nabla \mL_{\bB(t)}(\bw(t))\right\Vert^2+ \left\Vert \beta(\bw(t)-\bw(t-1))\right\Vert^2+2(1-\beta)\eta\beta \langle - \nabla \mL(\bw(t)), \bw(t)-\bw(t-1)\rangle
%     \\
% \nonumber
%     \le & \frac{N\sigmax^2}{b\gamma^2} \left\Vert -(1-\beta)\eta \nabla \mL(\bw(t))\right\Vert^2+ \left\Vert \beta(\bw(t)-\bw(t-1))\right\Vert^2+2(1-\beta)\eta\beta \langle - \nabla \mL(\bw(t)), \bw(t)-\bw(t-1)\rangle
%     \\
% \nonumber
%     \le & \frac{N\sigmax^2}{b\gamma^2} \left(\left\Vert -(1-\beta)\eta \nabla \mL(\bw(t))\right\Vert^2+ \left\Vert \beta(\bw(t)-\bw(t-1))\right\Vert^2+2(1-\beta)\eta\beta \langle - \nabla \mL(\bw(t)), \bw(t)-\bw(t-1)\rangle\right)
%     \\
\label{eq: estimation second moment update fake}
    \le&\frac{N\sigmax^2}{b\gamma^2}\left\Vert\mathbb{E} \left[\bw(t+1)-\bw(t)\left|\mF_t\right.\right]\right\Vert^2.
\end{align}
\end{small}

By Eqs. (\ref{eq: lyapunov_fake}) and (\ref{eq: estimation second moment update fake}), we have that to ensure Eq. (\ref{eq: requirement_lyapunov}), it is required that
\begin{align*}
    \frac{2-\beta}{2(1-\beta)\eta}\frac{b\gamma^2}{N\sigmax^2}-\frac{H\sigmax^2}{2N}\ge \frac{\beta}{2(1-\beta)\eta},
\end{align*}
which puts additional constraint on $\beta$ as 
\begin{equation*}
    \beta\le \frac{2b\gamma^2-H\eta\sigmax^4}{b\gamma^2+N\sigmax^2-H\sigmax^4\eta}.
\end{equation*}

Specifically, the upper bound becomes close to $0$ when $N$ becomes large, and constrains $\beta$ in a small range.

\section{Implicit regularization of deterministic Adam}
\label{subsec: implicit_bias_adam}
This section collects the proof of the convergent direction of Adam, i.e., Theorem \ref{thm: adam_main}. The methodology of this section bears great similarity with GDM, although the preconditioner of Adam requires specific treatment for analysis. The proof is still divided into two stages: (1). we first prove the sum of squared gradients along the trajectory is finite. Additionally, we prove the convergent rate of loss is $\mathcal{O}(\frac{1}{t})$; (2). we prove $\bw(t)-\ln(t)\hbw$ has bounded norm. Before we present these two stages of proof, we will first give the required range of $\eta$ for which Theorem \ref{thm: sgdm_main} holds. The analyses of this section hold for almost every dataset, and the "almost every" constraint is further moved in Section \ref{appen: every_data_set}.

\subsection{Choice of learning rate}
\label{subsec: choice_learning_rate}
Let $H_{s_0}$ be the smooth parameter over $[s_0,\infty)$ given by Assumption \ref{assum: smooth}. (D). Let $\beta_2=(c\beta_1)^4$ ($c>1$). The "sufficiently small learning rate" in Theorem \ref{thm: sgdm_main} means 
\begin{equation*}
    \eta\le \frac{\sqrt{\varepsilon}\inf_{t\ge 2}\left(\frac{1-\beta_1^t}{1-\beta_1}-\frac{1-\beta_1^{t-1}}{c(1-\beta_1)}\frac{1-(c\beta_1)^t}{1-(c\beta_1)^{t-1}}\right)}{H_{\ell^{-1}((1-c\beta_1)^{-1}N\mL(\bw(1)))}}.
\end{equation*}

To ensure $\eta$ is well-defined, we need to prove 
\begin{equation*}
    \inf_{t\ge 2}\left(\frac{1-\beta_1^t}{1-\beta_1}-\frac{1-\beta_1^{t-1}}{c(1-\beta_1)}\frac{1-(c\beta_1)^t}{1-(c\beta_1)^{t-1}}\right)>0,
\end{equation*}
and we introduce the following technical lemma:
\begin{lemma}
\label{lem: beta_1_2_inequality}
Define $f_t(x)=\frac{1-x^t}{x(1-x^{t-1})}$, $\forall t\in \mathbb{Z}, t\ge 2$. We have $f_t(x)$ is decreasing with respect to $x$.
Furthermore, for any $x\in [0,1)$, we have 
\begin{equation}
\label{eq: beta_1_2_inequality}
  f(x)\ge \sqrt[4]{ f(x^4)}.
\end{equation}
\end{lemma}
\begin{proof}
    First of all, by definition,
    \begin{equation*}
        f(x)=\frac{1-x^t}{x-x^{t}}=1+\frac{1-x}{x-x^{t}}=1+\frac{1-x}{x(1-x^{t-1})}=1+\frac{1}{x(1+x+\cdots+x^{t-2})}
    \end{equation*}
    is monotonously decreasing as $0\le x<1$. Secondly, Eq. (\ref{eq: beta_1_2_inequality}) is equivalent to 
    \begin{align*}
&\frac{(1-x^t)^4}{\beta^4_1(1-x^{t-1})^4}\ge  \frac{(1-x^{4t}
    )}{x^4(1-x^{4(t-1)})}
    \\
    \Longleftrightarrow&\frac{(1-x^t)^3}{(1-x^{t-1})^3}\ge  \frac{(1+x^{t}
    )(1+x^{2t}
    )}{(1+x^{t-1})(1+x^{2(t-1)})}.
\end{align*}

The left side of the above inequality is no smaller than 1, while the right side is no larger than 1, which completes the proof.
\end{proof}

We are now ready to prove $\eta$ is well-defined. First of all, for every $t$, we have 
\begin{align}
\nonumber
    &\frac{1-\beta_1^t}{1-\beta_1}-\frac{1-\beta_1^{t-1}}{c(1-\beta_1)}\frac{1-(c\beta_1)^t}{1-(c\beta_1)^{t-1}}
    \\
\nonumber
    =&\frac{\beta_1(1-\beta_1^{t-1})}{1-\beta_1}\left(\frac{1-\beta_1^t}{\beta_1 (1-\beta_1^{t-1})}-\frac{1-(c\beta_1)^t}{(c\beta_1) (1-(c\beta_1)^{t-1})}\right)
    \\
\label{eq:inf_larger_finite}
    \overset{(\star)}{=}&\frac{\beta_1(1-\beta_1^{t-1})}{1-\beta_1}\left(f_t(\beta_1)-f_t(c\beta_1)\right)>0,
\end{align}
where Eq. ($\star$) is by Lemma \ref{lem: beta_1_2_inequality} and $c\beta_1=\sqrt[4]{\beta_2}<1$.

On the other hand, we have 
\begin{equation}
\label{eq:inf_larger_limit}
    \lim_{t\rightarrow\infty}\left(\frac{1-\beta_1^t}{1-\beta_1}-\frac{1-\beta_1^{t-1}}{c(1-\beta_1)}\frac{1-(c\beta_1)^t}{1-(c\beta_1)^{t-1}}\right)=\left(1-\frac{1}{c}\right)\frac{1}{1-\beta_1}.
\end{equation}

By Eq. (\ref{eq:inf_larger_finite}) and Eq. (\ref{eq:inf_larger_limit}), we obtain $\frac{1-\beta_1^t}{1-\beta_1}-\frac{1-\beta_1^{t-1}}{c(1-\beta_1)}\frac{1-(c\beta_1)^t}{1-(c\beta_1)^{t-1}}$ is lower bounded by some positive constant across $t$, and $\eta$ is well defined.
\subsection{Sum of gradients along the trajectory is bounded}
\label{subsubsec: gradient_bounded_adam}

We start with the following lemma, which indicates $\mL(\bw(t))+\Vert \sqrt[4]{\varepsilon\mathds{1}_d+\hbnu(t)}\odot (\bw(t)-\bw(t-1))\Vert ^2$ is a proper Lyapunov function for Adam.
\begin{lemma}
\label{lem: loss_update_adam}
Let all conditions in Theorem \ref{thm: adam_main} hold. Then, for any $t\ge 1$,
\small
\begin{align}
\nonumber
    &\mL(t+1)
    +\frac{1}{2}\frac{1-\beta_1^t}{\eta(1-\beta_1)}\left\Vert \sqrt[4]{\varepsilon\mathds{1}_d+\hbnu(t)}\odot (\bw(t+1)-\bw(t))\right\Vert ^2
    \\
    \le &\mL(\bw(t))+\left\Vert \sqrt[4]{\varepsilon\mathds{1}_d+\hbnu(t-1)}\odot (\bw(t)-\bw(t-1))\right\Vert ^2
\label{eq: loss_update_adam_t}
      \frac{1-\beta_1^{t-1}}{2c\eta(1-\beta_1)}\frac{1-(c\beta_1)^t}{1-(c\beta_1)^{t-1}}.
\end{align}\normalsize
\end{lemma}

\begin{proof}
    We start with the case $t=1$. To begin with, we have $\mL$ is $H_{\ell^{-1}(N\mL(\bw(1)))}$ smooth around $\bw(1)$. By definition $H_{x}$ is non-increasing with respect to $x$, and since $\ell^{-1}$ is also non-increasing, we have
    \begin{equation*}
        H_{\ell^{-1}(N\mL(\bw(1)))}\le H_{\ell^{-1}(\frac{1}{1-c\beta_1}N\mL(\bw(1)))},
    \end{equation*}
    which further indicates when $\alpha$ is small enough, 
    \begin{align*}
        \mL(\bw(1+\alpha))\overset{(\star)}{\le}& \mL(\bw(1))+\alpha\langle  \nabla \mL(\bw(1)), \bw(2)-\bw(1)\rangle+\frac{L}{2}\alpha^2\Vert\bw(2)-\bw(1)\Vert
        \\
        =& \mL(\bw(1))-\alpha\left\langle \nabla \mL(\bw(1)), \eta\frac{1}{\sqrt{\varepsilon\mathds{1}_d+\hbnu(1)}}\odot\nabla \mL(\bw(1))\right\rangle+\bo(\alpha^2)
        \\
        \le& \mL(\bw(1))-\frac{1}{2\eta}\alpha^2\left\Vert \sqrt[4]{\varepsilon\mathds{1}_d+\hbnu(1)}\odot (\bw(2)-\bw(t))\right\Vert ^2,
    \end{align*}
    where in Eq. ($\star$) we denote $L\overset{\triangle}{=}H_{\ell^{-1}(\frac{1}{1-c\beta_1}N\mL(\bw(1)))}$, and the last inequality is due to $\frac{1}{2\eta}\alpha^2$ $\left\Vert \sqrt[4]{\varepsilon\mathds{1}_d+\hbnu(1)}\odot (\bw(2)-\bw(t))\right\Vert ^2=\bo(\alpha^2)$, and $\left\langle \nabla \mL(\bw(1)), \eta\frac{1}{\sqrt{\varepsilon\mathds{1}_d+\hbnu(1)}}\odot\nabla \mL(\bw(1))\right\rangle$ is positive.
    
    Now if there exists an $\alpha\in(0,1)$, such that Eq. (\ref{eq: loss_update_adam_t}) fails, we denote $\alpha^{*}=\inf\{\alpha:Eq. (\ref{eq: loss_update_adam_t})~fails~for~1+\alpha\}$. We have $\alpha^*>0$, and the equality in Eq. (\ref{eq: loss_update_adam_t}) holds for $1+\alpha^*$. Therefore, we have for any $\alpha \in (0,\alpha^*)$, 
    \begin{equation*}
        \mL(\bw(1+\alpha))\le \mL(\bw(1+\alpha))
    +\frac{1}{2\eta}\alpha^2\left\Vert \sqrt[4]{\varepsilon\mathds{1}_d+\hbnu(1)}\odot (\bw(2)-\bw(t))\right\Vert ^2
    \le  \mL(\bw(1)),
    \end{equation*}
    which by Lemma \ref{lem: smooth_guarantee} leads to $\mL$ is $H_{\ell^{-1}(N\mL(\bw(1)))}$ smooth (thus $ L $ smooth)   over the set $\{\bw(1+\alpha):\alpha\in [0, \alpha^*]\}$, and
    \begin{align}
        \nonumber
          &\mL(\bw(1+\alpha^*))
          \\
        \nonumber
          \le& \mL(\bw(1))+\alpha^*\langle  \nabla \mL(\bw(1)), \bw(2)-\bw(1)\rangle+\frac{L}{2}(\alpha^*)^2\Vert\bw(2)-\bw(1)\Vert^2
        \\
        \nonumber
        =& \mL(\bw(1))-\alpha^*\left\langle  \frac{1}{\eta}\sqrt{\varepsilon\mathds{1}_d+\hbnu(1)}\odot\left(\bw(2)-\bw(1)\right),\bw(2)-\bw(1)\right\rangle+\frac{L}{2}(\alpha^*)^2\Vert\bw(2)-\bw(1)\Vert^2
        \\
        \nonumber
       =& \mL(\bw(1))-\alpha^*\frac{1}{\eta}\left\Vert \sqrt[4]{\varepsilon\mathds{1}_d+\hbnu(1)}\odot\left(\bw(2)-\bw(1)\right)\right\Vert^2+\frac{L}{2}(\alpha^*)^2\left\Vert\frac{1}{\sqrt[4]{\varepsilon\mathds{1}_d+\hbnu(1)}}\odot\sqrt[4]{\varepsilon\mathds{1}_d+\hbnu(1)}\odot(\bw(2)-\bw(1))\right\Vert^2
       \\
       \nonumber
       \le & \mL(\bw(1))-\alpha^*\frac{1}{\eta}\left\Vert \sqrt[4]{\varepsilon\mathds{1}_d+\hbnu(1)}\odot\left(\bw(2)-\bw(1)\right)\right\Vert^2+\frac{L}{2\sqrt{\varepsilon}}(\alpha^*)^2\left\Vert\sqrt[4]{\varepsilon\mathds{1}_d+\hbnu(1)}\odot(\bw(2)-\bw(1))\right\Vert^2
       \\
       \nonumber
       < & \mL(\bw(1))-(\alpha^*)^2\frac{1}{\eta}\left\Vert \sqrt[4]{\varepsilon\mathds{1}_d+\hbnu(1)}\odot\left(\bw(2)-\bw(1)\right)\right\Vert^2+\frac{L}{2\sqrt{\varepsilon}}(\alpha^*)^2\left\Vert\sqrt[4]{\varepsilon\mathds{1}_d+\hbnu(1)}\odot(\bw(2)-\bw(1))\right\Vert^2
       \\
       \label{eq: contradictory_adam_1}
       \le &\mL(\bw(1))-(\alpha^*)^2\frac{1}{2\eta}\left\Vert \sqrt[4]{\varepsilon\mathds{1}_d+\hbnu(1)}\odot\left(\bw(2)-\bw(1)\right)\right\Vert^2,
    \end{align}
    where the second-to-last inequality is due to $\Vert \bw(2)-\bw(1)\Vert>0$ (by Lemma \ref{lem: non-zero}) and $\alpha^*>(\alpha^*)^2$, while the last inequality is due to 
    \begin{align*}
        \eta\le& \frac{\sqrt{\varepsilon}\inf_{t\ge 2}\left(\frac{1-\beta_1^t}{1-\beta_1}-\frac{1-\beta_1^{t-1}}{c(1-\beta_1)}\frac{1-(c\beta_1)^t}{1-(c\beta_1)^{t-1}}\right)}{L}
        \le \frac{\sqrt{\varepsilon}\left(\frac{1-\beta_1^2}{1-\beta_1}-\frac{1-\beta_1}{c(1-\beta_1)}\frac{1-(c\beta_1)^2}{1-(c\beta_1)}\right)}{L}
        \\
        =& \frac{\sqrt{\varepsilon}\left(1+\beta_1-\frac{1+(c\beta_1)}{c}\right)}{L}=\frac{\sqrt{\varepsilon}\left(1-\frac{1}{c}\right)}{L}<\frac{\sqrt{\varepsilon}}{L}.
    \end{align*}
    
    Eq. (\ref{eq: contradictory_adam_1}) contradicts the fact that the equality in Eq. (\ref{eq: loss_update_adam_t}) holds for $1+\alpha^*$, which completes the proof of $t=1$.
    
    If $t\ge 2$, following the similar routine as $t=1$, we also prove Eq. (\ref{eq: loss_update_adam_t}) by reduction to absurdity. If there exist $t$ and $\alpha$ such that Eq. (\ref{eq: loss_update_adam_t}) fails. Denote $t^*$ as the smallest time such that there exists an $\alpha\in[0,1)$ such that Eq. (\ref{eq: loss_update_adam_t}) fails for $t^*$ and $\alpha$.  By Lemma \ref{lem: non-zero}, $\left\Vert \sqrt[4]{\varepsilon\mathds{1}_d+\hbnu(t^*-1)}\odot (\bw(t^*)-\bw(t^*-1))\right\Vert ^2$ is positive, and strict inequality in Eq. (\ref{eq: loss_update_adam_t}) holds for $t$ and $\alpha=0$, which by continuity leads to 
    \begin{equation*}
       1> \alpha^{*}\overset{\triangle}{=}\inf\{\alpha\in[0,1]:Eq. (\ref{eq: loss_update_adam_t})~fails~for~1+\alpha\}>0.
    \end{equation*}
    
     Then, for any $\alpha\in[0,\alpha^*]$, we have
    \begin{align}
    \nonumber
        &\mL(\bw(t^*+\alpha))
    \\
    \nonumber
        \le&\mL(\bw(t^*+\alpha))
    +\frac{1}{2}\alpha^2\frac{1-\beta_1^{t^*}}{\eta(1-\beta_1)}\left\Vert \sqrt[4]{\varepsilon\mathds{1}_d+\hbnu(t^*)}\odot (\bw(t^*+1)-\bw(t^*))\right\Vert ^2
\\
    \nonumber
    \le & \mL(\bw(t^*))+\frac{1-\beta_1^{t^*-1}}{2c\eta(1-\beta_1)}\frac{1-(c\beta_1)^{t^*}}{1-(c\beta_1)^{t^*-1}}\left\Vert \sqrt[4]{\varepsilon\mathds{1}_d+\hbnu(t^*-1)}\odot (\bw(t^*)-\bw(t^*-1))\right\Vert ^2.
    \end{align}
    
    On the other hand, for any time $2\le s\le t^*-1$, we have
    \begin{align}
\nonumber
    &\mL(\bw(s+1))
    +\frac{1}{2}\frac{1-\beta_1^s}{\eta(1-\beta_1)}\left\Vert \sqrt[4]{\varepsilon\mathds{1}_d+\hbnu(s)}\odot (\bw(s+1)-\bw(s))\right\Vert ^2
    \\
\label{eq: adam_update_s}
    \le & \mL(\bw(s))+\frac{\beta_1(1-\beta_1^{s-1})}{2c\eta(1-\beta_1)}\frac{1-(c\beta_1)^s}{1-(c\beta_1)^{s-1}}\left\Vert \sqrt[4]{\varepsilon\mathds{1}_d+\hbnu(s-1)}\odot (\bw(s)-\bw(s-1))\right\Vert ^2.
\end{align}

By Eq. (\ref{eq:inf_larger_limit}), we have \begin{align*}
    \frac{1-\beta_1^s}{\eta(1-\beta_1)}
    > \frac{1-\beta_1^s}{c\eta(1-\beta_1)}
    = \frac{1-\beta_1^s}{c\eta(1-\beta_1)}\frac{1-(c\beta_1)^{s+1}}{1-(c\beta_1)^{s}}\frac{1-(c\beta_1)^{s}}{1-(c\beta_1)^{s+1}},
\end{align*}
which by $\frac{(1-\beta_1^{s-1})}{(1-\beta_1^{s})}$ further leads to
\begin{align}
\nonumber
    &\mL(\bw(s))+\frac{1-\beta_1^{s-1}}{2c\eta(1-\beta_1)}\frac{1-(c\beta_1)^s}{1-(c\beta_1)^{s-1}}\left\Vert \sqrt[4]{\varepsilon\mathds{1}_d+\hbnu(s-1)}\odot (\bw(s)-\bw(s-1))\right\Vert ^2
    \\
\nonumber
    \ge &\mL(\bw(s+1))
    +\frac{1}{2}\frac{1-\beta_1^s}{\eta(1-\beta_1)}\left\Vert \sqrt[4]{\varepsilon\mathds{1}_d+\hbnu(s)}\odot (\bw(s+1)-\bw(s))\right\Vert ^2
    \\
\nonumber
    > &\mL(\bw(s+1))
    +\frac{1-\beta_1^s}{2c\eta(1-\beta_1)}\frac{1-(c\beta_1)^{s+1}}{1-(c\beta_1)^{s}}\frac{1-(c\beta_1)^{s}}{1-(c\beta_1)^{s+1}}\left\Vert \sqrt[4]{\varepsilon\mathds{1}_d+\hbnu(s)}\odot (\bw(s+1)-\bw(s))\right\Vert ^2
    \\
    \label{eq: adam_loss_iter_s}
    >&\frac{1-(c\beta_1)^{s}}{1-(c\beta_1)^{s+1}}\left(\mL(\bw(s+1))
    +\frac{1-\beta_1^s}{2c\eta(1-\beta_1)}\frac{1-(c\beta_1)^{s+1}}{1-(c\beta_1)^{s}}\left\Vert \sqrt[4]{\varepsilon\mathds{1}_d+\hbnu(s)}\odot (\bw(s+1)-\bw(s))\right\Vert ^2\right).
\end{align}

On the other hand, for $s=1$, we have
\begin{align}
\nonumber
\mL(\bw(1))\ge&
    \mL(\bw(2))
    +\frac{1}{2}\frac{1-\beta_1}{\eta(1-\beta_1)}\left\Vert \sqrt[4]{\varepsilon\mathds{1}_d+\hbnu(1)}\odot (\bw(2)-\bw(1))\right\Vert ^2
    \\
\label{eq: adam_loss_iter_1}
    \ge &\frac{1-(c\beta_1)}{1-(c\beta_1)^{2}}\left(\mL(\bw(2))
    +\frac{1-\beta_1}{2c\eta(1-\beta_1)}\frac{1-(c\beta_1)^{2}}{1-(c\beta_1)}\left\Vert \sqrt[4]{\varepsilon\mathds{1}_d+\hbnu(1)}\odot (\bw(2)-\bw(1))\right\Vert ^2\right).
\end{align}

Combining Eqs. (\ref{eq: adam_update_s}), (\ref{eq: adam_loss_iter_s}), and  (\ref{eq: adam_loss_iter_1}), we have 
\begin{align*}
    &\mL(\bw(t^*+\alpha))
    \\
    \le & \mL(\bw(t^*))+\frac{1-\beta_1^{t^*-1}}{2c\eta(1-\beta_1)}\frac{1-(c\beta_1)^{t^*}}{1-(c\beta_1)^{t^*-1}}\left\Vert \sqrt[4]{\varepsilon\mathds{1}_d+\hbnu(t^*-1)}\odot (\bw(t^*)-\bw(t^*-1))\right\Vert ^2
    \\
    < &\frac{1-(c\beta_1)^{t^*}}{1-(c\beta_2)^{t^*-1}}\left(\mL(\bw(t^*-1))+\frac{1-\beta_1^{t^*-2}}{2c\eta(1-\beta_1)}\frac{1-(c\beta_1)^{t^*-1}}{1-(c\beta_1)^{t^*-2}}\left\Vert \sqrt[4]{\varepsilon\mathds{1}_d+\hbnu(t^*-2)}\odot (\bw(t^*-1)-\bw(t^*-2))\right\Vert ^2\right)
    \\
    <& \cdots
    \\
    <&\frac{1-(c\beta_1)^{t^*}}{1-(c\beta_1)^2}\left(\mL(\bw(2))+\frac{1-\beta_1}{2c\eta(1-\beta_1)}\frac{1-(c\beta_1)^{2}}{1-(c\beta_1)}\left\Vert \sqrt[4]{\varepsilon\mathds{1}_d+\hbnu(1)}\odot (\bw(2)-\bw(1))\right\Vert ^2\right)
    \\
    \le &\frac{1-(c\beta_1)^{t^*}}{1-c\beta_1} \mL(\bw(1))<\frac{1}{1-c\beta_1} \mL(\bw(1)).
\end{align*}

Therefore, by Lemma \ref{lem: smooth_guarantee}, $\mL$ is $H_{\ell^{-1}(\frac{1}{1-c\beta_1}N\mL(\bw(1)))}$ smooth (thus $L$ smooth)  over the set $\{\bw(t^*+\alpha): \alpha\in [0,\alpha^*]\}$, which further leads to
\begin{align*}
    &\mL(\bw(t^*+\alpha^*))
    \\
    \le & \mL(\bw(t^*))+\alpha^*\langle  \nabla \mL(\bw(t^*)), \bw(t^*+1)-\bw(t^*)\rangle+\frac{L}{2}(\alpha^*)^2\Vert\bw(t^*+1)-\bw(t^*)\Vert^2
    \\
    \overset{(\bullet)}{=}& -\frac{\alpha^*}{\eta(1-\beta_1)}\left\langle  \bw(t^*+1)-\bw(t^*), (1-\beta_1^{t^*})\sqrt{\varepsilon\mathds{1}_d+\hbnu(t^*)}\odot( \bw(t^*+1)-\bw(t^*))\right.
    \\
    -&\left.\beta_1(1-\beta_1^{t^*-1})\sqrt{\varepsilon\mathds{1}_d+\hbnu(t^*)}\odot( \bw(t^*)-\bw(t^*-1))\right\rangle
    \\
    +&\mL(\bw(t^*))+\frac{L}{2}(\alpha^*)^2\Vert\bw(t^*+1)-\bw(t^*)\Vert^2
    \\
    =&\mL(\bw(t^*))+\frac{L}{2}(\alpha^*)^2\Vert\bw(t^*+1)-\bw(t^*)\Vert^2-\frac{\alpha^*(1-\beta_1^{t^*})}{\eta(1-\beta_1)}\left \Vert \sqrt[4]{\varepsilon\mathds{1}_d+\hbnu(t^*)}\odot( \bw(t^*+1)-\bw(t^*))\right\Vert^2
    \\
    +&\beta_1\frac{\alpha^*(1-\beta_1^{t^*-1})}{\eta(1-\beta_1)}\left\langle  \bw(t^*+1)-\bw(t^*), \sqrt{\varepsilon\mathds{1}_d+\hbnu(t^*)}\odot( \bw(t^*)-\bw(t^*-1))\right\rangle
    \\
    = &\mL(\bw(t^*))+\frac{L}{2}(\alpha^*)^2\Vert\bw(t^*+1)-\bw(t^*)\Vert^2-\frac{\alpha^*(1-\beta_1^{t^*})}{\eta(1-\beta_1)}\left \Vert \sqrt[4]{\varepsilon\mathds{1}_d+\hbnu(t^*)}\odot( \bw(t^*+1)-\bw(t^*))\right\Vert^2
    \\
    +&\beta_1\frac{\alpha^*(1-\beta_1^{t^*-1})}{\eta(1-\beta_1)}\left\langle \frac{\sqrt[8]{\varepsilon\mathds{1}_d+\hbnu(t^*-1)}}{\sqrt[8]{\varepsilon\mathds{1}_d+\hbnu(t^*)}}\odot \sqrt[4]{\varepsilon\mathds{1}_d+\hbnu(t^*)}\odot(\bw(t^*+1)-\bw(t^*)),\right.
    \\
    &~~~~~~~~~~~~~~~~~~~~~~~~~~~~~~~~~~~~~~~~~~~~~~~~~~~~~~~~~~~\left.\frac{\sqrt[8]{\varepsilon\mathds{1}_d+\hbnu(t^*-1)}}{\sqrt[8]{\varepsilon\mathds{1}_d+\hbnu(t^*)}}\odot\sqrt[4]{\varepsilon\mathds{1}_d+\hbnu(t^*-1)}\odot( \bw(t^*)-\bw(t^*-1))\right\rangle
    \\
    \le & \mL(\bw(t^*))+\frac{L}{2}(\alpha^*)^2\Vert\bw(t^*+1)-\bw(t^*)\Vert^2-\frac{\alpha^*(1-\beta_1^{t^*})}{\eta(1-\beta_1)}\left \Vert \sqrt[4]{\varepsilon\mathds{1}_d+\hbnu(t^*)}\odot( \bw(t^*+1)-\bw(t^*))\right\Vert^2
    \\
    +&\beta_1\frac{(\alpha^*)^2(1-\beta_1^{t^*-1})}{2\eta(1-\beta_1)}\left\Vert \frac{\sqrt[8]{\varepsilon\mathds{1}_d+\hbnu(t^*-1)}}{\sqrt[8]{\varepsilon\mathds{1}_d+\hbnu(t^*)}}\odot \sqrt[4]{\varepsilon\mathds{1}_d+\hbnu(t^*)}\odot(\bw(t^*+1)-\bw(t^*))\right\Vert^2
    \\
    +&\beta_1\frac{(1-\beta_1^{t^*-1})}{2\eta(1-\beta_1)}\left\Vert\frac{\sqrt[8]{\varepsilon\mathds{1}_d+\hbnu(t^*-1)}}{\sqrt[8]{\varepsilon\mathds{1}_d+\hbnu(t^*)}}\odot\sqrt[4]{\varepsilon\mathds{1}_d+\hbnu(t^*-1)}\odot( \bw(t^*)-\bw(t^*-1))\right\Vert^2
\\
   \overset{(\diamond)}{\le} &\mL(\bw(t^*))+\frac{L}{2}(\alpha^*)^2\Vert\bw(t^*+1)-\bw(t^*)\Vert^2-\frac{\alpha^*(1-\beta_1^{t^*})}{\eta(1-\beta_1)}\left \Vert \sqrt[4]{\varepsilon\mathds{1}_d+\hbnu(t^*)}\odot( \bw(t^*+1)-\bw(t^*))\right\Vert^2
    \\
    +&\beta_1\frac{(\alpha^*)^2(1-\beta_1^{t^*-1})}{2\eta(1-\beta_1)}\frac{1-(c\beta_1)^{t^*}}{c\beta_1(1-(c\beta_1)^{t^*-1})}\left\Vert  \sqrt[4]{\varepsilon\mathds{1}_d+\hbnu(t^*)}\odot(\bw(t^*+1)-\bw(t^*))\right\Vert^2
    \\
    +&\beta_1\frac{(1-\beta_1^{t^*-1})}{2\eta(1-\beta_1)}\frac{1-(c\beta_1)^{t^*}}{c\beta_1(1-(c\beta_1)^{t^*-1})}\left\Vert\sqrt[4]{\varepsilon\mathds{1}_d+\hbnu(t^*-1)}\odot( \bw(t^*)-\bw(t^*-1))\right\Vert^2
    \\
    \le & \mL(\bw(t^*))+\frac{L}{2\sqrt{\varepsilon}}(\alpha^*)^2\Vert\sqrt[4]{\varepsilon\mathds{1}_d+\hbnu(t^*)}\odot(\bw(t^*+1)-\bw(t^*))\Vert^2
    \\
    -&\frac{\alpha^*(1-\beta_1^{t^*})}{\eta(1-\beta_1)}\left \Vert \sqrt[4]{\varepsilon\mathds{1}_d+\hbnu(t^*)}\odot( \bw(t^*+1)-\bw(t^*))\right\Vert^2
    \\
    +&\beta_1\frac{(\alpha^*)^2(1-\beta_1^{t^*})}{2\eta(1-\beta_1)}\frac{1-(c\beta_1)^{t^*}}{c\beta_1(1-(c\beta_1)^{t^*-1})}\left\Vert  \sqrt[4]{\varepsilon\mathds{1}_d+\hbnu(t^*)}\odot(\bw(t^*+1)-\bw(t^*))\right\Vert^2
    \\
    +&\beta_1\frac{(1-\beta_1^{t^*})}{2\eta(1-\beta_1)}\frac{1-(c\beta_1)^{t^*}}{c\beta_1(1-(c\beta_1)^{t^*-1})}\left\Vert\sqrt[4]{\varepsilon\mathds{1}_d+\hbnu(t^*-1)}\odot( \bw(t^*)-\bw(t^*-1))\right\Vert^2
    \\
    \overset{(\square)}{<}&\mL(\bw(t^*))-\frac{(\alpha^*)^2(1-\beta_1^{t^*})}{2\eta(1-\beta_1)}\left \Vert \sqrt[4]{\varepsilon\mathds{1}_d+\hbnu(t^*)}\odot( \bw(t^*+1)-\bw(t^*))\right\Vert^2
        \end{align*}
    \begin{align*}
    +&\frac{(1-\beta_1^{t^*-1})}{2\eta(1-\beta_1)}\frac{1-(c\beta_1)^{t^*}}{c(1-(c\beta_1)^{t^*-1})}\left\Vert\sqrt[4]{\varepsilon\mathds{1}_d+\hbnu(t^*-1)}\odot( \bw(t^*)-\bw(t^*-1))\right\Vert^2,
\end{align*}
where Eq. $(\bullet)$ is due to an alternative form of the Adam's update rule:
\begin{align}
\nonumber
   &(1-\beta_1^{t^*}) \sqrt{\varepsilon\mathds{1}_d+\hbnu(t^*)}\odot( \bw(t^*+1)-\bw(t^*))-\beta_1(1-\beta_1^{t^*-1})\sqrt{\varepsilon\mathds{1}_d+\hbnu(t^*-1)}\odot( \bw(t^*)-\bw(t^*-1))
   \\
   \label{eq: alter_adam}
   =&-\eta (1-\beta_1) \nabla \mL(\bw(t^*)),
\end{align}
Inequality $(\diamond)$ is due to 
\begin{align*}
    &\frac{\sqrt[4]{\varepsilon\mathds{1}_d+\hbnu(t^*-1)}}{\sqrt[4]{\varepsilon\mathds{1}_d+\hbnu(t^*)}}
    =\sqrt[4]{\frac{\varepsilon\mathds{1}_d+\hbnu(t^*-1)}{\varepsilon\mathds{1}_d+\hbnu(t^*)}}
    =\sqrt[4]{\frac{\varepsilon\mathds{1}_d+\hbnu(t^*-1)}{\varepsilon\mathds{1}_d+\frac{\beta_2\bnu(t^*-1)+(1-\beta_2)\nabla \mL(\bw(t^*))^2}{1-\beta_2^{t^*}}}}
    \\
    \le& \sqrt[4]{\frac{\varepsilon\mathds{1}_d+\hbnu(t^*-1)}{\varepsilon\mathds{1}_d+\frac{\beta_2\bnu(t^*-1)}{1-\beta_2^{t^*}}}}
    =\sqrt[4]{\frac{\varepsilon\mathds{1}_d+\hbnu(t^*-1)}{\varepsilon\mathds{1}_d+\frac{\beta_2(1-\beta_2^{t^*-1})\hbnu(t^*-1)}{1-\beta_2^{t^*}}}}
    \le \sqrt[4]{\frac{\varepsilon\mathds{1}_d+\hbnu(t^*-1)}{\frac{\beta_2(1-\beta_2^{t^*-1})\hbnu(t^*-1)}{1-\beta_2^{t^*}}\varepsilon\mathds{1}_d+\frac{\beta_2(1-\beta_2^{t^*-1})\hbnu(t^*-1)}{1-\beta_2^{t^*}}}}
    \\
    =& \sqrt[4]{\frac{1-\beta_2^{t^*}}{\beta_2(1-\beta_2^{t^*-1})}}\mathds{1}_{d} ~(all~the~computings~are ~component\text{-}wisely),
\end{align*}
and $f(c\beta_1)\ge \sqrt[4]{f((c\beta_1)^4)}$, and Inequality $(\square)$ is due to
\begin{equation*}
      \frac{L}{2\sqrt{\varepsilon}}\le \frac{\inf_{t\ge 2}\left(\frac{1-\beta_1^t}{1-\beta_1}-\frac{1-\beta_1^{t-1}}{c(1-\beta_1)}\frac{1-(c\beta_1)^t}{1-(c\beta_1)^{t-1}}\right)}{2\eta}\le \frac{\left(\frac{1-\beta_1^{t^*}}{1-\beta_1}-\frac{1-\beta_1^{t^*-1}}{c(1-\beta_1)}\frac{1-(c\beta_1)^{t^*}}{1-(c\beta_1)^{t^*-1}}\right)}{2\eta},
\end{equation*}
$\alpha^*>(\alpha^*)^2$, and $\left \Vert \sqrt[4]{\varepsilon\mathds{1}_d+\hbnu(t^*)}\odot( \bw(t^*+1)-\bw(t^*))\right\Vert^2>0$.

This contradicts to that the equality in Eq. (\ref{eq: loss_update_adam_t}) holds for $t^*+\alpha^*$.

The proof is completed.
\end{proof}

As $\lim_{t\rightarrow\infty} \beta_1^t=0$ and $\lim_{t\rightarrow\infty} (c\beta_1)^t=0$, we have the following corollary based on Lemma \ref{lem: update_rule_loss_gdm}.

\begin{corollary}
\label{coro: adam_loss_iteration}
Let all assumptions in Theorem \ref{thm: adam_main} hold. Then, for large enough $t$, we have 
\begin{align}
\nonumber
    &\mL(\bw(t+1))
    +\frac{1}{2\sqrt[4]{c}\eta(1-\beta_1)}\left\Vert \sqrt[4]{\varepsilon\mathds{1}_d+\hbnu(t)}\odot (\bw(t+1)-\bw(t))\right\Vert ^2
    \\
\label{eq: coro_loss_update_adam_t}
    \le & \mL(\bw(t))+\frac{1}{2\sqrt[2]{c}\eta(1-\beta_1)}\left\Vert \sqrt[4]{\varepsilon\mathds{1}_d+\hbnu(t-1)}\odot (\bw(t)-\bw(t-1))\right\Vert ^2.
\end{align}
Consequently, we have
\begin{equation}
\label{eq: adam_gradient_finite}
    \sum_{t=1}^{\infty} \Vert \nabla \mL(\bw(t)) \Vert^2<\infty.
\end{equation}
\end{corollary}
The proof of Corollary \ref{coro: adam_loss_iteration} relies on the following classical lemma on the equivalence between the convergence of two non-negative sequence. The proof is omitted here and can be found in \cite{wang2021implicit}.

\begin{lemma}[c.f. Lemma 27, \cite{wang2021implicit}]
\label{lem: series_convergence}
Let $\{a_i\}_{i=1}^{\infty}$ be a series of non-negative reals, and $\varepsilon$ be a positive real. Then, $\sum_{i=1}^{\infty} a_i<\infty$ is equivalent to   $\sum_{i=1}^{\infty} \frac{a_i}{\sqrt{\varepsilon+\sum_{s=1}^i a_s} }<\infty$.
\end{lemma}

\begin{proof}[Proof of Corollary \ref{coro: adam_loss_iteration}]
    We have 
    \begin{align*}
       \lim_{t\rightarrow\infty} \frac{1-\beta_1^{t-1}}{2c\eta(1-\beta_1)}\frac{1-(c\beta_1)^t}{1-(c\beta_1)^{t-1}}=\frac{1}{2c\eta(1-\beta_1)}<\frac{1}{2\sqrt[2]{c}\eta(1-\beta_1)},
       \\
       \lim_{t\rightarrow\infty}\frac{1-\beta_1^t}{2\eta(1-\beta_1)}=\frac{1}{2\eta(1-\beta_1)}>\frac{1}{2\sqrt[4]{c}\eta(1-\beta_1)},
    \end{align*}
    which completes the proof of Eq. (\ref{eq: coro_loss_update_adam_t}). Rearranging Eq. (\ref{eq: coro_loss_update_adam_t}) leads to
    \begin{align*}
       &\frac{\sqrt[4]{c}-1}{2\sqrt[2]{c}\eta(1-\beta_1)}\left\Vert \sqrt[4]{\varepsilon\mathds{1}_d+\hbnu(t)}\odot (\bw(t+1)-\bw(t))\right\Vert ^2\le  \frac{1}{2\sqrt[4]{c}\eta(1-\beta_1)}\left\Vert \sqrt[4]{\varepsilon\mathds{1}_d+\hbnu(t-1)}\odot (\bw(t)-\bw(t-1))\right\Vert ^2
    \\
     +&\mL(\bw(t)) -\left(\mL(\bw(t+1))
    +\frac{1}{2\sqrt[4]{c}\eta(1-\beta_1)}\left\Vert \sqrt[4]{\varepsilon\mathds{1}_d+\hbnu(t)}\odot (\bw(t+1)-\bw(t))\right\Vert ^2\right),
    \end{align*}
    which by iteration further leads to that for a large enough time $T_1$
    \begin{align*}
        &\sum_{t=T_1}^{T_2}\frac{\sqrt[4]{c}-1}{2\sqrt[2]{c}\eta(1-\beta_1)}\left\Vert \sqrt[4]{\varepsilon\mathds{1}_d+\hbnu(t)}\odot (\bw(t+1)-\bw(t))\right\Vert ^2
        \\
        \le &\mL(\bw(T_1))+\frac{1}{2\sqrt[4]{c}\eta(1-\beta_1)}\left\Vert \sqrt[4]{\varepsilon\mathds{1}_d+\hbnu(T_1-1)}\odot (\bw(T_1)-\bw(T_1-1))\right\Vert ^2
        \\
        -&\mL(\bw(T_2+1))+\frac{1}{2\sqrt[4]{c}\eta(1-\beta_1)}\left\Vert \sqrt[4]{\varepsilon\mathds{1}_d+\hbnu(T_2)}\odot (\bw(T_2+1)-\bw(T_2+1))\right\Vert ^2
        \\
        <&\mL(\bw(T_1))+\frac{1}{2\sqrt[4]{c}\eta(1-\beta_1)}\left\Vert \sqrt[4]{\varepsilon\mathds{1}_d+\hbnu(T_1-1)}\odot (\bw(T_1)-\bw(T_1-1))\right\Vert ^2.
    \end{align*}
    Consequently, we obtain 
    \begin{align}
    \label{eq: adam_mid_result_1}
        \sum_{t=1}^{\infty}\frac{\sqrt[4]{c}-1}{2\sqrt[2]{c}\eta(1-\beta_1)}\left\Vert \sqrt[4]{\varepsilon\mathds{1}_d+\hbnu(t)}\odot (\bw(t+1)-\bw(t))\right\Vert ^2<\infty.
    \end{align}
    
    On the other hand, for any $t$, we have
    \begin{align*}
        &\left\Vert \sqrt[4]{\varepsilon\mathds{1}_d+\hbnu(t)}\odot (\bw(t+1)-\bw(t))\right\Vert\left\Vert {\sqrt[4]{\varepsilon\mathds{1}_d+\hbnu(t)}}\odot \hbw\right\Vert
        \\
        \ge &\left\langle \sqrt[4]{\varepsilon\mathds{1}_d+\hbnu(t)}\odot (\bw(t+1)-\bw(t)), \sqrt[4]{\varepsilon\mathds{1}_d+\hbnu(t)}\odot \hbw\right\rangle
        =\left\langle \sqrt[2]{\varepsilon\mathds{1}_d+\hbnu(t)}\odot (\bw(t+1)-\bw(t)), \hbw\right\rangle
        \\
        =&\left\langle -\eta\hbom(t), \hbw\right\rangle=-\frac{\eta(1-\beta_1)}{1-\beta_1^t}\left\langle\sum_{s=1}^t \beta_1^{t-s}\nabla \mL(\bw(s)), \hbw\right\rangle
        \\
        =&-\frac{\eta(1-\beta_1)}{1-\beta_1^t}\frac{1}{N}\left\langle\sum_{s=1}^t \beta_1^{t-s}\sum_{\bx_i\in \bS}\ell'(\langle \bx_i,\bw(s)\rangle) \bx_i, \hbw\right\rangle\ge-\frac{\eta(1-\beta_1)}{1-\beta_1^t}\frac{1}{N}\sum_{s=1}^t \beta_1^{t-s}\sum_{\bx_i\in \bS}\ell'(\langle \bx_i,\bw(s)\rangle) 
        \\
        \ge&-\frac{\eta(1-\beta_1)}{1-\beta_1^t}\frac{1}{N} \sum_{\bx_i\in \bS}\ell'(\langle \bx_i,\bw(t)\rangle) \ge \frac{\eta(1-\beta_1)}{1-\beta_1^t} \Vert \nabla\mL(\bw(t))\Vert,
    \end{align*}
    which by Eq. (\ref{eq: adam_mid_result_1}) indicates
    \begin{align*}
         \sum_{t=1}^{\infty} \left(\frac{\eta(1-\beta_1)}{1-\beta_1^t} \right)^2\frac{\Vert \nabla\mL(\bw(t))\Vert^2}{\left\Vert {\sqrt[4]{\varepsilon\mathds{1}_d+\hbnu(t)}}\odot \hbw\right\Vert^2}<\infty.
    \end{align*}
    
    As $\lim_{t\rightarrow\infty}\left(\frac{\eta(1-\beta_1)}{1-\beta_1^t} \right)^2=\eta^2(1-\beta_1)^2$, we then obtain
     \begin{align*}
     &\sum_{t=1}^{\infty} \frac{\Vert \nabla\mL(\bw(t))\Vert^2}{ {\sqrt[2]{\varepsilon+\sum_{s=1}^t\Vert \nabla\mL(\bw(t))\Vert^2}}}
     \le\sum_{t=1}^{\infty} \frac{\Vert \nabla\mL(\bw(t))\Vert^2}{ {\sqrt[2]{\varepsilon+\sum_{s=1}^t(1-\beta)\beta^{t-s}\Vert \nabla\mL(\bw(t))\Vert^2}}}
     \\
     \le&\sqrt{\frac{1}{1-\beta}}\sum_{t=1}^{\infty} \frac{\Vert \nabla\mL(\bw(t))\Vert^2}{ {\sqrt[2]{\varepsilon+\frac{\sum_{s=1}^t(1-\beta)\beta^{t-s}\Vert \nabla\mL(\bw(t))\Vert^2}{1-\beta^t}}}}\le d\sqrt{\frac{1}{1-\beta}}\sum_{t=1}^{\infty} \frac{\Vert \nabla\mL(\bw(t))\Vert^2}{ \left\Vert\sqrt[4]{\varepsilon\mathds{1}_d+\frac{\sum_{s=1}^t(1-\beta)\beta^{t-s} \nabla\mL(\bw(t))^2}{1-\beta^t}}\right\Vert^2}
         \\
         = &d\sqrt{\frac{1}{1-\beta}}\sum_{t=1}^{\infty} \frac{\Vert \nabla\mL(\bw(t))\Vert^2}{\left\Vert {\sqrt[4]{\varepsilon\mathds{1}_d+\hbnu(t)}}\right\Vert^2}\le d\Vert \hbw\Vert^2_{\infty}\sqrt{\frac{1}{1-\beta}} \sum_{t=1}^{\infty} \frac{\Vert \nabla\mL(\bw(t))\Vert^2}{\left\Vert {\sqrt[4]{\varepsilon\mathds{1}_d+\hbnu(t)}}\odot \hbw\right\Vert^2}<\infty,
    \end{align*}
    which by Lemma \ref{lem: series_convergence} completes the proof.
\end{proof}

Based on Corollary \ref{coro: adam_loss_iteration}, we can further prove Lemma \ref{lem: adam_asymptotic}, characterizing the convergent rate of loss $\mL$ directly.
\begin{lemma}
\label{lem: adam_asymptotic}
Let all conditions in Theorem \ref{thm: adam_main} hold. Then,
    $\mL(\bw(t))=\Theta\left(t^{-1}\right), \Vert \bw(t) \Vert=\Theta (\ln (t))$,
    and $\Vert \bw(t)-\bw(t-1)\Vert =\Theta(t^{-1})$.
\end{lemma}
\begin{proof}[Proof of Lemma \ref{lem: adam_asymptotic}]
    To begin with, Eq. (\ref{eq: alter_adam}) indicates 
    \begin{small}
    \begin{align}
    \nonumber
        &\Vert\eta (1-\beta_1) \nabla \mL(\bw(t))\Vert^2
   \\
   \nonumber
   =&\left\Vert(1-\beta_1^{t}) \sqrt{\varepsilon\mathds{1}_d+\hbnu(t)}\odot( \bw(t+1)-\bw(t))-\beta_1(1-\beta_1^{t-1})\sqrt{\varepsilon\mathds{1}_d+\hbnu(t-1)}\odot( \bw(t)-\bw(t-1))\right\Vert^2
   \\
   \nonumber
   \le &\left\Vert(1-\beta_1^{t}) \sqrt{\varepsilon\mathds{1}_d+\hbnu(t)}\odot( \bw(t+1)-\bw(t))\right\Vert+\left\Vert \beta_1(1-\beta_1^{t-1})\sqrt{\varepsilon\mathds{1}_d+\hbnu(t-1)}\odot( \bw(t)-\bw(t-1))\right\Vert^2
   \\
   \nonumber
   \le &\left(\left\Vert \sqrt{\varepsilon\mathds{1}_d+\hbnu(t)}\odot( \bw(t+1)-\bw(t))\right\Vert+\left\Vert \sqrt{\varepsilon\mathds{1}_d+\hbnu(t-1)}\odot( \bw(t)-\bw(t-1))\right\Vert\right)^2
   \\
   \label{eq: gradient_update}
   \le &2\left(\left\Vert \sqrt{\varepsilon\mathds{1}_d+\hbnu(t)}\odot( \bw(t+1)-\bw(t))\right\Vert^2+\left\Vert \sqrt{\varepsilon\mathds{1}_d+\hbnu(t-1)}\odot( \bw(t)-\bw(t-1))\right\Vert^2\right)
    \end{align}
    \end{small}
    On the other hand, by Corollary \ref{coro: adam_loss_iteration}, 
    \begin{equation*}
        \sum_{s=1}^{\infty} \Vert \nabla \mL(\bw(s)) \Vert^2<\infty,
    \end{equation*}
    which following the same routine as Corollary \ref{coro: sgdm_gradient_finite} leads to 
    \begin{equation*}
       \langle\bw(t),\bx \rangle \rightarrow \infty, \forall \bx\in \tbS.
    \end{equation*}
    Therefore, by Lemma \ref{lem: equivalence_loss_derivative}, there exists a large enough time $T_1$, such that $\forall t\ge T_1$,
    \begin{align*}
        \frac{1}{K} \ell(\langle \bw(t),\bx\rangle)\le -\ell'(\langle \bw(t),\bx\rangle)\le  K\ell(\langle \bw(t),\bx\rangle), \forall \bx\in \bS,
    \end{align*}
    which by the separable assumption further leads to 
    \begin{align}
    \nonumber
        &\frac{\gamma}{K}\mL(\bw(t))\le -\frac{\gamma }{N}\sum_{\bx\in \bS}\ell'(\langle \bw(t),\bx\rangle)
       \le  \frac{1 }{N}\left\langle-\sum_{\bx\in \bS}\ell'(\langle \bw(t),\bx\rangle) \bx, \gamma \hbw\right\rangle 
       \\
     \nonumber
       \le & \frac{1 }{N}\left\Vert\sum_{\bx\in \bS}\ell'(\langle \bw(t),\bx\rangle) \bx\right\Vert \Vert \gamma \hbw\Vert=\left\Vert\nabla \mL(\bw(t))\right\Vert
       \\
      \label{eq: equivalence_loss_gradient}
       \le& - \frac{1 }{N}\sum_{\bx\in \bS}\ell'(\langle \bw(t),\bx\rangle) \le K\mL(\bw(t)).
    \end{align}
    
    Combining Eq. (\ref{eq: update_to_gradient}) and the above inequality, we have
    \begin{align}
    \nonumber
       & \left(\frac{\eta(1-\beta_1)\gamma}{K}\right)^2 \mL (\bw(t))^2\le2\left(\left\Vert \sqrt{\varepsilon\mathds{1}_d+\hbnu(t)}\odot( \bw(t+1)-\bw(t))\right\Vert^2\right.
       \\
       \label{eq: adam_rate_part_1}
        &~~~~~~~~~~~`~~~~~~~~~~~~~~~~~~~~~~~~~~~~~~~~~~~~~~+\left.\left\Vert \sqrt{\varepsilon\mathds{1}_d+\hbnu(t-1)}\odot( \bw(t)-\bw(t-1))\right\Vert^2\right).
    \end{align}
    
    On the other hand, by Eq. (\ref{eq: adam_mid_result_1}), we have
    \begin{equation*}
         \sum_{t=1}^{\infty}\left\Vert \sqrt[4]{\varepsilon\mathds{1}_d+\hbnu(t)}\odot (\bw(t+1)-\bw(t))\right\Vert ^2<\infty.
    \end{equation*}
    Therefore, there exists large enough time $T_2$, such that $\forall t>T_2$, 
    \begin{equation*}
        \left\Vert \sqrt[4]{\varepsilon\mathds{1}_d+\hbnu(t)}\odot (\bw(t+1)-\bw(t))\right\Vert ^2<1,
    \end{equation*}
    and thus, 
    \begin{equation}
    \label{eq: adam_rate_part_2}
        \left\Vert \sqrt[4]{\varepsilon\mathds{1}_d+\hbnu(t)}\odot (\bw(t+1)-\bw(t))\right\Vert ^4<\left\Vert \sqrt[4]{\varepsilon\mathds{1}_d+\hbnu(t)}\odot (\bw(t+1)-\bw(t))\right\Vert ^2.
    \end{equation}
    
    Combining Eq. (\ref{eq: adam_rate_part_1}) and Eq. (\ref{eq: adam_rate_part_2}), there exists a positive real constant $C$, such that
    \begin{align*}
     \mL (\bw(t))^2\le& C\left(\left\Vert \sqrt{\varepsilon\mathds{1}_d+\hbnu(t)}\odot( \bw(t+1)-\bw(t))\right\Vert^2\right.
       \\
        &+\left.\left\Vert \sqrt{\varepsilon\mathds{1}_d+\hbnu(t-1)}\odot( \bw(t)-\bw(t-1))\right\Vert^2\right),
        \\
          \left\Vert \sqrt[4]{\varepsilon\mathds{1}_d+\hbnu(t)}\odot (\bw(t+1)-\bw(t))\right\Vert ^4\le& C\left(\left\Vert \sqrt{\varepsilon\mathds{1}_d+\hbnu(t)}\odot( \bw(t+1)-\bw(t))\right\Vert^2\right.
       \\
        &+\left.\left\Vert \sqrt{\varepsilon\mathds{1}_d+\hbnu(t-1)}\odot( \bw(t)-\bw(t-1))\right\Vert^2\right).
    \end{align*}
    
    Rearranging Eq. (\ref{eq: coro_loss_update_adam_t}) leads to
    \begin{align*}
        &\frac{\sqrt[4]{c}-1}{4\sqrt[2]{c}\eta(1-\beta_1)}\left(\left\Vert \sqrt[4]{\varepsilon\mathds{1}_d+\hbnu(t-1)}\odot (\bw(t)-\bw(t-1))\right\Vert ^2\right.
        \\
        &~~~~~~~~~~~~~~~~~~~~~~~~~~~~~~~~~~~~~+\left.\left\Vert \sqrt[4]{\varepsilon\mathds{1}_d+\hbnu(t-1)}\odot (\bw(t)-\bw(t-1))\right\Vert ^2\right)
    \\
    \le & \mL(\bw(t))+\frac{\sqrt[4]{c}+1}{4\sqrt[2]{c}\eta(1-\beta_1)}\left\Vert \sqrt[4]{\varepsilon\mathds{1}_d+\hbnu(t-1)}\odot (\bw(t)-\bw(t-1))\right\Vert ^2
    \\
    &~~~~~~~~~~~~~~~~~~~~~-\left(\mL(\bw(t+1))
    +\frac{\sqrt[4]{c}+1}{4\sqrt[2]{c}\eta(1-\beta_1)}\left\Vert \sqrt[4]{\varepsilon\mathds{1}_d+\hbnu(t)}\odot (\bw(t+1)-\bw(t))\right\Vert ^2\right),
    \end{align*}
    which further indicates
    \begin{align*}
        &\left(\mL(\bw(t))+\frac{\sqrt[4]{c}+1}{4\sqrt[2]{c}\eta(1-\beta_1)}\left\Vert \sqrt[4]{\varepsilon\mathds{1}_d+\hbnu(t-1)}\odot (\bw(t)-\bw(t-1))\right\Vert ^2\right)^2
        \\
        \le &2\left(\mL(\bw(t))^2+\frac{\sqrt[4]{c}+1}{4\sqrt[2]{c}\eta(1-\beta_1)}\left\Vert \sqrt[4]{\varepsilon\mathds{1}_d+\hbnu(t-1)}\odot (\bw(t)-\bw(t-1))\right\Vert ^4\right)
        \\
        \le &2C\left(1+\frac{\sqrt[4]{c}+1}{4\sqrt[2]{c}\eta(1-\beta_1)}\right)\left(\left\Vert \sqrt{\varepsilon\mathds{1}_d+\hbnu(t)}\odot( \bw(t+1)-\bw(t))\right\Vert^2\right.
        \\
        &~~~~~~~~~~~~~~~~~~~~~~~~~~~~~~~~~~~~~~~~~~~~~~~~~~+\left.\left\Vert \sqrt{\varepsilon\mathds{1}_d+\hbnu(t-1)}\odot( \bw(t)-\bw(t-1))\right\Vert^2\right)
        \\
        \le&2C\left(1+\frac{\sqrt[4]{c}+1}{4\sqrt[2]{c}\eta(1-\beta_1)}\right)\frac{4\sqrt[2]{c}\eta(1-\beta_1)}{\sqrt[4]{c}-1} \left(\mL(\bw(t))+\frac{\sqrt[4]{c}+1}{4\sqrt[2]{c}\eta(1-\beta_1)}\right.
    \\
    &\left.\cdot\left\Vert \sqrt[4]{\varepsilon\mathds{1}_d+\hbnu(t-1)}\odot (\bw(t)-\bw(t-1))\right\Vert ^2-\left(\mL(\bw(t+1))\right.\right.
    \\
    &+\left.\left.\frac{\sqrt[4]{c}+1}{4\sqrt[2]{c}\eta(1-\beta_1)}\left\Vert \sqrt[4]{\varepsilon\mathds{1}_d+\hbnu(t)}\odot (\bw(t+1)-\bw(t))\right\Vert ^2\right)\right).
    \end{align*}
    Denote $\xi(t)$ as
    \begin{equation*}
        \xi(t)\overset{\triangle}{=}\mL(\bw(t))+\frac{\sqrt[4]{c}+1}{4\sqrt[2]{c}\eta(1-\beta_1)}\left\Vert \sqrt[4]{\varepsilon\mathds{1}_d+\hbnu(t-1)}\odot (\bw(t)-\bw(t-1))\right\Vert ^2.
    \end{equation*}
    We then have
    \begin{equation*}
        \xi(t)^2\le 2C\left(1+\frac{\sqrt[4]{c}+1}{4\sqrt[2]{c}\eta(1-\beta_1)}\right)\frac{4\sqrt[2]{c}\eta(1-\beta_1)}{\sqrt[4]{c}-1}(\xi(t)-\xi(t+1)),
    \end{equation*}
    which leads to 
    \begin{align*}
        &\xi(t)=\mathcal{O}\left(\frac{1}{t}\right), ~i.e.,~ \mL(\bw(t))=\mathcal{O}\left(\frac{1}{t}\right),
        \\
        ~and&~ \left\Vert \sqrt[4]{\varepsilon\mathds{1}_d+\hbnu(t-1)}\odot (\bw(t)-\bw(t-1))\right\Vert ^2=\mathcal{O}\left(\frac{1}{t}\right).
    \end{align*}
    Due to Eq. (\ref{eq: equivalence_loss_gradient}), we further have $\Vert \nabla \mL(\bw(t))\Vert =\mathcal{O}(t^{-1})$, which indicates
    \begin{align*}
        \Vert \bw(t) \Vert \le& \Vert \bw(1)\Vert +\sum_{s=1}^t\Vert \bw(s+1)-\bw(s)\Vert
        =\Vert \bw(1)\Vert +\eta\sum_{s=1}^t\left\Vert \frac{\hbom(s)}{\sqrt{\hbnu(s)+\varepsilon\mathds{1}_d}}\right\Vert
        \\
        \le& \Vert \bw(1)\Vert +\frac{\eta}{\sqrt{\varepsilon}}\sum_{s=1}^t\left\Vert \hbom(s)\right\Vert=\Vert \bw(1)\Vert +\frac{\eta}{\sqrt{\varepsilon}}\sum_{s=1}^t\frac{1}{(1-\beta^s)}\left\Vert \sum_{i=1}^{s}\beta^{s-i} \nabla \mL(\bw(i))\right\Vert
        \\
        \le &\Vert \bw(1)\Vert +\frac{\eta}{\sqrt{\varepsilon}(1-\beta)}\sum_{s=1}^t\sum_{i=1}^{s}\beta^{s-i}\left\Vert  \nabla \mL(\bw(i))\right\Vert
        \\
        \le &\Vert \bw(1)\Vert +\frac{\eta}{\sqrt{\varepsilon}(1-\beta)^2}\sum_{s=1}^t\left\Vert  \nabla \mL(\bw(s))\right\Vert=\mathcal{O}(\ln(t)).
    \end{align*}
    
    Therefore, for any $\bx\in\bS$, we have $\langle \bw(t),\bx\rangle=\mathcal{O}(\ln(t))$, which by $\ell$ is exponential-tailed leads to  $\ell(\langle \bw(t),\bx\rangle)=\Omega(t^{-1})$, and thus $\mL(\bw(t))=\Theta(t^{-1})$. Also, since $\mL(\bw(t))=\mathcal{O}(t^{-1})$, we have  $\langle \bw(t),\bx\rangle=\Omega(\ln(t))$, which further leads to $\Vert \bw(t) \Vert=\Omega(\ln(t))$, and thus $\Vert \bw(t) \Vert=\Theta(\ln(t))$.
    
    Finally, we have
    \begin{align*}
       & \gamma\sum_{s=1}^t\beta^{t-s}\Vert \nabla \mL(\bw(s)) \Vert
        =\frac{\gamma}{N}\sum_{s=1}^t\beta^{t-s}\left\Vert\sum_{\bx\in\bS}\ell'(\langle \bw(s),\bx \rangle)\bx \right\Vert
        \\
        \le & -\frac{\gamma}{N}\sum_{s=1}^t\beta^{t-s}\sum_{\bx\in\bS}\ell'(\langle \bw(s),\bx \rangle)\le -\frac{1}{N}\sum_{s=1}^t\beta^{t-s}\left\langle\sum_{\bx\in\bS}\ell'(\langle \bw(s),\bx \rangle)\bx,\gamma\hbw\right\rangle
        \\
        \le &\left\Vert\sum_{s=1}^t\beta^{t-s} \nabla \mL(\bw(s))\right\Vert \left\Vert\gamma\hbw \right\Vert = \left\Vert\bom(t) \right\Vert\le\sum_{s=1}^t\beta^{t-s}\Vert \nabla \mL(\bw(s)) \Vert,
    \end{align*}
    which leads to $\Vert \bom(t) \Vert=\Theta(t^{-1})$. Similarly, we have $\bnu(t)=\mathcal{O}(t^{-2})$, component-wisely. As $\lim_{t\rightarrow\infty} \beta_1^t=0$ and $\lim_{t\rightarrow\infty} \beta_2^t=0$, we have 
    \begin{equation*}
        \left\Vert\bw(t)-\bw(t-1) \right\Vert= \left\Vert\frac{\hbom(t)}{\sqrt{\varepsilon\mathds{1}_d+\hbnu(t)}} \right\Vert=\Theta(t^{-1}).
    \end{equation*}
    
    The proof is completed.
\end{proof}

\subsection{Parameter dynamics}
By Lemma \ref{lem: structure of max-margin vector}, there exists a solution $\tbw$ as the solution of Eq. (\ref{eq: represent_v}) with ${  C_3}=\frac{\eta}{(1-\beta)\sqrt{\varepsilon}}$. Define $\br(t)$ as
\begin{equation}
\label{eq: def_r_adam}
    \br(t)\overset{\triangle}{=} \bw(t)-\ln (t)\hbw-\tbw,
\end{equation}
and we only need to prove $\Vert \br(t) \Vert$ is bounded over time. We then prove $\br(t)$ has bounded norm. Specifically, we will prove the following lemma:
\begin{lemma}
\label{lem: g_equivalent_r_adam}
Let all conditions in Theorem \ref{thm: adam_main} hold. Then, $\Vert \br(t) \Vert$ is bounded  if and only if $g(t)$ is upper bounded, where $g(t)$ is defined as follows.
\begin{small}
\begin{align*}
   &g(t)\overset{\triangle}{=}\left\langle \br(t),(1-\beta_1^{t-1}) \sqrt{\varepsilon\mathds{1}_d+\hbnu(t-1)}\odot\left(\bw(t)-\bw(t-1)\right)\right\rangle
   \\
   \cdot&\frac{\beta_1}{1-\beta_1}
   +\frac{\sqrt{\varepsilon}}{2}\Vert \br(t)\Vert^2-\frac{\beta_1}{1-\beta_1}\sum_{\tau=2}^t\langle \br(\tau)-\br(\tau-1),
   \\
   &(1-\beta_1^{\tau-1}) \sqrt{\varepsilon\mathds{1}_d+\hbnu(\tau-1)}\odot\left(\bw(\tau)-\bw(\tau-1)\right)\rangle.
\end{align*}
\end{small}
 Furthermore, we have $\sum_{s=1}^{\infty} (g(t+1)-g(t))$ is upper bounded.
\end{lemma}

Similar to GDM, the proof of Lemma \ref{lem: g_equivalent_r_adam} is divided into two parts, each focus on one claim of it. We start with the first claim.

\begin{lemma}
\label{lem: g_equivalent_r_adam_1}
Let all conditions in Theorem \ref{thm: adam_main} hold. Then, $\Vert \br(t) \Vert$ is bounded  if and only if $g(t)$ is upper bounded.
\end{lemma}

\begin{proof}
    Following the same routine as Lemma \ref{lem:gdm_r_t_bounded} and Lemma \ref{lem:sgdm_r_t_bounded}, we only need to prove
    \begin{equation}
    \label{eq: adam_g_equiv_1}
       \lim_{t\rightarrow \infty} \left\Vert(1-\beta_1^{t-1}) \sqrt{\varepsilon\mathds{1}_d+\hbnu(t-1)}\odot\left(\bw(t)-\bw(t-1)\right)\right\Vert=0,
    \end{equation}
    and 
    \begin{equation}
    \label{eq: adam_g_equiv_2}
        \sum_{\tau=2}^{\infty}\left\vert\langle \br(\tau)-\br(\tau-1),(1-\beta_1^{\tau-1}) \sqrt{\varepsilon\mathds{1}_d+\hbnu(\tau-1)}\odot\left(\bw(\tau)-\bw(\tau-1)\right)\rangle\right\vert<\infty.
    \end{equation}
    
    As for Eq. (\ref{eq: adam_g_equiv_1}), by Lemma \ref{lem: adam_asymptotic}, we have
    \begin{align*}
        &\left\Vert(1-\beta_1^{t-1}) \sqrt{\varepsilon\mathds{1}_d+\hbnu(t-1)}\odot\left(\bw(t)-\bw(t-1)\right)\right\Vert
        \\
        =&\mathcal{O}(t^{-1})=\bo(1).
    \end{align*}
    
    As for Eq. (\ref{eq: adam_g_equiv_1}), we have
    \begin{align*}
        &\left\vert\left\langle \br(\tau)-\br(\tau-1),(1-\beta_1^{\tau-1}) \sqrt{\varepsilon\mathds{1}_d+\hbnu(\tau-1)}\odot\left(\bw(\tau)-\bw(\tau-1)\right)\right\rangle\right\vert
        \\
        =&\left\vert\left\langle \bw(\tau)-\bw(\tau-1)-\ln \frac{\tau}{\tau-1}\hbw,(1-\beta_1^{\tau-1}) \sqrt{\varepsilon\mathds{1}_d+\hbnu(\tau-1)}\odot\left(\bw(\tau)-\bw(\tau-1)\right)\right\rangle\right\vert
        \\
        \le &\left\vert\left\langle\ln \frac{\tau}{\tau-1}\hbw,(1-\beta_1^{\tau-1}) \sqrt{\varepsilon\mathds{1}_d+\hbnu(\tau-1)}\odot\left(\bw(\tau)-\bw(\tau-1)\right)\right\rangle\right\vert
        \\
        &+(1-\beta_1^{\tau-1})\left\Vert  \sqrt[4]{\varepsilon\mathds{1}_d+\hbnu(\tau-1)}\odot\left(\bw(\tau)-\bw(\tau-1)\right)\right\Vert^2
        \\
        \overset{(\star)}{=}& \mathcal{O}(\tau^{-2}),
    \end{align*}
    where Eq. $(\star)$ is due to Lemma \ref{lem: adam_asymptotic} and $\ln(\frac{\tau}{\tau-1})=\Theta(\tau^{-1})$.
    
    The proof is completed.
\end{proof}

We conclude the proof of Theorem \ref{thm: adam_main} by showing $g(t)$ is upper bounded.
\begin{lemma}
Let all conditions in Theorem \ref{thm: adam_main} hold. Then, $g(t)$ is upper bounded.
\end{lemma}
\begin{proof}
    $g(t)$ is upper bounded is equivalent to $\sum_{t=1}^{\infty} g(t+1)-g(t)<\infty$. We then prove this lemma by calculating $g(t+1)-g(t)$ directly.
    \begin{align*}
        &g(t+1)-g(t)
        \\
        =&\frac{\sqrt{\varepsilon}}{2}\Vert \br(t+1)\Vert^2+\frac{\beta_1}{1-\beta_1}\left\langle \br(t+1),(1-\beta_1^{t}) \sqrt{\varepsilon\mathds{1}_d+\hbnu(t)}\odot\left(\bw(t+1)-\bw(t)\right)\right\rangle
        \\
        &-\left(\frac{\sqrt{\varepsilon}}{2}\Vert \br(t)\Vert^2+\frac{\beta_1}{1-\beta_1}\left\langle \br(t),(1-\beta_1^{t-1}) \sqrt{\varepsilon\mathds{1}_d+\hbnu(t-1)}\odot\left(\bw(t)-\bw(t-1)\right)\right\rangle\right)
        \\
        &-\frac{\beta_1}{1-\beta_1}\langle \br(t+1)-\br(t),(1-\beta_1^{t}) \sqrt{\varepsilon\mathds{1}_d+\hbnu(t)}\odot\left(\bw(t+1)-\bw(t)\right)\rangle
        \\
        =&\frac{\sqrt{\varepsilon}}{2}\Vert \br(t+1)-\br(t)\Vert^2+\frac{\beta_1}{1-\beta_1}\left\langle \br(t),(1-\beta_1^{t}) \sqrt{\varepsilon\mathds{1}_d+\hbnu(t)}\odot\left(\bw(t+1)-\bw(t)\right)\right.
        \\
        &\left.-(1-\beta_1^{t-1}) \sqrt{\varepsilon\mathds{1}_d+\hbnu(t-1)}\odot\left(\bw(t)-\bw(t-1)\right)\right\rangle+\sqrt{\varepsilon}\langle \br(t+1)-\br(t),\br(t)\rangle
        \\
        \overset{(\star)}{=}& \left\langle \br(t),-(1-\beta_1^{t}) \sqrt{\varepsilon\mathds{1}_d+\hbnu(t)}\odot\left(\bw(t+1)-\bw(t)\right)-\frac{\eta}{1-\beta}\nabla \mL(\bw(t))\right\rangle
        \\
        &+\frac{\sqrt{\varepsilon}}{2}\Vert \br(t+1)-\br(t)\Vert^2+\sqrt{\varepsilon}\langle \br(t+1)-\br(t),\br(t)\rangle,
    \end{align*}
    where Eq. $(\star)$ is due to a simple rearranging of the update rule of Adam, i.e., 
    \begin{small}
    \begin{align*}
        &\frac{\beta_1}{1-\beta_1}\left((1-\beta_1^{t}) \sqrt{\varepsilon\mathds{1}_d+\hbnu(t)}\odot\left(\bw(t+1)-\bw(t)\right)-(1-\beta_1^{t-1}) \sqrt{\varepsilon\mathds{1}_d+\hbnu(t-1)}\odot\left(\bw(t)-\bw(t-1)\right)\right)
        \\
        &=-\frac{\eta}{1-\beta_1}\nabla \mL(\bw(t))-(1-\beta_1^{t}) \sqrt{\varepsilon\mathds{1}_d+\hbnu(t)}\odot\left(\bw(t+1)-\bw(t)\right).
    \end{align*}
    \end{small}
    On the one hand, as $\Vert\br(t+1)-\br(t)\Vert= \Vert\bw(t+1)-\bw(t)- \ln \frac{t+1}{t} \hbw\Vert=\mathcal{O}(t^{-1})$,
    \begin{equation*}
        \sum_{t=1}^{\infty} \frac{\sqrt{\varepsilon}}{2}\Vert\br(t+1)-\br(t)\Vert^2<\infty.
    \end{equation*}
    
    On the other hand, 
    \begin{align*}
        &\left\langle \br(t),-(1-\beta_1^{t}) \sqrt{\varepsilon\mathds{1}_d+\hbnu(t)}\odot\left(\bw(t+1)-\bw(t)\right)-\frac{\eta}{1-\beta}\nabla \mL(\bw(t))\right\rangle
        \\
        &+\sqrt{\varepsilon}\langle \br(t+1)-\br(t),\br(t)\rangle
        \\
        =&\left\langle \br(t),-(1-\beta_1^{t}) \sqrt{\varepsilon\mathds{1}_d+\hbnu(t)}\odot\left(\bw(t+1)-\bw(t)\right)-\frac{\eta}{1-\beta}\nabla \mL(\bw(t))\right\rangle
        \\
        &+\sqrt{\varepsilon}\left\langle \bw(t+1)-\bw(t)-\ln\left(\frac{t+1}{t}\right)\hbw,\br(t)\right\rangle
        \\
        =&\left\langle \br(t),-(1-\beta_1^{t}) \sqrt{\varepsilon\mathds{1}_d+\hbnu(t)}\odot\left(\bw(t+1)-\bw(t)\right)+\sqrt{\varepsilon} (\bw(t+1)-\bw(t))\right\rangle
        \\
        &+\left\langle \br(t),-\sqrt{\varepsilon}\ln\left(\frac{t+1}{t}\right)\hbw-\frac{\eta}{1-\beta}\nabla \mL(\bw(t))\right\rangle
        \\
        \overset{(\bullet)}{=}&\mathcal{O}(\beta_1^t+t^{-2})+\left\langle \br(t),-\sqrt{\varepsilon}\ln\left(\frac{t+1}{t}\right)\hbw-\frac{\eta}{1-\beta}\nabla \mL(\bw(t))\right\rangle,
    \end{align*}
    where Eq. ($\bullet$) is due to $\hbnu(t)=\mathcal{O}(t^{-2})$.
    
    Furthermore, following exactly the same routine as Lemma \ref{lem: bounded_g_gdm}, we have 
    \begin{equation*}
        \sum_{t=1}^{\infty}\left\langle \br(t),-\sqrt{\varepsilon}\ln\left(\frac{t+1}{t}\right)\hbw-\frac{\eta}{1-\beta}\nabla \mL(\bw(t))\right\rangle<\infty.
    \end{equation*}
    
    The proof is completed.
\end{proof}

\section{Implicit regularization of  RMSProp (w/o. r) with decaying learning rate}
\label{appen: random_shuffling_rmsprop}

This section collects the proof of Theorem \ref{thm: ib_rms}. To begin with, we formally define RMSProp (w/o. r) as follows to facilitate latter analysis: for each $t\in\{0,1,2\cdots\}$, divide the sample set $S$ into $K$ subsets $\{\bB(Kt+1),\cdots,\bB(K(t+1))\}$ uniformly and i.i.d., and let
\begin{gather}
\nonumber
\bnu(0)=0,\bnu(\tau)=\beta_2 \bnu(\tau-1)+(1-\beta_2)\left(\nabla \mL_{\bB(\tau)}(\bw(\tau-1))\right)^2,
\\
\label{eq: mb_rms}
   \textit{(RMSProp (w/o. r))}:~ \bw(\tau)=\bw(\tau-1)-\eta_{\tau}\frac{\nabla \mL_{\bB(\tau)}(\bw(\tau-1))}{\sqrt{\bnu(\tau)+\varepsilon\mathds{1}_d}}.
\end{gather}
Here the $\mL_{\bB(\tau)}$ is the individual loss average over $\bB(\tau)$, i.e., $ \mL_{\bB(\tau)}(\bw)=\frac{\sum_{(\bx,\by)\in\bB(\tau)}\ell(-\by\langle \bw,\bx\rangle)}{b}$, where $b=\frac{N}{K}$ is the batch size.
With Eq. (\ref{eq: mb_rms}), we restate the loss  convergence result in \cite{shi2021rmsprop} as the following proposition.

\begin{proposition}[Corollary 4.1 in \cite{shi2021rmsprop}, restated] Suppose $\ell$ is non-negative and
$L$-smooth. Furthermore, assume that there exists a constant $D$, s.t., $\forall t\ge 0$, $\forall \bw \in \mathbb{R}^d$ 
\begin{equation}
\label{eq: realizable_condition}
    \sum_{\tau=Kt+1}^{K(t+1)} \Vert \nabla \mL_{\bB(\tau)}(\bw)\Vert^2 \le D \Vert N \nabla\mL(\bw)\Vert^2,
\end{equation}
and 
\begin{equation*}
    T_{2}\left(\beta_{2}\right) \triangleq \sqrt{\frac{10 d K}{\beta_{2}^{K}}} d K D\left(\left(1-\beta_{2}\right) \frac{\left(\frac{4 K^{2}}{\beta_{2}^{K}}-1\right)}{2}+\left(\frac{1}{\sqrt{\beta_{2}^{K}}}-1\right)\right) \leq \frac{\sqrt{2}-1}{2 \sqrt{2}}.
\end{equation*}
Then, RMSProp (w/o. r) with decaying learning rate $\eta_{\tau}=\frac{\eta_1}{\sqrt{\tau}}$ satisfies
\begin{equation*}
    \sum_{t=0}^T \frac{1}{\sqrt{t+1}} \Vert\nabla\mL(\bw(Kt))\Vert=\mathcal{O}(\ln (T)).
\end{equation*}
\end{proposition}

Combining Assumptions \ref{assum: separable} and \ref{assum: exponential-tailed}, we immediately get the following corollary:
\begin{corollary}
\label{coro: naichen}
Let Assumptions \ref{assum: separable}, \ref{assum: exponential-tailed}, and \ref{assum: smooth}. (S) hold. Let 
\begin{equation*}
     T_{2}\left(\beta_{2}\right) \triangleq \sqrt{\frac{10 d K}{\beta_{2}^{K}}} \frac{d K}{b^2\gamma^2}\left(\left(1-\beta_{2}\right) \frac{\left(\frac{4 K^{2}}{\beta_{2}^{K}}-1\right)}{2}+\left(\frac{1}{\sqrt{\beta_{2}^{K}}}-1\right)\right)\leq \frac{\sqrt{2}-1}{2 \sqrt{2}}.
\end{equation*}
Then, there exists a positive constant $C_1$ and $C_3$ independent of random sampling, s.t.,
\begin{equation*}
    \sum_{t=0}^T \frac{1}{\sqrt{t+1}} \Vert\nabla\mL(\bw(Kt))\Vert\le C_1\ln (T), \forall T\ge 0,
\end{equation*}
and
\begin{equation*}
    \mL(\bw(K(t+1)))\le\mL(\bw(Kt)) -\frac{C_3}{\sqrt{t+1}}\Vert \nabla \mL(\bw(Kt)) \Vert+\frac{C_1C_3}{t+1}.
\end{equation*}
\end{corollary}

Applying relationship between $\ell$ and $\ell'$ of the exponentially-tailed loss, we can further obtain the loss convergent rate.

\begin{lemma}
\label{lem: loss_convergence_rmsprop}
Let all the conditions in Theorem \ref{thm: ib_rms} hold. Then, we have
\begin{equation*}
     \mL(\bw(Kt))=\mathcal{O}\left(\frac{1}{\sqrt{t}}\right).
\end{equation*}
\end{lemma}

\begin{proof}
    To begin with, we show that there exists an increasing positive integer sequence $\{t_i\}_{i=1}^{\infty}$, and $\Vert\nabla\mL(\bw(Kt_i))\Vert\le 2C_1\frac{1}{\sqrt{t_i+1}}$ by reduction to absurdity. Otherwise, suppose there exists a positive integer $T_1$, such that  $\Vert\nabla\mL(\bw(Kt))\Vert> 2C_1\frac{1}{\sqrt{t+1}}$, $\forall t\ge T_1$. Therefore, for $T\ge T_1$, we have
    \begin{align*}
        &\sum_{t=0}^T \frac{1}{\sqrt{t+1}} \Vert\nabla\mL(\bw(Kt))\Vert
        \ge  \sum_{t=T_1}^T \frac{1}{\sqrt{t+1}} \Vert\nabla\mL(\bw(Kt))\Vert
        \\
        \ge & 2C_1\sum_{t=T_1}^T \frac{1}{t+1} 
        \ge  2C_1 \ln \frac{T+2}{T_1+1}.
    \end{align*}
    Let $T$ be large enough, we have $2C_1 \ln \frac{T+2}{T_1+1}> C_1\ln T$, which contradicts Corollary \ref{coro: naichen}. 
    
    Denote $\mT=\{t>0: \Vert\nabla\mL(\bw(Kt))\Vert \le 2C_1\frac{1}{\sqrt{t+1}}\}$. By the above discussion, $\mT$ contains an increasing positive integer sequence. We then prove that if $s\in \mT$ and $s>s^{\mT}$, where $s^{\mT}$ is defined as
    \begin{equation*}
        s^{\mT}\triangleq \max\left\{\left(\frac{2C_1}{C_g}\right)^2-1,(C_3\gamma)^2-1,\left(\frac{9C_1}{\gamma C_l}\right)^2-1,4,\frac{4}{C_3\gamma(\sqrt{\frac{3}{2}}-1)}\ln \frac{18\sqrt{\frac{3}{2}}}{\gamma}-2\right\}
    \end{equation*}
    ($C_l$ and $C_g$ is defined in Corollary \ref{coro:gradient_small_equivalence}), there exists $s<r\le 2s$, and $r\in \mT$. We slightly abuse the notation and let $r=\inf\{t: t\in \mT, t>s\}$. If $r=s+1$, this claim trivially holds. Otherwise, as $s\in \mT$, we have 
    \begin{equation*}
        \Vert \nabla \mL (\bw(Ks)) \Vert\le 2C_1\frac{1}{\sqrt{s+1}}\le C_g,
    \end{equation*}
    which by Corollary \ref{coro:gradient_small_equivalence} further leads to
    \begin{equation*}
        \mL(\bw(Ks))\le \frac{4}{\gamma}  \Vert \nabla \mL (\bw(Ks)) \Vert.
    \end{equation*}
    Therefore, by Corollary \ref{coro: naichen}, we have 
    \begin{equation*}
         \mL(\bw(K(s+1)))\le \mL(\bw(Ks)) + \frac{C_1C_3}{s+1}\le \frac{8C_1}{\gamma}\frac{1}{\sqrt{s+1}}+\frac{C_1C_3}{s+1}\le \frac{9C_1}{\gamma}\frac{1}{\sqrt{s+1}},
    \end{equation*}
    which by $s>s^{\mT}$ further leads to
    \begin{equation*}
        \mL(\bw(K(s+1)))\le C_l,
    \end{equation*}
    and thus $\mL(\bw(K(s+1)))\le \frac{4}{\gamma} \Vert \nabla \mL(\bw(K(s+1))) \Vert$.
    As $(s+1)\notin \mT$, we have
    \begin{equation*}
        \Vert\nabla\mL(\bw(K(s+1)))\Vert > 2C_1\frac{1}{\sqrt{s+2}},
    \end{equation*}
    which leads to
    \begin{equation*}
         \mL(\bw(K(s+2)))\le \mL(\bw(K(s+1)))-\frac{C_3}{2\sqrt{s+2}} \Vert \nabla \mL(\bw(K(s+1)))\Vert\le \left(1-\frac{C_3\gamma}{8\sqrt{s+2}}\right)\mL(\bw(K(s+1))),
    \end{equation*}
    and thus $\mL(\bw(K(s+2)))\le C_l$. By the inductive method, we have for any $j\in \{s+2,\cdots,r\}$,
    \begin{align*}
         \mL(\bw(K(j)))&\le \Pi_{i=s+2}^j\left(1-\frac{C_3\gamma}{8\sqrt{i}}\right) \mL(\bw(K(s+1)))\le e^{-\sum_{i=s+2}^j\frac{C_3\gamma}{8\sqrt{i}}}\mL(\bw(K(s+1)))
         \\
         \le & e^{-\frac{C_3\gamma(\sqrt{j+1}-\sqrt{s+2})}{4}}\mL(\bw(K(s+1)))\le e^{-\frac{C_3\gamma(\sqrt{j+1}-\sqrt{s+2})}{4}}\frac{9C_1}{\gamma \sqrt{s+1}}.
         \end{align*}
         
    If $r>2s$, applying $j=2s$ into the above equation leads to
    \begin{equation*}
         \mL(\bw(K(2s)))\le  e^{-\frac{C_3\gamma(\sqrt{2s+1}-\sqrt{s+2})}{4}}\frac{9C_1}{\gamma \sqrt{s+1}}\le C_l,
    \end{equation*}
    and
    \begin{equation*}
         \Vert \nabla\mL(\bw(K(2s))) \Vert\le  e^{-\frac{C_3\gamma(\sqrt{2s+1}-\sqrt{s+2})}{4}}\frac{36C_1}{\gamma \sqrt{s+1}}\le \frac{2C_1}{\sqrt{2s+1}},
    \end{equation*}
    which leads to $2s\in \mT$, and contradicts the definition of $r$. Therefore, we have $r\le 2s$.
    
    As $\mT$ contains an increasing integer sequence, there exists an $s_0$, s.t., $s_0\in \mT$ and $s_0>s^{\mT}$. Let $t$ be any positive integer larger than $s_0$ and let $t'$ be the largest integer smaller than $t$ and belongs to $\mT$. We have $t'\ge s_0$, and $t\le 2t'$ by the above discussion. Therefore, we have
    \begin{equation*}
        \mL(Kt)\le \frac{9C_1}{\gamma}\frac{1}{\sqrt{t'+1}}\le \frac{9\sqrt{2}C_1}{\gamma}\frac{1}{\sqrt{t+2}}.
    \end{equation*}
    
    The proof is completed.
\end{proof}

As a corollary, we can obtain an asymptotic estimation of $\nabla \mL_{\bB(\tau)} (\bw(\tau))$. 
\begin{corollary}
\label{coro: loss_rms_mb}
Let all the conditions in Theorem \ref{thm: ib_rms} hold. Then, we have 
\begin{equation*}
    \Vert\nabla\mL_{\bB(\tau)}(\bw(\tau))\Vert=\mathcal{O}\left(\frac{1}{\sqrt{\tau}}\right).
\end{equation*}
\end{corollary}
\begin{proof}
    Let $\tau>K (s^{\mT}+1)$, where $s^{\mT}$ is defined as Lemma \ref{lem: loss_convergence_rmsprop}. Let $t=\lceil\frac{\tau}{K} \rceil> s^{\mT}$  and $s=\tau-Kt$. Then, we have
    \begin{equation*}
        \Vert\nabla \mL(Kt)\Vert\le C_l,
    \end{equation*}
    and 
    \begin{equation*}
        \Vert\nabla \mL(\bw(Kt))\Vert\le 4\mL(\bw(Kt))\le \frac{36\sqrt{2}C_1}{\gamma}\frac{1}{\sqrt{t+2}}.
    \end{equation*}
    
    On the other hand, we have 
    \begin{equation*}
        \Vert\nabla \mL_{\bB(\tau)}(\bw(Kt))\Vert \le \frac{1}{\gamma} \frac{N}{b}\Vert\nabla \mL(\bw(Kt))\Vert\le\frac{N}{b\gamma} \frac{36\sqrt{2}C_1}{\gamma}\frac{1}{\sqrt{t+2}}.
    \end{equation*}
    
    As $\mL_{\bB(\tau)}$ is $H$ smooth, we further have
    \begin{align*}
       \Vert\nabla \mL_{\bB(\tau)}(\bw(Kt))-\nabla \mL_{\bB(\tau)}(\bw(\tau))\Vert\le& H \left\Vert \sum_{i=0}^{s-1}  \eta_{Kt+i+1}\frac{\nabla \mL_{\bB(Kt+i+1)}(\bw(Kt+i))}{\sqrt{\hbnu(Kt+i+1)+\varepsilon\mathds{1}_d}}\right\Vert
       \\
       \le & \frac{KH}{\sqrt{1-\beta_2}\sqrt{Kt+1}}.
    \end{align*}
Combining the above two equations, we have
\begin{align*}
    \Vert \nabla \mL_{\bB(\tau)}(\bw(\tau)) \Vert\le &\frac{KH}{\sqrt{1-\beta_2}\sqrt{Kt+1}}+\frac{N}{b\gamma} \frac{36\sqrt{2}C_1}{\gamma}\frac{1}{\sqrt{t+2}}
    \\
    \le &\frac{KH}{\sqrt{1-\beta_2}\sqrt{\tau-K}}+\frac{N}{b\gamma} \frac{36\sqrt{2}C_1}{\gamma}\frac{\sqrt{K}}{\sqrt{\tau+K}}.
\end{align*}

The proof is completed.
\end{proof}

Using Corollary \ref{coro: loss_rms_mb}, we are now able to prove Theorem \ref{thm: ib_rms}.

\begin{proof}[Proof of Theorem \ref{thm: ib_rms} (for almost every dataset)]

To begin with, define $\br(t)\triangleq\bw(t)-\frac{1}{2}\ln \left(\frac{\eta}{K} t\right)\hbw-\tbw-\frac{1}{2}\bn(t)$, where $\tbw$ is the solution of Eq. (\ref{eq: represent_v}) with ${   C_3}=\frac{2}{N\sqrt{\varepsilon}}$, and $\bn(t)$ is given by Lemma \ref{lem: define_n_t}. As in the case of GDM, SGDM, and Adam (w/s), $\bw(t)-\frac{1}{2}\ln \left(\frac{\eta}{K} t\right)$ has bounded norm if and only if $\br(t)$ has bounded norm. Also, it is a sufficient condition to ensure $\Vert \br(t) \Vert$ is bounded that both $A\triangleq\sum_{t=1}^{\infty} \Vert \br(t+1)-\br(t)\Vert^2<\infty$ and  $B\triangleq\sum_{t=1}^{\infty} \langle \br(t+1)-\br(t), \br(t)\rangle<\infty$.

As for $A$, we have
\begin{align*}
    A=&\sum_{t=1}^{\infty}\left\Vert \br(t)-\br(t-1) \right\Vert^2
    \\
    =&\sum_{t=1}^{\infty} \left\Vert \bw(t+1)-\bw(t) -\frac{1}{2}\ln \left( \frac{t+1}{t}\right)\hbw-\frac{1}{2}\bn(t+1)+\frac{1}{2}\bn(t)\right\Vert^2
    \\
    \le & 3\left(\sum_{t=1}^{\infty} \left\Vert \bw(t+1)-\bw(t) \right\Vert^2+\frac{1}{2}\sum_{t=1}^{\infty}\ln \left( \frac{t+1}{t}\right) \left\Vert\hbw\right\Vert^2+\sum_{t=1}^{\infty}\frac{1}{4} \left\Vert -\bn(t+1)+\bn(t)\right\Vert^2\right)
    \\
    \overset{(*)}{<}&\infty,
\end{align*}
where Inequality $(*)$ is due to
\begin{equation*}
    \left\Vert \bw(t+1)-\bw(t) \right\Vert=\frac{\eta_1}{\sqrt{t+1}}\left\Vert \frac{\nabla \mL_{\bB(t)}(\bw(t))}{\sqrt{\varepsilon\mathds{1}_d+\hbnu(t+1)}}  \right\Vert\le \frac{\eta_1}{\sqrt{t+1}\sqrt{\varepsilon}}\left\Vert \nabla \mL_{\bB(t)}(\bw(t))\right\Vert=\mathcal{O}\left(\frac{1}{t}\right),
\end{equation*}
$\ln \left( \frac{t+1}{t}\right)=\mathcal{O}\left(\frac{1}{t}\right)$, and $\left\Vert -\bn(t+1)+\bn(t)\right\Vert=\mathcal{O}\left(\frac{1}{t}\right)$.

As for $B$, we have 
\begin{align}
\nonumber
    B=&\sum_{t=1}^{\infty} \left\langle \bw(t+1)-\bw(t) -\frac{1}{2}\ln \left( \frac{t+1}{t}\right)\hbw-\frac{1}{2}\bn(t+1)+\frac{1}{2}\bn(t), \br(t)\right\rangle
    \\
    \nonumber
    =&\sum_{t=1}^{\infty} \left\langle\bw(t+1)-\bw(t)-\frac{N}{2bt}\sum_{i: \bx_i\in \bB(t)\cap \bS_s} \bv_i \bx_i, \br(t)\right\rangle
    \\
    \nonumber
    =&\sum_{t=1}^{\infty} \left\langle-\frac{1}{\sqrt{t+1}}\frac{\nabla \mL_{\bB(t)}(\bw(t))}{\sqrt{\varepsilon\mathds{1}_d+\hbnu(t+1)}}-\frac{N}{2bt}\sum_{i: \bx_i\in \bB(t)\cap \bS_s} \bv_i \bx_i, \br(t)\right\rangle
    \\
    \nonumber
    =&\sum_{t=1}^{\infty} \left\langle-\frac{1}{\sqrt{t+1}}\frac{\nabla \mL_{\bB(t)}(\bw(t))}{\sqrt{\varepsilon\mathds{1}_d+\hbnu(t+1)}}+\frac{1}{\sqrt{t}}\frac{\nabla \mL(\bw(t))}{\sqrt{\varepsilon\mathds{1}_d+\hbnu(t+1)}}, \br(t)\right\rangle
    \\
    \nonumber
    &+\sum_{t=1}^{\infty} \left\langle-\frac{1}{\sqrt{t}}\frac{\nabla \mL_{\bB(t)}(\bw(t))}{\sqrt{\varepsilon\mathds{1}_d+\hbnu(t+1)}}+\frac{1}{\sqrt{t}\sqrt{\varepsilon}}\nabla \mL_{\bB(t)}(\bw(t)), \br(t)\right\rangle
    \\
    \label{eq: rms_ib_mid_result}
    &+\sum_{t=1}^{\infty} \left\langle-\frac{1}{\sqrt{t}\sqrt{\varepsilon}}\nabla \mL_{\bB(t)}(\bw(t))-\frac{N}{2bt}\sum_{i: \bx_i\in \bB(t)\cap \bS_s} \bv_i \bx_i, \br(t)\right\rangle.
\end{align}

On the other hand, as 
\begin{equation*}
    \Vert \bw(t) \Vert \le \Vert\bw(1) \Vert+ \sum_{s=1}^{t-1} \frac{1}{\sqrt{s+1}}\left\Vert \frac{\nabla \mL_{\bB(t)} (\bw(s))}{\sqrt{\varepsilon\mathds{1}_d}+\hbnu(s+1)} \right\Vert =\mathcal{O}(\ln(t)),
\end{equation*}
we have $\Vert \br(t) \Vert=\mathcal{O}(\ln(t))$. Also, we have $\Vert \nabla \mL_{\bB(t)}(\bw(t))^2\Vert =\mathcal{O}(\frac{1}{t})$, and thus $\left\Vert \frac{1}{\sqrt{\varepsilon\mathds{1}_d+\hbnu(t)}}-\frac{1}{\sqrt{\varepsilon\mathds{1}_d}}\right\Vert =\mathcal{O}(\frac{1}{t})$. Combining these estimations, we have 
\begin{equation*}
   \left\langle-\frac{1}{\sqrt{t+1}}\frac{\nabla \mL_{\bB(t)}(\bw(t))}{\sqrt{\varepsilon\mathds{1}_d+\hbnu(t+1)}}+\frac{1}{\sqrt{t}}\frac{\nabla \mL_{\bB(t)}(\bw(t))}{\sqrt{\varepsilon\mathds{1}_d+\hbnu(t+1)}}, \br(t)\right\rangle=\mathcal{O}\left(\frac{\ln t}{t^{\frac{5}{2}}}\right).
\end{equation*}
and 
\begin{equation*}
    \left\langle-\frac{1}{\sqrt{t}}\frac{\nabla \mL_{\bB(t)}(\bw(t))}{\sqrt{\varepsilon\mathds{1}_d+\hbnu(t+1)}}+\frac{1}{\sqrt{t}\sqrt{\varepsilon}}\nabla \mL_{\bB(t)}(\bw(t)), \br(t)\right\rangle=\mathcal{O}\left(\frac{\ln t}{t^2}\right),
\end{equation*}
Therefore, 
\begin{gather*}
\sum_{t=1}^{\infty} \left\langle-\frac{1}{\sqrt{t+1}}\frac{\nabla \mL_{\bB(t)}(\bw(t))}{\sqrt{\varepsilon\mathds{1}_d+\hbnu(t+1)}}+\frac{1}{\sqrt{t}}\frac{\nabla \mL_{\bB(t)}(\bw(t))}{\sqrt{\varepsilon\mathds{1}_d+\hbnu(t+1)}}, \br(t)\right\rangle<\infty,
\\
  \sum_{t=1}^{\infty} \left\langle-\frac{1}{\sqrt{t}}\frac{\nabla \mL_{\bB(t)}(\bw(t))}{\sqrt{\varepsilon\mathds{1}_d+\hbnu(t+1)}}+\frac{1}{\sqrt{t}\sqrt{\varepsilon}}\nabla \mL_{\bB(t)}(\bw(t)), \br(t)\right\rangle <\infty .
\end{gather*}

As for the last term in Eq. (\ref{eq: rms_ib_mid_result}), we have 
\begin{align*}
    &\left\langle-\frac{1}{\sqrt{t}\sqrt{\varepsilon}}\nabla \mL_{\bB(t)}(\bw(t))-\frac{N}{2bt}\sum_{i: \bx_i\in \bB(t)\cap \bS_s} \bv_i \bx_i, \br(t)\right\rangle
    \\
    =&\frac{1}{b\sqrt{t}\sqrt{\varepsilon}}\sum_{i: \bx_i\in \bB(s)\cap \bS_s}\left(- \ell'(\langle \bw(t),\bx_i\rangle)-\frac{1}{\sqrt{t}} e^{-\langle \tbw,\bx_i\rangle}\right)\langle \br(t), \bx_i\rangle
    \\
    &+\frac{1}{b\sqrt{t}\sqrt{\varepsilon}}\sum_{i: \bx_i\in \bB(s)\cap \bS_s^c}\left\langle \br(t),  -\ell'(\langle \bw(t),\bx_i\rangle)\bx_i\right\rangle.
\end{align*}

Then, following the same routine as Lemma \ref{lem: sgdm_ib}, we have 
\begin{equation*}
    \sum_{t=1}^{\infty} \left\langle-\frac{1}{\sqrt{t}\sqrt{\varepsilon}}\nabla \mL_{\bB(t)}(\bw(t))-\frac{N}{2bt}\sum_{i: \bx_i\in \bB(t)\cap \bS_s} \bv_i \bx_i, \br(t)\right\rangle<\infty.
\end{equation*}

The proof is completed.
\end{proof}

\section{Applications \& Extensions}
\subsection{Deriving the conclusion for every dataset}
\label{appen: every_data_set}
In Sections \ref{subsec: main_gdm}, \ref{subsec: main_sgdm}, and \ref{subsec: main_adam}, we only derive the implicit regularization for almost every dataset, but not all the separable datasets. In this section, we show that the "almost every" condition can be removed as the following theorem.

\begin{theorem}
\label{thm: every_dataset}
We have the following conclusions:
\begin{itemize}
    \item For GDM, let all the conditions in Theorem \ref{thm: gdm_main} hold. Then, GDM converges to the $L^2$ max-margin solution;
    \item For SGDM sampling with replacement, let all the conditions in Theorem \ref{thm: sgdm_main} hold. Then, SGDM (w/. r) converges to the $L^2$ max-margin solution (the same for SGDM (w/o. r), except a different learning rate upper bound);
    \item For deterministic Adam, let all the conditions in Theorem \ref{thm: adam_main} hold. Then, Adam (w/s)  converges to the $L^2$ max-margin solution;
    \item For RMSProp (w/o. r), let all the conditions in Theorem \ref{thm: ib_rms} hold. Then, RMSProp (w/o. r) converges to the $L^2$ max-margin solution.
\end{itemize}
\end{theorem}

It can be easily observe that to prove Theorem \ref{thm: every_dataset}, the analysis of Stage I of every optimizer still works. Therefore, we only need to change Stage II. As the analyses of Stage II are highly overlapped for different optimizers, we only provide a proof of GDM.

To begin with, we present some notations and results on the structure of the separable dataset from \cite{soudry2018implicit}. 

Let $\bar{\mS}_0=\{1,\cdots, N\}$, and $\bar{\bP}_0=\mathbb{I}_{d\times d}$. We then recursively define the index sets $\mS_m^{+}$, $\mS_{m}^{=}$, $\mS_{m}$, and $\bar{\mS}_{m}$:
\begin{gather*}
    \mS_{m}^{+}=\left\{i \in \bar{\mS}_{m-1} \mid \left\langle\hbw_{m}, \bar{\bP}_{m-1} \bx_i\right\rangle>1\right\},
    \\
      \mS_{m}^{=}=\left\{i \in \bar{\mS}_{m-1} \mid \left\langle\hbw_{m}, \bar{\bP}_{m-1} \bx_i\right\rangle=1\right\}=\bar{\mS}_{m-1}/  \mS_{m}^{+},
      \\
      \mS_{m}=\left\{i \in \mS_{m}^{=} \mid \exists \boldsymbol{\alpha} \in \mathbb{R}_{\geq 0}^{N}: \hbw_{m}=\sum_{k=1}^{N} \alpha_{k} \bar{\bP}_{m-1} \bx_{k}, \alpha_{i}>0, \forall j \notin \mathcal{S}_{m}^{=}: \alpha_{j}=0\right\},
      \\
      \bar{\mS}_{m}=\mS_{m}^{=}/\mS_{m}.
\end{gather*}
where $\bar{\mathbf{P}}_{m}=\bar{\mathbf{P}}_{m-1}\left(\mathbf{I}_{d}-\bS_{\mS_{m}} \bS_{\mS_{m}}^{\dagger}\right)$ (we also denote $\bP_m=\mathbb{I}_{d\times d}-\bar{\bP}_m$), and $\hbw_{m}$ is defined as the max-margin solution of dataset $ \bar{\bP}_{m-1} \bS_{\bar{\mS}_{m-1}}$ (that is, the transferred data $x_i$ with index in $\bar{\mS}_{m-1}$ projected through matrix $ \bar{\bP}_{m-1}$):
\begin{equation}
\label{eq: def_m_max_margin_solution}
    \hbw_{m}=\underset{\bw \in \mathbb{R}^{d}}{\operatorname{argmin}}\|\bw\|^{2}, \text { s.t. } \langle \bw, \bar{\bP}_{m-1} \bx_{i}\rangle \geq 1, \forall i \in \bar{\mS}_{m-1} .
\end{equation}

The existence of the $\boldsymbol{\alpha}$ is guaranteed by the KKT condition of Eq. (\ref{eq: def_m_max_margin_solution}). The above procedure will produce a sequence $\hbw_1$, $\hbw_2$, $\cdots$, and will stop at $\hbw_M$ if $\bar{\mS}_M$ is empty (if the sequence is infinite, we let $M=\infty$). For every $i\le M$, we have $\hbw_i$ is non-zero, and $\mS_i$ is non-empty, which leads to $\vert \bar{\mS}_{i} \vert <\vert \bar{\mS}_{i+1} \vert$, and $M\le N$.
 
The following lemma characterize the structure of the dataset.
\begin{lemma}[Lemma 17, \cite{soudry2018implicit}]
\label{lem: soudry_every_dataset_1}
$\forall \boldsymbol{\beta}\in \mathbb{R}^{\vert \bS_1 \vert }_{>0}$, we can find a unique $\tbw_1$, such that 
\begin{equation*}
    \sum_{i \in \mS_{1}} \bx_{i} \beta_{i} \exp \left(-\langle\bx_i, \tbw_{1}\rangle\right)=\hbw_{1},
\end{equation*}
and $\tbw_1\in\operatorname{Col} (\bS_{\mS_1})$.
\end{lemma}

With Lemma \ref{lem: soudry_every_dataset_1}, we then define $\tbw_m$ as the solution of 
\begin{equation*}
 \frac{\eta}{1-\beta} \sum_{i \in \mathcal{S}_{m}} \exp \left(-\sum_{k=1}^{m} 
 \langle\tbw_{k}, \bx_{i}\rangle\right) \bar{\bP}_{m-1} \bx_i=\hbw_{m},
\end{equation*}
with $\bP_{m-1}\tbw_{m}=0$ and $\bar{\bP}_{m}\tbw_{m}=0$. We also define $\tbw=\sum_{m=1}^M \tbw_m$.

We then have the following lemma:
\begin{lemma}[Lemma 18, \cite{soudry2018implicit}] $\forall m>k\ge 1$, the equations
\begin{equation*}
\sum_{i \in \mS_{m}} \exp \left(-\langle\tbw, \bx_{i}\rangle\right) \bP_{m-1} \bx_{i}=\sum_{k=1}^{m-1}\left[\sum_{i \in \mathcal{S}_{k}} \exp \left(-\langle \tbw ,\bx_i\right) \bx_i \bx_i^{\top}\right] \check{\bw}_{k, m}\end{equation*}
under the constraints $\bP_{k-1}\check{\bw}_{k, m}=0 $ and $\bar{\bP}_{k}\check{\bw}_{k, m}=0 $ have the unique solution $\check{\bw}_{k, m}$.
\end{lemma}

We then denote
\begin{equation*}
    \br(t)=\bw(t)-\hbw_{1} \log (t)-\tbw-\tau(t),
\end{equation*}
where $\tau(t)=\left(\sum_{m=2}^{M} \hbw_{m} \log ^{\circ m}(t)+\sum_{m=1}^{M} \sum_{k=1}^{m-1} \frac{\check{\bw}_{k, m}}{\prod_{r=k}^{m-1} \log ^{\circ r}(t)}\right)$

Similar to Eq. \ref{eq: definition_of_g}, we define
\begin{equation*}
    g(t)\overset{\triangle}{=}\frac{1}{2}\Vert \br(t)\Vert^2+\frac{\beta}{1-\beta}\langle \br(t),\bw(t)-\bw(t-1)\rangle
    -\frac{\beta}{1-\beta}\sum_{\tau=2}^t\langle \br(\tau)-\br(\tau-1),\bw(\tau)-\bw(\tau-1)\rangle.
\end{equation*}

We then have the following lemma parallel to Lemma \ref{lem:gdm_r_t_bounded}:
\begin{lemma}
Let all conditions in Theorem \ref{thm: gdm_main} hold. 
We have $\sup_t g(t)$ is finite.  Furthermore, $\sup_t\Vert\br(t)\Vert$ is finite if and only if $\sup_t g(t)$ is finite, and consequently $\sup_t\Vert\br(t)\Vert$ is finite.
\end{lemma}
\begin{proof}
The proof of the second argument follows the same routine as the proof of the first argument in Lemma \ref{lem:gdm_r_t_bounded}, and we omit it here.

As for the first argument, we have 
\begin{small}
\begin{align*}
    \sum_{t=1}^{\infty}(g(t+1)-g(t))=&\frac{1}{2} \sum_{t=1}^{\infty}\Vert \br(t+1)-\br(t) \Vert^2
    \\
    +&\sum_{t=1}^{\infty}\left\langle \br(t), - \frac{\eta}{1-\beta}\nabla \mL(\bw(t))-\ln\frac{t+1}{t} \hbw_1-(\tau(t+1)-\tau(t))\right\rangle,
\end{align*}
\end{small}
where the first term can be shown to be finite similar to Lemma \ref{lem:gdm_r_t_bounded}, while the second term is finite by Lemma 14 in \cite{soudry2018implicit}.

The proof is completed.
\end{proof}

\subsection{Implicit regularization of SGDM (w/o. r)}
\label{appen: mini_SGDM}
This section provides formal description of the implicit regularization of  SGDM (w/o. r) and its corresponding proof. To begin with, we would like to provide a formal definition of SGDM (w/o. r).  SGDM (w/o. r) differs from SGDM by applying sampling without replacement to obtain $\bB(t)$ in Eq. (\ref{eq: def_sgdm}). Specifically, let $K=\frac{N}{b}$. For any $T\ge 0$, we call time series $\{KT+1,\cdots,KT+K\}$ the $(T+1)$-th epoch, and during the $T+1$-th epoch, the dataset $\bS$ is randomly uniformly divided into $K$ parts $\{\bB(KT+1),\cdots,\bB(KT+K)\}$, with $\bigcup_{t=KT+1}^{KT+K} \bB(t)=\bS$. The implicit regularization of SGDM (w/o. r) is then stated as the following theorem:
\begin{theorem}
\label{thm: main_mini_batch_sgd}
Let Assumptions  \ref{assum: separable}, \ref{assum: exponential-tailed}, and \ref{assum: smooth}. (S) hold. Let learning rate $\eta$ be small enough, and $\beta\in [0,1)$.Then, for almost every dataset $\bS$, SGDM (w/o. r) satisfies  $\bw(t)-\ln(t) \hbw$ is bounded as $t\rightarrow \infty$, and $\lim_{t\rightarrow\infty}\frac{\bw(t)}{\Vert \bw(t) \Vert} =\frac{\hbw}{\Vert \hbw \Vert}$.
\end{theorem}
The without-replacement sampling method leads to the direction of every trajectory of mini-SGDM converge to the max-margin solution, compared to the same conclusion holds for SGDM a.s.. We prove the theorem following the same framework of GDM, by proceeding with two stages.

\textbf{Stage I.} The following lemma proves $\mL(\bu(t))$ is an Lyapunov function for SGDM (w/o. r) and without the a.s. condition.
\begin{lemma}
\label{lem: loss_iteration_mini_batch_sgld}
Let all conditions in Theorem \ref{thm: main_mini_batch_sgd} hold. Then, we have
\begin{equation*}
    \mL(\bu(t+1))\le \mL(\bu(1))-\Omega(\eta) \sum_{s=1}^t\Vert \nabla \mL(\bw(s)) \Vert^2. 
\end{equation*}
\end{lemma}

\begin{proof}
    By the Taylor Expansion of $\mL(\bu(t+1))$ at $\bu(t)$, we have
    \begin{align}
    \nonumber
        &\mL(\bu(KT+T+1))
        \\
    \nonumber
        \le& \mL(\bu(KT+1))-\eta \left\langle\nabla \mL(\bu(KT+1)),\sum_{t=1}^{K} \nabla\mL_{\bB(t+KT)}(\bw(t+KT)) \right\rangle
        \\
    \label{eq: expansion_mini}
        +&\frac{H\eta^2}{2}\left\Vert \sum_{t=1}^{K} \mL_{\bB(t+KT)}(\bw(t+KT)) \right\Vert^2.
    \end{align}
    
    On the other hand, for any $t\in\{2,\cdots,K\}$, we have
    \begin{align*}
        &\bw(KT+t)-\bw(KT+1)
        =\eta\sum_{s=1}^t\left(\sum_{\ell=1}^{KT+s}\beta^{KT+s-\ell}\nabla \mL_{\bB(\ell)}(\bw(\ell))\right)
        \\
        =&\eta\sum_{s=1}^t\left(\sum_{\ell=KT+1}^{KT+s}\beta^{KT+s-\ell}\nabla \mL_{\bB(\ell)}(\bw(\ell))\right)+\eta\sum_{s=1}^t\left(\sum_{\ell=1}^{KT}\beta^{KT+s-\ell}\nabla \mL_{\bB(\ell)}(\bw(\ell))\right)
        \\
        =&\eta \sum_{\ell=1}^t \frac{1-\beta^{t-\ell+1}}{1-\beta} \nabla \mL_{\bB(KT+\ell)}(\bw(KT+\ell))+\eta \frac{\beta(1-\beta^t)}{1-\beta}\sum_{\ell=1}^{KT} \beta^{KT-\ell} \nabla \mL_{\bB(\ell)} (\bw(\ell))
        \\
        =&\eta \sum_{\ell=1}^t \frac{1-\beta^{t-\ell+1}}{1-\beta} \nabla \mL_{\bB(KT+\ell)}(\bw(KT+\ell))-\eta \sum_{\ell=1}^t \frac{1-\beta^{t-\ell+1}}{1-\beta} \nabla \mL_{\bB(KT+\ell)}(\bw(KT+1))
        \\
        +&\eta \frac{\beta(1-\beta^t)}{1-\beta}\sum_{\ell=1}^{KT} \beta^{KT-\ell} \nabla \mL_{\bB(\ell)} (\bw(\ell))+\eta \sum_{\ell=1}^t \frac{1-\beta^{t-\ell+1}}{1-\beta} \nabla \mL_{\bB(KT+\ell)}(\bw(KT+1)),
    \end{align*}
    which by $\eta$ is small enough further indicates
    \begin{align*}
        &\Vert \bw(KT+t)-\bw(KT+1)\Vert
        \\
        \le &\eta\left\Vert  \sum_{\ell=1}^t \frac{1-\beta^{t-\ell+1}}{1-\beta} \nabla \mL_{\bB(KT+\ell)}(\bw(KT+\ell))- \sum_{\ell=1}^t \frac{1-\beta^{t-\ell+1}}{1-\beta} \nabla \mL_{\bB(KT+\ell)}(\bw(KT+1))\right\Vert
        \\
        +&\eta\left\Vert \frac{\beta(1-\beta^t)}{1-\beta}\sum_{\ell=1}^{KT} \beta^{KT-\ell} \nabla \mL_{\bB(\ell)} (\bw(\ell))\right\Vert+\eta\left\Vert \sum_{\ell=1}^t \frac{1-\beta^{t-\ell+1}}{1-\beta} \nabla \mL_{\bB(KT+\ell)}(\bw(KT+1))\right\Vert
        \\
        = & \mathcal{O}(\eta) \sum_{\ell=2}^t \Vert \bw(KT+\ell)-\bw(KT+1)\Vert+\mathcal{O}(\eta)\left(\sum_{\ell=1}^{KT} \beta^{KT-\ell}\left\Vert \nabla \mL_{\bB(\ell)} (\bw(\ell))\right\Vert\right)
        \\
        &+\mathcal{O}(\eta)\left\Vert\nabla \mL(\bw(KT+1)) \right\Vert.
    \end{align*}
    Applying the same analysis to $\Vert \bw(KT+t-1)-\bw(KT+1)\Vert$ recursively, we finally obtain  
    \begin{align}
    \nonumber
        &\Vert \bw(KT+t)-\bw(KT+1)\Vert
        \\
    \label{eq: iteration_mini_batch}
        \le &\mathcal{O}(\eta)\left(\sum_{\ell=1}^{KT} \beta^{KT-\ell}\left\Vert \nabla \mL_{\bB(\ell)} (\bw(\ell))\right\Vert\right)+\mathcal{O}(\eta)\left\Vert\nabla \mL(\bw(KT+1)) \right\Vert.
    \end{align}
    
    Applying Eq. (\ref{eq: iteration_mini_batch}) to  the $\left\Vert \nabla \mL_{\bB(\ell)} (\bw(\ell))\right\Vert$ in Eq. (\ref{eq: iteration_mini_batch}) ($\forall \ell \in [1,KT]$) iterative and choosing $\eta$ to be small enough, we further have 
    \begin{align}
    \nonumber
        &\Vert \bw(KT+t)-\bw(KT+1)\Vert
        \\
        \nonumber
        \le &\mathcal{O}(\eta)\left(\sum_{\ell=0}^{T-1} \sqrt{\beta^{K(T-\ell)}}\left\Vert \nabla \mL_{\bB(K\ell+1)} (\bw(K\ell+1))\right\Vert\right)+\mathcal{O}(\eta)\left\Vert\nabla \mL(\bw(KT+1)) \right\Vert
        \\
        \nonumber
        =&\mathcal{O}(\eta)\left(\sum_{\ell=0}^{T} \sqrt{\beta^{K(T-\ell)}}\left\Vert \nabla \mL_{\bB(K\ell+1)} (\bw(K\ell+1))\right\Vert\right).
    \end{align}
    
    Therefore, 
    \begin{align}
        \nonumber
        &\sum_{t=1}^{K} \nabla\mL_{\bB(t+KT)}(\bw(t+KT))
        \\
        \nonumber
        =&\sum_{t=1}^{K} \nabla\mL_{\bB(t+KT)}(\bw(t))+\mathcal{O}\left(\eta\left(\sum_{\ell=0}^{T} \sqrt{\beta^{K(T-\ell)}}\left\Vert \nabla \mL_{\bB(K\ell+1)} (\bw(K\ell+1))\right\Vert\right)\right)
        \\
        \label{eq: mini_batch_change_w}
        =&K\nabla\mL(\bw(t))+\mathcal{O}\left(\eta\left(\sum_{\ell=0}^{T} \sqrt{\beta^{K(T-\ell)}}\left\Vert \nabla \mL_{\bB(K\ell+1)} (\bw(K\ell+1))\right\Vert\right)\right).
    \end{align}
    
    Similarly, one can obtain
    \begin{align}
    \nonumber
        &\nabla\mL(\bu(KT+1))
        \\
        \nonumber
        =&\nabla\mL(\bw(KT+1))+\mathcal{O}\left(\Vert\bw(KT+1)-\bw(KT)\Vert\right)
        \\
        \label{eq: mini_batch_change_u}
        =&\nabla\mL(\bw(KT+1))+\mathcal{O}\left(\eta\left(\sum_{\ell=0}^{T} \sqrt{\beta^{K(T-\ell)}}\left\Vert \nabla \mL_{\bB(K\ell+1)} (\bw(K\ell+1))\right\Vert\right)\right).
    \end{align}
    
    Applying Eq. (\ref{eq: mini_batch_change_w}) and Eq. (\ref{eq: mini_batch_change_u}) back to the Taylor Expansion (Eq. (\ref{eq: expansion_mini})), we have
    \begin{align*}
       & \mL(\bu(KT+T+1))
       \\
       \le& \mL(\bu(KT+1))-\Omega(\eta) \left\langle \nabla \mL(\bw(KT+1)),\nabla \mL(\bw(KT+1))\right\rangle
       \\
       &+\mathcal{O}\left(\eta^2\left(\sum_{\ell=0}^{T} \sqrt{\beta^{K(T-\ell)}}\left\Vert \nabla \mL_{\bB(K\ell+1)} (\bw(K\ell+1))\right\Vert\right)^2\right)
       \\
       \le &\mL(\bu(KT+1))-\Omega(\eta) \left\langle \nabla \mL(\bw(KT+1)),\nabla \mL(\bw(KT+1))\right\rangle
       \\
       &+\mathcal{O}\left(\eta^2\left(\sum_{\ell=0}^{T} \sqrt{\beta^{K(T-\ell)}}\left\Vert \nabla \mL_{\bB(K\ell+1)} (\bw(K\ell+1))\right\Vert^2\right)\right).
    \end{align*}
    
    Summing the above inequality over $T$ and setting $\eta$ small enough leads to the conclusion.
    
    The proof is completed.
\end{proof}

\subsection{Extension to the multi-class classification problem}
\label{appen: multi-class}
{  
As mentioned in Section \ref{sec: preliminary}, despite all the previous analyses are aimed at the binary classification problem, they can be naturally extended to the analyses multi-class classification problem. Specifically, in the linear multi-class classification problem, for any $(\bx,\by)\in \mathbb{R}^{d_X}\times \{1,\cdots,C\}$ in the sample set $\bS$, the (individual) logistic loss with parameter $\bW\in \mathbb{R}^{C\times d_X}$ is denoted as 
\begin{equation*}
    \ell(\by,\bW \bx)=\ln \frac{e^{\bW_{\by,}\bx}}{\sum_{i=1}^C e^{\bW_{i,}\bx}}.
\end{equation*}
Correspondingly, dataset $\bS$ is separable if there exists a parameter $\bW$, such that $\forall (\bx,\by)\in \bS$, we have $\bW_{\by,}\bx> \bW_{i,}\bx$, $\forall i\ne \by$. The multi-class $L^2$ max-margin problem is then defined as
\begin{equation*}
    \min \Vert \bW\Vert_{F},~ subject~to:~\bW_{\by,}\bx\ge \bW_{i,}\bx+1, \forall (\bx,\by)\in\bS, i\ne \by,
\end{equation*}where $\Vert\cdot\Vert_{F}$ denotes the Frobenius norm.
Denote $\hbW$ as the $L^2$ max-margin solution, we have SGDM  and Adam {   (w/s)} still converges to the direction of $\hbW$.
\begin{theorem}
\label{thm: multi}
For linear multi-class classification problem using logistic loss and almost every separable data, with a small enough learning rate, and $1>\beta_2>\beta_1^4\ge 0$ (for Adam {   (w/s)}),  SGDM and Adam {   (w/s)} converge to the multi-class $L^2$ max-margin solution (a.s. for SGDM SGDM (w/. r)).
\end{theorem}
}
Here we use several notations and lemmas from \cite{soudry2018implicit}. We define $\bw=\vect(\bW)$, $\hbw=\vect(\hbW)$, $\be_i\in \mR^{C}$ ($i\in\{1,\cdots,C\}$) satisfying $(\be_i)_j=\delta_{ij}$, and $\bA_i=\be_i\otimes \mI_{d_{X}}$, where $ \mI_{d_{X}}$ is the identity matrix with dimension $d_{X}$. We still consider the normalized data, i.e., $\Vert \bx\Vert \le 1$, $\forall (\bx,\by)\in \bS$. Then, the individual loss of sample $(\bx,\by)$ can be then represented as
\begin{equation*}
    \ell(\by,\bW\bx)=\ln\frac{e^{\langle\bw, \bA_{\by} \bx\rangle}}{\sum_{i=1}^C e^{\langle \bw, \bA_{i}\bx\rangle}}.
\end{equation*}
Furthermore, the gradient of training error at $\bW$ has the form
\begin{equation*}
    \nabla \mL(\bw)=\frac{1}{N}\sum_{(\bx,\by)\in\bS} \sum_{i=1}^C\frac{1}{\sum_{j=1}^C e^{\langle \bw, (\bA_{j}-\bA_{i})\bx\rangle}}(\bA_i-\bA_{\by})\bx.
\end{equation*}
and the Hessian matrix of $\mL$ can be represented as
\begin{equation*}
    \mathcal{H} \mL(\bw)=\frac{1}{N}\sum_{(\bx,\by)\in\bS} \sum_{i=1}^C\frac{\sum_{j=1}^C e^{\langle \bw, (\bA_{j}-\bA_{i})\bx\rangle}}{\left(\sum_{j=1}^C e^{\langle \bw, (\bA_{j}-\bA_{i})\bx\rangle}\right)^2}(\bA_i-\bA_{\by})\bx((\bA_j-\bA_{i})\bx)^{\top},
\end{equation*}
one can then easily verify all absolute value of the eigenvalues of $\mathcal{H}\mL(\bw)$ is no larger than $2$, which indicates $\mL$ is $2$-globally smooth.

On the other hand, the separable assumption leads to $\langle\hbw,(\bA_{\by}-\bA_i)\bx\rangle>0$, $\forall \by\ne i$, which further indicates 
\begin{equation*}
    \langle \nabla \mL(\bw),\hbw\rangle>0.
\end{equation*}

Let $\gamma=\frac{1}{\Vert \hbw\Vert}$, following the similar routine as the binary case, we have for a random subset of $\bS$ sampled uniformly without replacement with size $b$, we have 
\begin{equation}
\label{eq: multi_momentum}
     \Vert\nabla \mL(\bw) \Vert^2\le \mathbb{E}_{\bB(t)} \Vert\nabla \mL_{\bB(t)}(\bw) \Vert^2\le \frac{2N}{\gamma b^2} \Vert\nabla \mL(\bw) \Vert^2.
\end{equation}
Similarly, we have for any positive real series $\{a_t\}_{t=t_1}^{t_2}$,
\begin{equation}
\label{eq: multi_sum}
    \gamma \sum_{t=t_1}^{t_2} a(t) \Vert\nabla \mL(\bw(t)) \Vert\le \left\Vert \sum_{t=t_1}^{t_2} a(t)  \nabla \mL(\bw(t)) \right\Vert \le  \sum_{t=t_1}^{t_2}a(t)\Vert   \nabla \mL(\bw(t)) \Vert.
\end{equation}

The proofs of Stage I can then be obtained with Lyapunov functions unchanged and by replacing the corresponding lemmas using Eq. (\ref{eq: multi_momentum}) and Eq. (\ref{eq: multi_sum}).

As for the proofs of Stage II, the Lyapunov functions are still the same, while we only need to prove the sum of $\langle \br(t), -\frac{\eta}{1-\beta}\nabla \mL(\bw(t))-\ln\frac{t+1}{t}\hbw\rangle$ (for GDM, $\langle \br(t), -\frac{\eta}{1-\beta}\nabla \mL_{\bB(t)} (\bw(t))-\frac{N}{bt}\sum_{i: \bx_i\in \bB(t)\cap \bS_s} \bv_i\bx_i\rangle$ for SGDM, $\langle \br(t),-\sqrt{\varepsilon}\ln\left(\frac{t+1}{t}\right)\hbw-\frac{\eta}{1-\beta}\nabla \mL(\bw(t))\rangle$ for Adam). For the multi-class case using GDM, We present the following lemma from \cite{soudry2018implicit}, while the other two cases can be proved similarly:
\begin{lemma}[Part of the proof of Lemma 20, \cite{soudry2018implicit}] If $\langle\bw(t),(\bA_{\by}-\bA_i)\bx\rangle\rightarrow \infty$ as $t\rightarrow \infty$,  $\forall (\bx,\by)\in \bS$ and $\forall i\ne \by$, we have the sum of $\langle \br(t), -\frac{\eta}{1-\beta}\nabla \mL(\bw(t))-\ln\frac{t+1}{t}\hbw\rangle$ is upper bounded.
\end{lemma}
The proof of Theorem \ref{thm: multi} is then completed.

\subsection{Precisely characterize the convergence rate}
\label{appen: precise}
Theorem \ref{thm: gdm_main} and Theorem \ref{thm: sgdm_main} can be further extended to precisely characterize the asymptotic convergence rate for (S)GDM as the following theorem.

\begin{theorem}
\label{thm: precise}
Let all the conditions in Theorem \ref{thm: gdm_main} (Theorem \ref{thm: sgdm_main}) hold. Assume the linear span of support vectors contains the whole dataset.  Then, we have 
\begin{equation*}
    \lim_{t\rightarrow \infty} t\mL(\bw(t))=C\frac{1}{\eta},
\end{equation*}
where $C$ is a constant independent of learning rate $\eta$, momentum hyperparameter $\beta$, and mini-batch size $b$.
\end{theorem}
The proof follows exactly the same routine as (Corollary 1, \cite{nacson2019convergence}) and we omit it here.

% {   
% \subsection{Adaptive Heavy Ball Algorithm}
% \label{appen: AHB}
% Adaptive Heavy Ball Algorithm is an adaptive variant of the Heavy Ball algorithm, which is proposed by \cite{tao2021role} and has the following form\footnote{Compared to \cite{tao2021role}, we put the $\varepsilon \mathds{1}_d$ inside the square-root and use $\hbnu(t)$ instead of $\bnu(t)$ to demonstrate the difference between Adaptive Heavy Ball and Adam, while these changes will not influence the convergent direction.}:
% \begin{equation*}
%     \bw(t+1)=\bw(t)-\eta\frac{\nabla \mL_{\bB(t)}(\bw(t))}{\sqrt{\varepsilon \mathds{1}_d+\hbnu (t)}}+\beta(\bw(t)-\bw(t-1)), t\ge 1,
% \end{equation*}where $\hbnu(t)$ is defined as Eq. (\ref{eq: def_adam}).The proof of the convergent direction of the Adaptive Heavy Ball algorithm is a simple combination of the proof techniques in SGDM and Adam by observing that the stochastic adaptive heavy-ball algorithm has the following equivalent update rule ($\bu(1)=\bw(1)$)
% {\begin{gather}
% \nonumber
%     \bu(t+1)=-\frac{\eta}{1-\beta}\frac{\nabla \mL_{\bB(t)}(\bw(t))}{\sqrt{\varepsilon \mathds{1}_d+\hbnu (t)}}+\bu(t),
%     \\
% %\label{eq: alter_sgdm}
%     \bw(t+1)=\beta \bw(t)+(1-\beta)\bu(t+1).\nonumber
% \end{gather}}}

\section{Experiments}
\label{appen: experiment}
This section collects several experiments supporting our theoretical results.
\subsection{Experiments on linear model}
\subsubsection{Comparing the training behavior of (S)GD, (S)GDM, deterministic Adam and RMSProp}
The experiments in this section is designed to verify our theoretical results, i.e., Theorems \ref{thm: gdm_main}, \ref{thm: sgdm_main}, \ref{thm: adam_main}, and \ref{thm: ib_rms}. We use the synthetic dataset in (Figure 1, \cite{soudry2018implicit}) and logistic loss, and run (S)GD, (S)GDM, deterministic Adam and stochastic RMSProp over it with different learning rates $\eta=0.1, 0.001$ and different random seeds (for random initialization, random samples despite the support sets $\{((1.5,0.5),1),((0.5,1.5),1), ((-1.5,-0.5),-1),((-1.5,-0.5),-1)\}$), and random mini-batches. Both the angle between the output parameter and max-margin solution and the training accuracy are plotted in Figure \ref{fig: deter_experi_1}. The observations can be summarized as follows:
\begin{itemize}
    \item With proper learning rates, all of (S)GD, (S)GDM, deterministic Adam and stochastic RMSProp converge to the max-margin solution, which supports our theoretical results;
    
      \item (Similarity between (S)GD and (S)GDM).  The asymptotic training behaviors of GD, SGD, GDM and SGDM are highly similar, which supports our Theorem \ref{thm: precise}.
    
    \item (The acceleration effect of Adam). Deterministic Adam and stochastic RMSProp achieve smaller angle with the max-margin solution under the same number of iterations.

\end{itemize}

\begin{figure}[htbp]
\centering
\vspace{-1mm}

\begin{minipage}{1.0\textwidth}
\centering
\begin{subfigure}{.45\textwidth}
          \centering
          \includegraphics[width=1.0\textwidth]{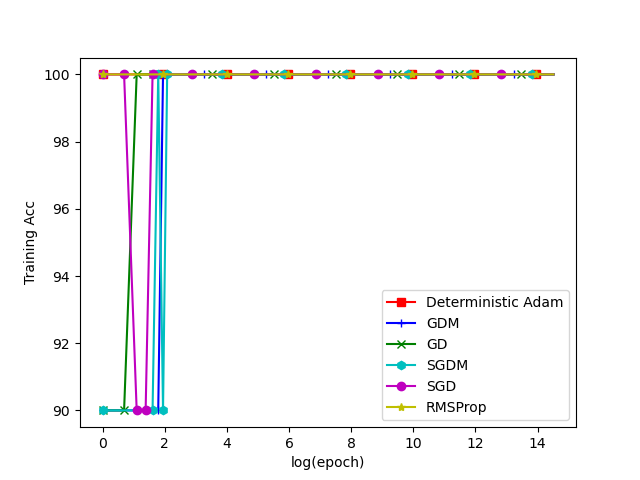}
          \caption{Accuracy: $\eta=0.1$, random seed $=1$}
        \end{subfigure}%
        \hspace{1mm}
\begin{subfigure}{.45\textwidth}
          \centering
          \includegraphics[width=1.0\textwidth]{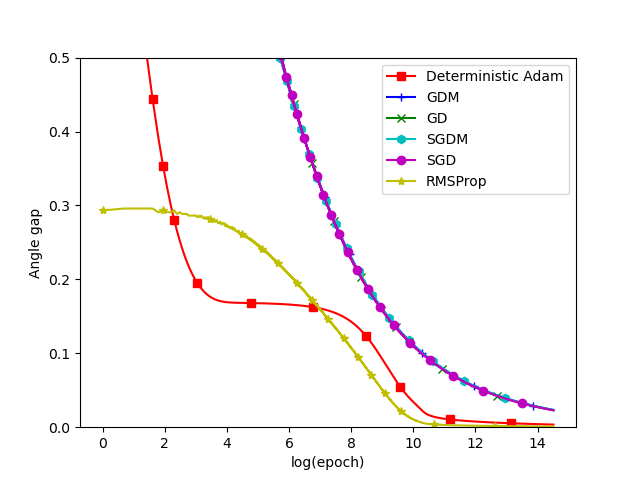}
          \caption{Angle: $\eta=0.1$, random seed $=1$}
\end{subfigure}%
\end{minipage}%

\begin{minipage}{1.0\textwidth}
\centering
\begin{subfigure}{.45\textwidth}
          \centering
          \includegraphics[width=1.0\textwidth]{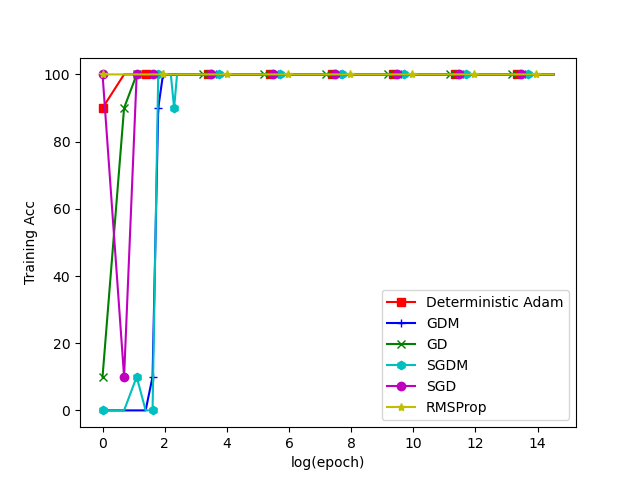}
          \caption{Accuracy: $\eta=0.1$, random seed $=2$}
        \end{subfigure}%
        \hspace{1mm}
\begin{subfigure}{.45\textwidth}
          \centering
          \includegraphics[width=1.0\textwidth]{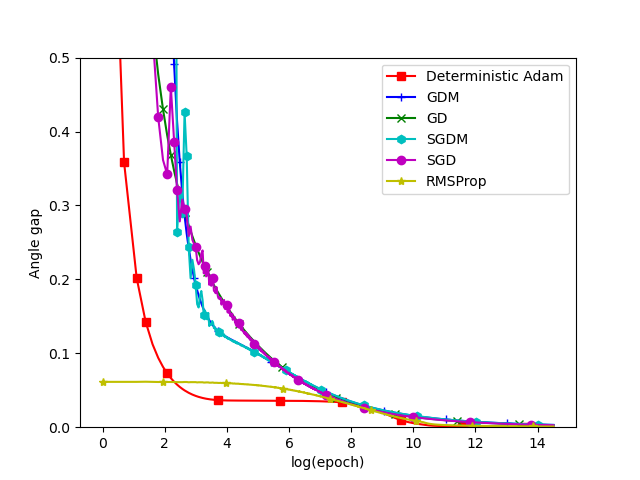}
          \caption{Angle: $\eta=0.1$, random seed $=2$}
\end{subfigure}%
\end{minipage}%

\begin{minipage}{1.0\textwidth}
\centering
\begin{subfigure}{.45\textwidth}
          \centering
          \includegraphics[width=1.0\textwidth]{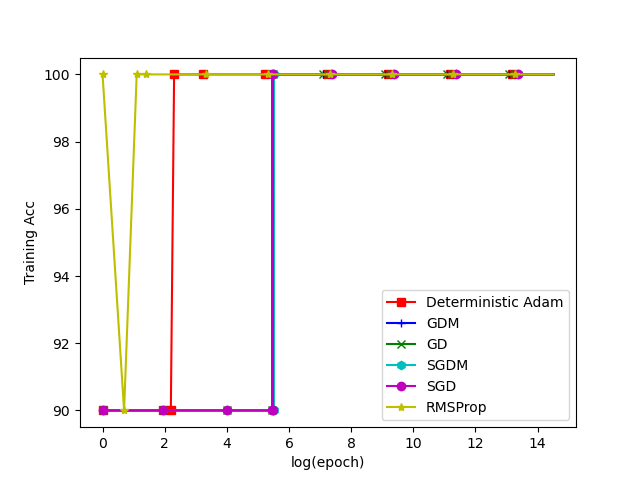}
          \caption{Accuracy: $\eta=0.001$, random seed $=1$}
        \end{subfigure}%
        \hspace{1mm}
\begin{subfigure}{.45\textwidth}
          \centering
          \includegraphics[width=1.0\textwidth]{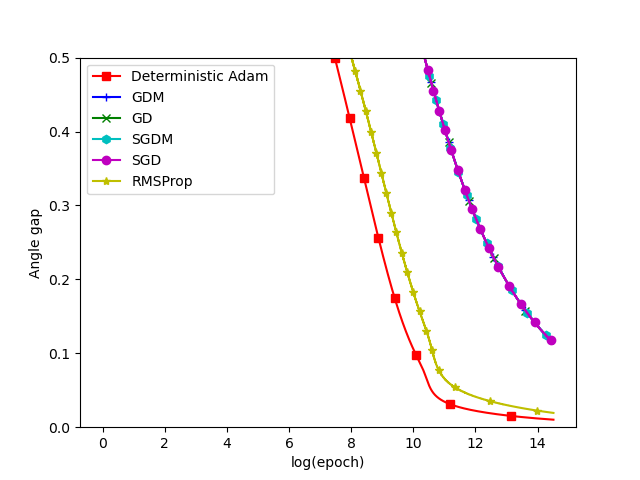}
          \caption{Angle: $\eta=0.001$, random seed $=1$}
\end{subfigure}%
\end{minipage}%

\begin{minipage}{1.0\textwidth}
\centering
\begin{subfigure}{.45\textwidth}
          \centering
          \includegraphics[width=1.0\textwidth]{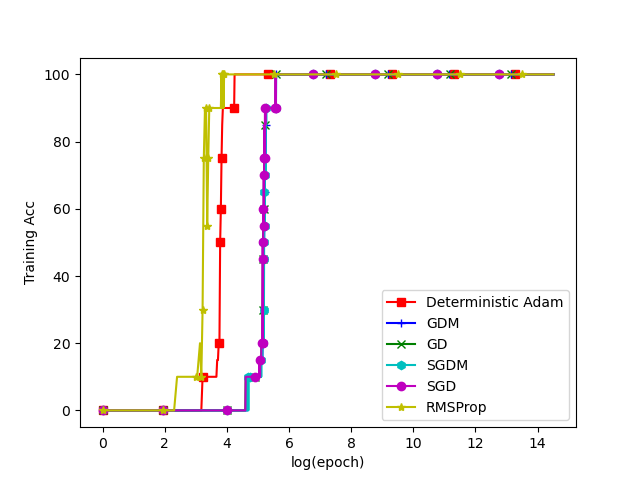}
          \caption{Accuracy: $\eta=0.001$, random seed $=2$}
        \end{subfigure}%
        \hspace{1mm}
\begin{subfigure}{.45\textwidth}
          \centering
          \includegraphics[width=1.0\textwidth]{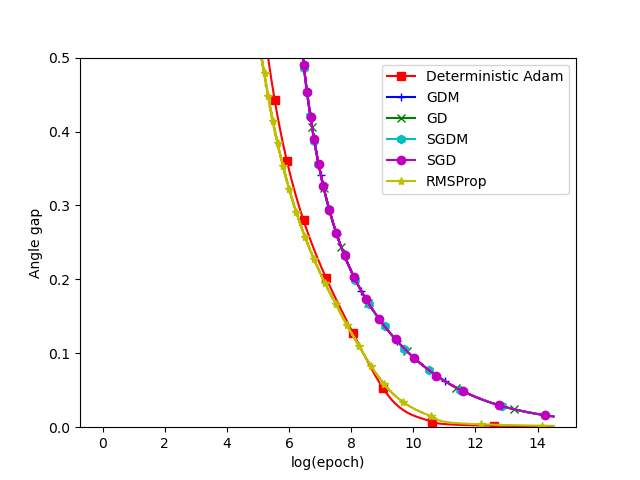}
          \caption{Angle: $\eta=0.001$, random seed $=2$}
\end{subfigure}%
\end{minipage}%
\centering
\caption{Comparison of (S)GD, (S)GDM, deterministic Adam, and stochastic Adam on the synthetic dataset in \cite{soudry2018implicit}.}
\label{fig: deter_experi_1}
\end{figure}

\subsubsection{Adam on ill-posed dataset}
\label{appen:ill}
In Figure 3 of \cite{soudry2018implicit}, an ill-posed synthetic dataset is proposed to support the argument "Adam does not converge to max-margin solution", which contradicts to the theoretical results of this paper. We re-conduct the experiment of Figure 3 in \cite{soudry2018implicit} with the same ill-posed synthetic dataset with different learning rates and different random seeds as Figure \ref{fig: ill_experi_1}. Figure \ref{fig: ill_experi_1}. (f) is  similar to Figure 3 in \cite{soudry2018implicit}, where with learning rate $\eta=0.1$ and random seed $1$, the angle of GD to the max-margin solution is smaller than Adam all the time. However, it can be observed from the amplified figure that the angle of GD keeps still and above $0$ all the time, meaning that GD doesn't converge to the max-margin solution under this setting. However, the angle of Adam to the max-margin solution still keeps decreasing and it's unreasonable to claim "Adam doesn't converge to the max-margin solution"  in this case (the same issue exists in Figure 3 in \cite{soudry2018implicit}). Also, as we mentioned at the beginning of this section, this dataset is ill-posed, which is due to the imbalance between the two components of the data (for all data $((x_1,x_2),y)$ in the dataset, $\vert x_1\vert $ is always smaller than $2$, while $\vert x_2\vert$ is larger than $10$ (and even larger than $30$ despite two data in the dataset)), which requires smaller learning rate. To tackle this problem, we need to tune down the learning rate. By Figure \ref{fig: ill_experi_1}. (b) and (d), after scaling down the learning rate, both GD's angle and Adam's angle keep decreasing.

\begin{figure}[htbp]
\centering
\vspace{-1mm}
\begin{minipage}{1.0\textwidth}
\centering
\begin{subfigure}{.45\textwidth}
          \centering
          \includegraphics[width=1.0\textwidth]{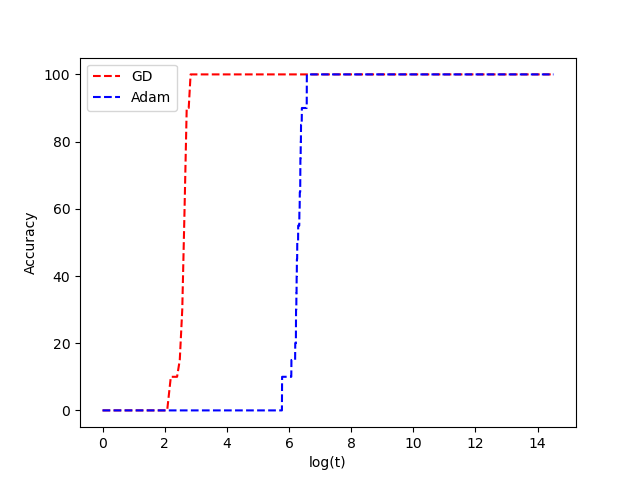}
          \caption{Comparison of Accuracy: $\eta=0.001$, random seed $=1$}
        \end{subfigure}%
        \hspace{1mm}
\begin{subfigure}{.45\textwidth}
          \centering
          \includegraphics[width=1.0\textwidth]{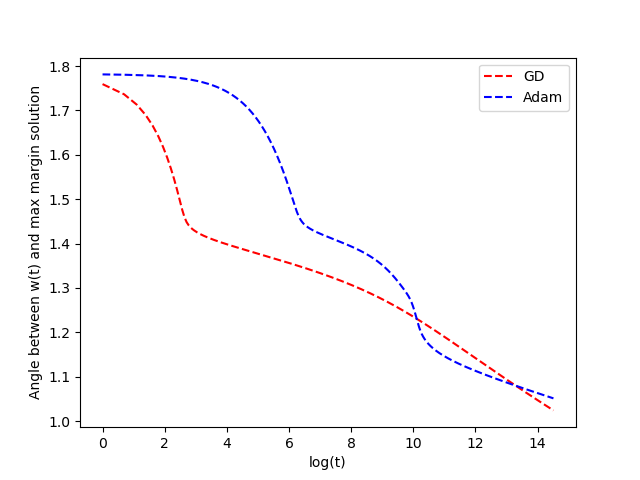}
          \caption{Comparison of Angle: $\eta=0.001$, random seed $=1$}
\end{subfigure}%
\end{minipage}%

\begin{minipage}{1.0\textwidth}
\centering
\begin{subfigure}{.45\textwidth}
          \centering
          \includegraphics[width=1.0\textwidth]{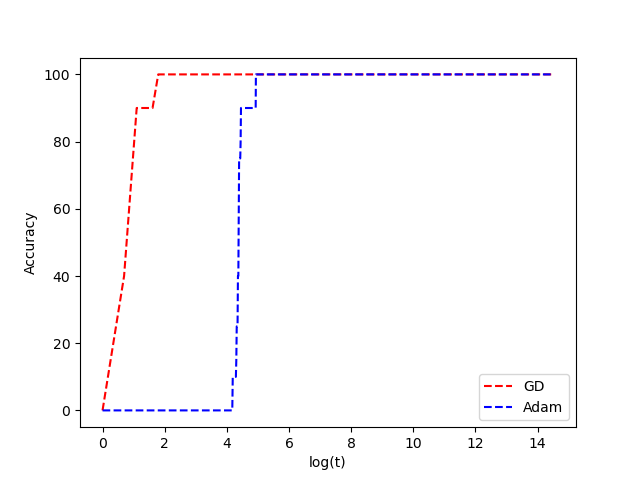}
          \caption{Comparison of Accuracy: $\eta=0.001$, random seed $=2$}
        \end{subfigure}%
        \hspace{1mm}
\begin{subfigure}{.45\textwidth}
          \centering
          \includegraphics[width=1.0\textwidth]{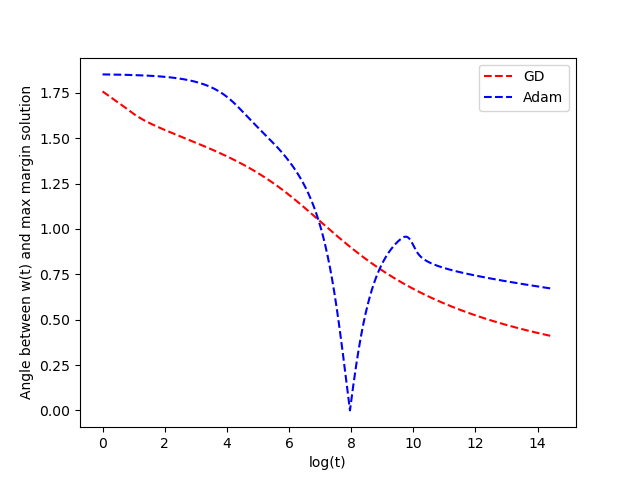}
          \caption{Comparison of Angle: $\eta=0.001$, random seed $=2$}
\end{subfigure}%
\end{minipage}%

\begin{minipage}{1.0\textwidth}
\centering
\begin{subfigure}{.45\textwidth}
          \centering
          \includegraphics[width=1.0\textwidth]{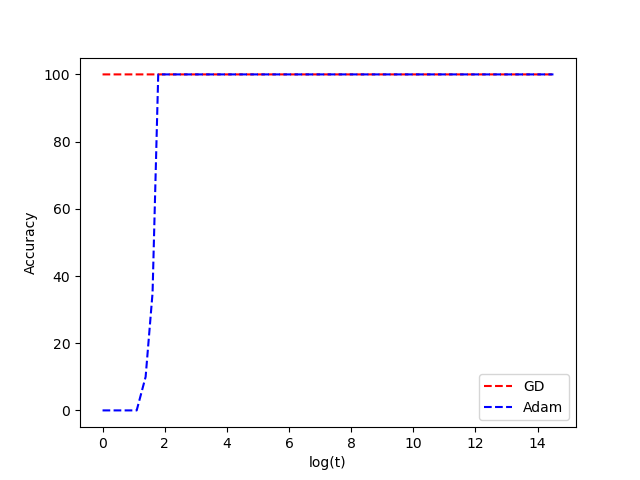}
          \caption{Comparison of Accuracy: $\eta=0.1$, random seed $=1$}
        \end{subfigure}%
        \hspace{1mm}
\begin{subfigure}{.45\textwidth}
          \centering
          \includegraphics[width=1.0\textwidth]{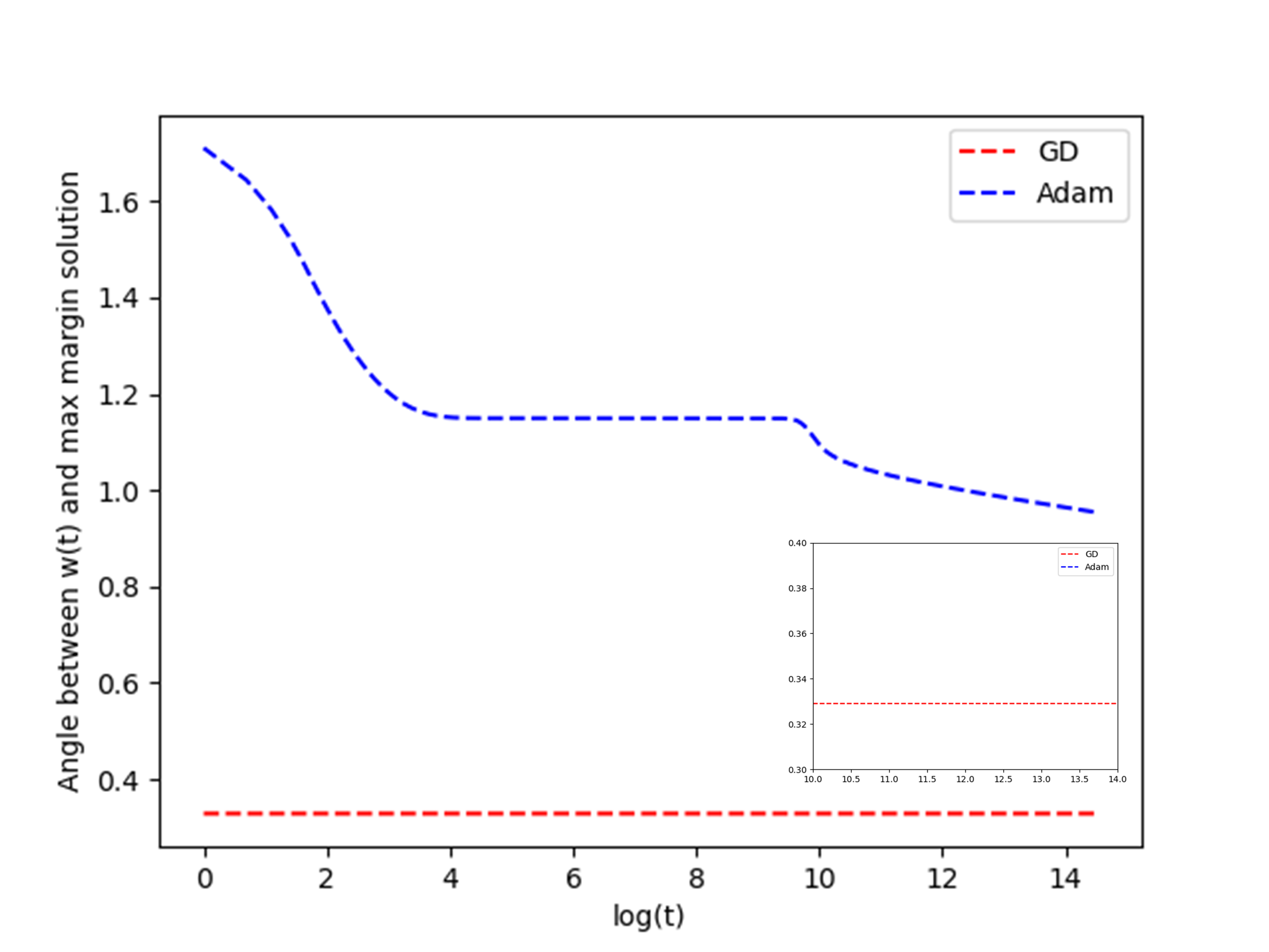}
          \caption{Comparison of Angle: $\eta=0.1$, random seed $=1$}
\end{subfigure}%
\end{minipage}%

\begin{minipage}{1.0\textwidth}
\centering
\begin{subfigure}{.45\textwidth}
          \centering
          \includegraphics[width=1.0\textwidth]{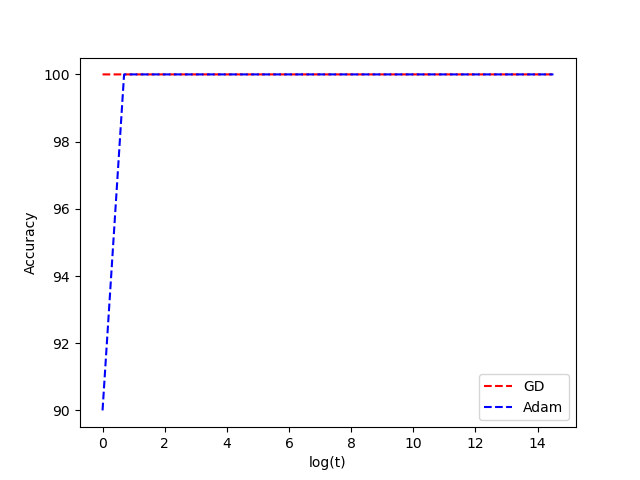}
          \caption{Comparison of Accuracy: $\eta=0.1$, random seed $=2$}
        \end{subfigure}%
        \hspace{1mm}
\begin{subfigure}{.45\textwidth}
          \centering
          \includegraphics[width=1.0\textwidth]{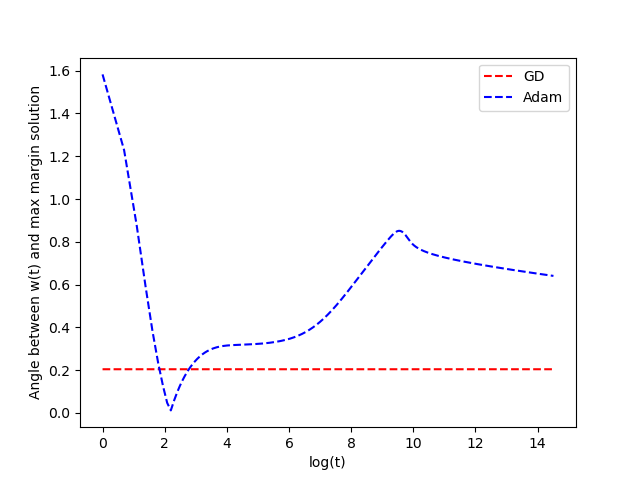}
          \caption{Comparison of Angle: $\eta=0.1$, random seed $=2$}
\end{subfigure}%
\end{minipage}%

\centering
\caption{Comparison of GD and Adam on the ill-posed synthetic dataset in \cite{soudry2018implicit}.}
\label{fig: ill_experi_1}
\end{figure}

\subsection{Evidence from deep neural networks}
We conduct an experiment on the MNIST dataset using the four layer convolutional networks used in \cite{lyu2019gradient, wang2021implicit} (first proposed by \cite{madry2017towards}) to verify whether SGD and SGDM still behave similarly in (homogeneous) deep neural networks. The learning rates of the optimizers are all set to be the default in Pytorch. The results can be seen in Figure \ref{fig: cnn}. It can be observed that (1). SGDM achieves similar test accuracy compared to SGD while (2). SGDM converges faster than SGD.

\begin{figure}[htbp]
\centering
\vspace{-1mm}

\begin{minipage}{1.0\textwidth}
\centering
\begin{subfigure}{.45\textwidth}
          \centering
          \includegraphics[width=1.0\textwidth]{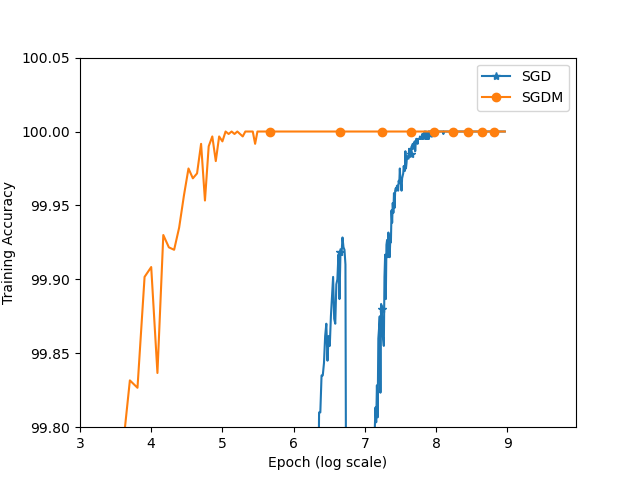}
          \caption{Comparison of Training Accuracy}
        \end{subfigure}%
        \hspace{1mm}
\begin{subfigure}{.45\textwidth}
          \centering
          \includegraphics[width=1.0\textwidth]{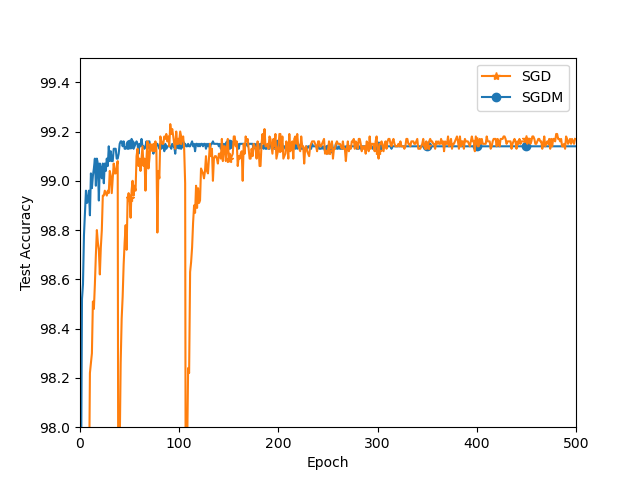}
          \caption{Comparison of Test Accuracy}
\end{subfigure}%
\end{minipage}%
\centering
\caption{Comparison of SGD and SGDM on MNIST with a four-layer CNN.}
\label{fig: cnn}
\end{figure}